\documentclass[twoside,11pt]{article}

\usepackage{blindtext}

\usepackage[abbrvbib, nohyperref]{jmlr2e}

\usepackage{lastpage}
\jmlrheading{25}{2024}{1-\pageref{LastPage}}{9/23}{4/24}{23-1257}{Zhen Fang, Yixuan Li, Feng Liu, Bo Han, Jie Lu}
\ShortHeadings{On the Learnability of Out-of-distribution Detection}{Fang, Li, Liu, Han, Lu}

\firstpageno{1}

\usepackage{ragged2e}
\usepackage{microtype}
\usepackage{graphicx}
\usepackage{subfigure}
\usepackage{booktabs} 
\usepackage{amsmath,amsfonts}
\usepackage{mathrsfs}
\usepackage{algorithmic}
\usepackage{float}
\usepackage{soul}
\usepackage{graphicx}
\usepackage{textcomp}
\usepackage{xcolor}
\usepackage{colortbl,booktabs}
\definecolor{mydarkblue}{rgb}{0,0.08,0.45}
\usepackage[colorlinks, citecolor=mydarkblue]{hyperref}
\usepackage[OT1]{fontenc} 
\usepackage{helvet}  
\usepackage{courier}  
\usepackage{bm}
\usepackage{bbm}
\newtheorem{Theorem}{{Problem}}
\newtheorem{assumption}{{Assumption}}
\newtheorem{theorem*}{Theorem}
\newtheorem{Definition}{Definition}
\newtheorem{lemma}{Lemma} 
\newtheorem{Proposition}{Proposition}

\newtheorem{Condition}{\textbf{Condition}}
\DeclareMathOperator*{\argmin}{arg\,min}
\DeclareMathOperator*{\argmax}{arg\,max}
\usepackage{subfigure}
\usepackage{graphbox}
\usepackage{thm-restate}
\usepackage{setspace}
\usepackage{titletoc}
\usepackage[misc]{ifsym}
\makeatletter
  \def\ttl@Hy@steplink#1{%
    \Hy@MakeCurrentHrefAuto{#1*}%
    \edef\ttl@Hy@saveanchor{%
      \noexpand\Hy@raisedlink{%
        \noexpand\hyper@anchorstart{\@currentHref}%
        \noexpand\hyper@anchorend
        \def\noexpand\ttl@Hy@SavedCurrentHref{\@currentHref}%
        \noexpand\ttl@Hy@PatchSaveWrite
      }%
    }%
  }%
  \def\ttl@Hy@PatchSaveWrite{%
    \begingroup
      \toks@\expandafter{\ttl@savewrite}%
      \edef\x{\endgroup
        \def\noexpand\ttl@savewrite{%
          \let\noexpand\@currentHref
              \noexpand\ttl@Hy@SavedCurrentHref
          \the\toks@
        }%
      }%
    \x
  }%
  \def\ttl@Hy@refstepcounter#1{%
    \let\ttl@b\Hy@raisedlink
    \def\Hy@raisedlink##1{%
      \def\ttl@Hy@saveanchor{\Hy@raisedlink{##1}}%
    }%
    \refstepcounter{#1}%
    \let\Hy@raisedlink\ttl@b
  }%
\def\ttl@gobblecontents#1#2#3#4{\ignorespaces}%
\makeatother
\newcommand\DoToC{%
  \startcontents
  \printcontents{}{1}{
  \textbf{\Large \noindent Table of Contents of Appendix}
  \vskip2pt\hrule\vskip2pt
  \protect\setlength{\parskip}{0pt}\protect\onehalfspacing
  }
  \vskip2pt\hrule\vskip2pt
}

\begin{document}

\title{On the Learnability of Out-of-distribution Detection}

\author{\name Zhen Fang \email zhen.fang@uts.edu.au \\
       \addr Australian Artificial Intelligence Institute \\
       University of Technology Sydney\\
       61 Broadway, Ultimo NSW 2007, Australia
       \AND
       \name Yixuan Li \email sharonli@cs.wisc.edu \\
       \addr Department of Computer Sciences\\
       The University of Wisconsin Madison\\
       1210 W Dayton St, Madison, WI 53706, USA
       \AND
       \name Feng Liu~$^{\textrm{\Letter}}$ \email feng.liu1@unimelb.edu.au \\
       \addr School of Computing and Information Systems\\
       The University of Melbourne\\
       700 Swanston Street, Carlton VIC 3053, Australia
       \AND
       \name Bo Han \email bhanml@comp.hkbu.edu.hk \\
       \addr Department of Computer Science\\
       Hong Kong Baptist University\\
       Kowloon Tong, Hong Kong SAR
       \AND
       \name Jie Lu~$^{\textrm{\Letter}}$ \email jie.lu@uts.edu.au \\
       \addr Australian Artificial Intelligence Institute \\
       University of Technology Sydney\\
       61 Broadway, Ultimo NSW 2007, Australia}

\editor{Amos Storkey}
\maketitle
\vspace{-1.5em}
\noindent {\small \textbf{Accepted by JMLR} in 7th of April, 2024} \\ \\

\begin{abstract}
Supervised learning aims to train a classifier under the assumption that training and test data are from the same distribution. 
To ease the above assumption, researchers have studied a more realistic setting: \emph{out-of-distribution} (OOD) detection, where test data may come from classes that are unknown during training (\textit{i.e.}, OOD data).
Due to the unavailability and diversity of OOD data,
good generalization ability is crucial for effective OOD detection algorithms, and corresponding learning theory is still an \emph{open problem}.
To study the generalization of OOD detection, this paper investigates the \emph{probably approximately correct} (PAC) learning theory of OOD detection that fits the commonly used evaluation metrics in the literature. First, we find a necessary condition for the learnability of OOD detection. Then, using this condition, we prove several impossibility theorems for the learnability of OOD detection under some scenarios. Although the impossibility theorems are frustrating, we find that some conditions of these impossibility theorems may not hold in some practical scenarios.
Based on this observation, 
we next give several necessary and sufficient conditions to characterize the learnability of OOD detection in some practical scenarios. 
Lastly, we offer theoretical support for representative OOD detection works based on our OOD theory. 
\end{abstract}

\section{Introduction}
\label{sec:intro}

The success of supervised learning is established on an \textit{in-distribution} (ID) assumption that training and test data share the same distribution  \citep{Dosovitskiy2021animage,Huang2017densely,DBLP:conf/cvpr/HsuSJK20,Yang2021generalized}. However, in many real-world scenarios, the distribution of test data violates the assumption and, instead, contains \emph{out-of-distribution} (OOD) data whose labels have not been seen during the training process \citep{openmax16cvpr,chen2021robustifying}. To mitigate the risk brought by OOD data, a more practical learning scenario is considered in the machine learning field: OOD detection, which determines whether an input is ID/OOD, while classifying the ID data into respective classes. 
\\

\vspace{-0.6em}
\noindent OOD detection can significantly increase the reliability of machine learning models when deploying them in the real world. 
Many seminar algorithms have been developed to {empirically} address the OOD detection problem \citep{ Hendrycks2017abaseline,liang2018enhancing,lee2018asimple,zong2018deep, Pidhorskyi2018generative,Nalisnick2019do,Hendrycks2019deep, ren2019likelihood, lin2021mood, salehi2021unified,sun2021react}. A common solution paradigm to OOD detection is to propose a new learning objective or/and a score function to identify if one upcoming data point is OOD data. When evaluating algorithms under this solution paradigm, both threshold-dependent metrics (\textit{e.g.}, risk) and threshold-independent metrics (\textit{e.g.}, AUC) will be used to see to what extent the algorithms can successfully identify OOD data. However, very few works study theory of OOD detection,  which hinders the rigorous path forward for the field. This paper aims to bridge the gap.
\\

\vspace{-0.6em}
\noindent In this paper, a theoretical framework is proposed to understand the learnability of OOD detection problem in view of threshold-dependent metrics and threshold-independent metrics\footnote{This paper is an extended version of our previous conference paper \citep{fang2022is}. In Section~\ref{sec:discuss}, we discuss the main difference between this paper and \cite{fang2022is}.}. We investigate the probably approximately correct (PAC) learning theory of OOD detection when the evaluation metrics are risk and AUC, which is posed as an open problem to date. Unlike the classical PAC learning theory in a supervised setting, our problem setting is fundamentally challenging due to the \emph{absence of OOD data} in training. Because OOD data can be diverse in many real-world scenarios, we want to study whether there exists an algorithm that can be used to detect data from various OOD distributions instead of merely data from some specified OOD distributions. Such is the significance of studying the learning theory for OOD detection \citep{Yang2021generalized}. This motivates our question: \textit{is OOD detection PAC learnable? {i.e.}, is there the PAC learning theory to guarantee the generalization ability of OOD detection under two common metrics: risk and AUC?} 
\\

\vspace{-0.5em}
\noindent To answer the above research question and investigate the learning theory, we mainly focus on two basic spaces: domain space and function space. The domain space is a space consisting of some distributions, and the function space is a space consisting of some classifiers or ranking functions. Existing agnostic PAC theories in supervised learning \citep{shalev2014understanding,mohri2018foundations} are distribution-free, \textit{i.e.}, the domain space consists of all domains. Yet, in {Theorem} \ref{T4} and {Theorem} \ref{T4_auc}, we show that the learning theory of OOD detection is not distribution-free. Furthermore, we find that OOD detection is learnable only if the domain space and the function space satisfy {some special conditions}, \textit{e.g.,} Conditions \ref{C1}, \ref{C1_auc}, and \ref{Con2}. Notably, 
there are many conditions and theorems in existing learning theories and many OOD detection algorithms in the literature. Thus, it is very difficult to analyze the relation between these theories and algorithms, and explore useful conditions to ensure the learnability of OOD detection, especially when we have to explore them \textit{from the scratch}. Thus, the main aim of our paper is to study these essential conditions under risk and AUC metrics. From these essential conditions, we can know \textit{when} OOD detection can be successful in practical scenarios.
We restate our question and goal in the following:

\noindent \fbox{
\parbox{0.97\textwidth}{
\textit{Given hypothesis spaces and several representative domain spaces, what are the {conditions} to ensure the learnability of OOD detection in terms of risk and AUC? If possible, we hope that these conditions are necessary and sufficient in some scenarios.}
}
}

\paragraph{Main Results.} We start to study the learnability of OOD detection in the largest space---the total space, and give two {necessary conditions} for the learnability of OOD detection under risk and AUC (Condition \ref{C1} for risk and Condition \ref{C1_auc} for AUC). However, we find that the overlap between ID and OOD data may result in that both necessary conditions do not hold.
Therefore, we give two impossibility theorems to demonstrate that OOD detection fails in the total space ({Theorem}~\ref{T4} under risk and {Theorem}~\ref{T4_auc} under AUC). Then, we investigate OOD detection in a separate space, where the ID and OOD data do not overlap. Unfortunately, there still exists impossibility theorems ({Theorem}~\ref{T12} under risk and {Theorem}~\ref{T12_auc} under AUC), meaning that we cannot expect OOD detection is learnable under risk and AUC in the separate space under some conditions of the developed theorems. 
\\

\vspace{-0.5em}
\noindent It is frustrating to find the impossibility theorems regarding OOD detection in a separate space, but we find that some conditions of these impossibility theorems may not hold in several practical scenarios.
Stemming from this observation, we give several \emph{necessary and sufficient} conditions to characterize the learnability of OOD detection under risk and AUC in the separate space ({Theorems} \ref{T13} and \ref{T24} under risk, and Theorems \ref{T13_auc} and \ref{T24_auc} under AUC). Especially, when our function space is based on \emph{fully-connected neural network} (FCNN),  OOD detection is learnable under risk and AUC in the separate space if and only if the feature space is finite.
Then, we focus on other more practical domain spaces, \textit{e.g.}, the finite-ID-distribution space and the density-based space and investigate the learnability of OOD detection in both spaces. {Theorem} \ref{T-SET} shows a necessary and sufficient condition of learnability of OOD detection under risk. {Theorems} \ref{T-SET2} and \ref{T-SET2_auc} show two sufficient conditions for the learnability of OOD detection under risk and AUC, respectively. It should be noted that when studying learnability of OOD detection in the finite-ID-distribution space, we discover a compatibility condition ({Condition} \ref{Con2}) that is a necessary and sufficient condition of learnability of OOD detection under risk for this space. Then, we explore the compatibility condition in the density-based space, and find that such condition is also the necessary and sufficient condition in some practical scenarios ({Theorem} \ref{T24.3}).

\paragraph{Implications and Impacts of Theory.} Our study is not of purely theoretical interest; it has also practical impacts.
(i) From the perspective of domain space, we consider the finite-ID-distribution space that fits the common scenarios in the real world: we normally only have finite ID datasets. In this case, {Theorem} \ref{T-SET} gives a necessary and sufficient condition to the success of OOD detection under risk. More importantly, our theory shows that OOD detection is learnable in image-based scenarios when ID images have clearly different semantic labels and styles  (\textit{far-OOD}) from OOD images. 
(ii) From the perspective of function space, we investigate the learnability of OOD detection under risk and AUC for commonly used FCNN-based function spaces. Our theory provides theoretical support ({Theorems} \ref{T24} and \ref{T24.3} under risk, and {Theorems} \ref{T24_auc} and \ref{T24.3_auc} under AUC) for several representative OOD detection works \citep{Hendrycks2017abaseline,liang2018enhancing,liu2020energy}. 
(iii) From the perspective of evaluation metrics, our paper studies the learnability of OOD detection under risk and the learnability of OOD detection under AUC, which covers the major evaluation metrics used in OOD detection evaluation and provides theoretical guidance when users have different requirement in evaluating OOD detection performance.
Based on all of our theoretical results, they suggest we should not expect a universally working OOD detection algorithm. It is necessary to design different algorithms in different scenarios.



\section{Learning Setups}\label{S3}
We begin by introducing the necessary concepts and notations for our theoretical framework. Given a feature space $\mathcal{X}\subset \mathbb{R}^d$ and a label space $\mathcal{Y}:=\{1,\ldots,K\}$, we have an ID joint distribution $D_{X_{\rm I}Y_{\rm I}}$ over $\mathcal{X}\times \mathcal{Y}$, where $X_{\rm I}\in \mathcal{X}$ and $Y_{\rm I}\in \mathcal{Y}$ are random variables. We also have an OOD joint distribution $D_{X_{\rm O}Y_{\rm O}}$, where $X_{\rm O}$ is a random variable from $\mathcal{X}$, but $Y_{\rm O}$ is a random variable whose outputs do not belong to $\mathcal{Y}$. During testing, we encounter a mixture of ID and OOD joint distributions: $D_{XY}=(1-\pi^{\rm out})D_{X_{\rm I}Y_{\rm I}}+\pi^{\rm out}D_{X_{\rm O}Y_{\rm O}}$, and we can only observe the marginal distribution $D_X=(1-\pi^{\rm out})D_{X_{\rm I}}+\pi^{\rm out}D_{X_{\rm O}}$, where the constant $\pi^{\rm out}\in [0,1)$ represents an unknown class-prior probability. Next, we provide the formal definition of the OOD detection problem and key concepts used in this paper.

\subsection{Problem Setting and Concepts}

\begin{Theorem}[OOD Detection \citep{Yang2021generalized}]\label{P1}
Given an ID joint distribution $D_{X_{\rm I}Y_{\rm I}}$ and a training data $S:=\{(\mathbf{x}^1,{y}^1),...,(\mathbf{x}^n,{y}^n)\}$ drawn {independent and identically distributed}  from $D_{X_{\rm I}Y_{\rm I}}$,
the aim of OOD detection is to train a classifier $f$ by using the training data $S$ such that, for any test data $\mathbf{x}$ drawn from the mixed marginal distribution $D_X$:
\begin{itemize}
    \item if $\mathbf{x}$ is an observation from $D_{X_{\rm I}}$, $f$ can classify $\mathbf{x}$ into correct ID classes;
    \item if $\mathbf{x}$ is an observation from $D_{X_{\rm O}}$, $f$ can detect $\mathbf{x}$ as OOD data.
\end{itemize}
\end{Theorem}
According to \cite{Yang2021generalized},  when $K=1$, OOD detection reduces to one-class novelty detection or  semantic anomaly detection \citep{DBLP:conf/icml/RuffGDSVBMK18,DBLP:conf/icml/GoyalRJS020,DBLP:conf/pkdd/DeeckeVRMK18}. Next, we introduce some basic and important concepts and notations. 

$~$
\\
\textbf{OOD Label and Domain Space.} Based on Problem \ref{P1}, we know it is not necessary to classify OOD data
into the correct OOD classes. Without loss of generality,
let all OOD data be allocated to one big OOD class, \textit{i.e.}, $Y_{\rm O}=K+1$ \citep{Fang2021learning,Fang2020Open}. 
 To investigate the PAC learnability of OOD detection, we define a domain space $\mathscr{D}_{XY}$, which is a set consisting of some joint distributions $D_{XY}$ mixed by some ID joint distributions and some OOD joint distributions.
In this paper, the joint distribution $D_{XY}$ mixed by ID joint distribution $D_{X_{\rm I}Y_{\rm I}}$ and OOD joint distribution $D_{X_{\rm O}Y_{\rm O}}$ is called \textbf{\textit{domain}}.
\\

\noindent \textbf{Hypothesis Spaces and Scoring Function Spaces.} A hypothesis space $\mathcal{H}$ is a subset of function space, \textit{i.e.},
$
  \mathcal{H}\subset  \{ { h}: \mathcal{X}\rightarrow \mathcal{Y}\cup \{K+1\}\}.
$ We  set $\mathcal{H}^{\rm in}\subset  \{ { h}: \mathcal{X}\rightarrow \mathcal{Y}\}$ to the ID hypothesis space. We also define $\mathcal{H}^{\rm b}\subset  \{ { h}: \mathcal{X}\rightarrow \{1,2\}\}$ as the hypothesis space for binary classification, where $1$ represents the ID data, and $2$ represents the OOD data. The function $h$ is called the hypothesis function. A scoring function space is a subset of function space, \textit{i.e.}, $\mathcal{F}_{l}\subset  \{ \mathbf{f}: \mathcal{X}\rightarrow \mathbb{R}^{l}\}$, where $l$ is the output's dimension of the vector-valued function $\mathbf{f}$. The function $\mathbf{f}$ is called the scoring function. 
\\

\noindent \textbf{Ranking Function Spaces.} 
Most representative OOD detection algorithms \citep{liu2020energy} output a ranking function from a given ranking function space $\mathcal{R} \subset \mathcal{R}_{\rm all}=\{r:\mathcal{X}\rightarrow \mathbb{R}\}$. If the ranking function $r(\mathbf{x})$ has a higher value, then $\mathbf{x}$ is from $D_{X_{\rm I}}$ with a higher probability. A perfect ranking function $r^*$ fulfills the condition $r^*(\mathbf{x})>r^*(\mathbf{x}')$ for all $\mathbf{x}$ from $D_{X_{\rm I}}$ and all $\mathbf{x}'$ from $D_{X_{\rm O}}$, indicating that rankings of ID data are always higher than rankings of OOD data. The general strategy to construct the ranking function space $\mathcal{R}$ is to design a scoring function $E: \mathbb{R}^l \rightarrow \mathbb{R}$ and integrate it with the scoring function space $\mathcal{F}_{l}$, i.e., $\mathcal{R} = E\circ \mathcal{F}_{l}$.
\\



\noindent \textbf{Loss, Risks and AUC Metric.} Let $\mathcal{Y}_{\rm all}=\mathcal{Y}\cup\{K+1\}$. Given a loss function $\ell:\mathcal{Y}_{\rm all}\times \mathcal{Y}_{\rm all}\rightarrow \mathbb{R}_{\geq 0}$ satisfying that $\ell(y_1,y_2)=0$ if and only if $y_1=y_2$, and any ${ h}\in \mathcal{H}$, then the \textit{risk} with respect to ${D}_{XY}$ is
\begin{equation}\label{risks}
\begin{split}
   & R_{D}({h}):= \mathbb{E}_{(\mathbf{x},y)\sim D_{XY}} \ell({ h}(\mathbf{x}),{y}).
\end{split}
 \end{equation} 
The $\alpha$-risk $R_{D}^{\alpha}({h}):= (1-\alpha)R^{\rm in}_D({h})+\alpha R^{\rm out}_D({h}), \forall \alpha\in [0,1]$, where $R^{\rm in}_D({h})$ and $R^{\rm out}_D({h})$ are
\begin{equation*}
\begin{split}
  & R^{\rm in}_D({h}):= \mathbb{E}_{(\mathbf{x},y)\sim D_{X_{\rm I}Y_{\rm I}}} \ell({ h}(\mathbf{x}),{y}),
~~~~~~R^{\rm out}_D({h}):= \mathbb{E}_{\mathbf{x}\sim D_{X_{\rm O}}} \ell({ h}(\mathbf{x}),K+1).
\end{split}
\end{equation*}
Except for using risk to evaluate the OOD detection performance, AUC is also a {promising metric} to see if a ranking function $r$ can separate the ID and OOD data:
 \begin{equation}\label{auc}
     {\rm AUC}(r;D_{XY}) = \mathbb{E}_{\mathbf{x}\sim D_{X_{\rm I}}}\mathbb{E}_{\mathbf{x}'\sim D_{X_{\rm O}}} \big [ \mathbf{1}_{r(\mathbf{x})>r(\mathbf{x}')}  + \frac{1}{2} \mathbf{1}_{r(\mathbf{x})=r(\mathbf{x}')}\big].
\end{equation}
Note that since value of ${\rm AUC}$ only denpends on the marginal distributions $D_{X_{\rm I}}$ and $D_{X_{\rm O}}$, therefore, it is also convientent for us to rewrite ${\rm AUC}(r;D_{XY})$ as 
$
     {\rm AUC}(r;D_{X_{\rm I}},D_{X_{\rm O}}).
$


\subsection{Learnability under Risk} 

Based on risk defined in Eq.~\eqref{risks}, OOD detection aims to select a hypothesis function $h \in \mathcal{H}$ with approximately minimal risk,
based on finite data. Generally, we expect the approximation to get better, with the increase in sample size. Algorithms achieving this are said to be consistent under risk. Formally, we have: 
\begin{Definition}[Learnability of OOD Detection under Risk]\label{D0}
Given a domain space $\mathscr{D}_{XY}$ and a hypothesis space $\mathcal{H}\subset  \{ { h}: \mathcal{X}\rightarrow \mathcal{Y}_{\rm all}\}$, we say  OOD detection is \textbf{learnable} in  $\mathscr{D}_{XY}$ for $\mathcal{H}$ under risk, if there exists an algorithm $\mathbf{A}\footnote{Similar to \cite{shalev2010learnability}, in this paper, we regard an algorithm as a mapping from $\cup_{n=1}^{+\infty}(\mathcal{X}\times\mathcal{Y})^n$ to $\mathcal{H}$ or $\mathcal{R}$.}: \cup_{n=1}^{+\infty}(\mathcal{X}\times\mathcal{Y})^n\rightarrow \mathcal{H}$ and a monotonically decreasing
sequence $\epsilon_{\rm cons}(n)$, such that $\epsilon_{\rm cons}(n)\rightarrow 0$, as $n\rightarrow +\infty$, and for any domain $D_{XY}\in \mathscr{D}_{XY}$,
\begin{equation}\label{issue-definition1}
    \mathbb{E}_{S\sim D^n_{X_{\rm I}Y_{\rm I}}}\big[ R_D(\mathbf{A}(S))- \inf_{h\in \mathcal{H}}R_D(h)\big] \leq \epsilon_{\rm cons}(n),
\end{equation}
An algorithm $\mathbf{A}$ for which this holds is said to be consistent with respect to $\mathscr{D}_{XY}$.
\end{Definition}
 Definition \ref{D0} is a natural extension of agnostic PAC learnability of supervised learning \citep{shalev2010learnability}. If for any $D_{XY}\in \mathscr{D}_{XY}$, $\pi^{\rm out}=0$, then Definition \ref{D2} is the agnostic PAC learnability of supervised learning. Although the expression of Definition \ref{D0} is different from the normal definition of agnostic PAC learning in \cite{shalev2014understanding}, one can prove that they are equivalent if $\ell$ is bounded, see Appendix \ref{SB.3}.

\noindent Since OOD data are unavailable, it is impossible to obtain any information about the class-prior probability $\pi^{\rm out}$. Furthermore, in the real world, it is possible that $\pi^{\rm out}$
can be any value in $[0,1)$. Therefore, the imbalance issue between ID and OOD distributions, and the priori-unknown issue (\textit{i.e.}, $\pi^{\rm out}$ is unknown) are the core challenges. To mitigate this challenge, we revise Eq. \eqref{issue-definition1} as follows:
\begin{equation}\label{issue-definition2}
\begin{split}
   & \mathbb{E}_{S\sim D^n_{X_{\rm I}Y_{\rm I}}}\big[ R_D^{\alpha}(\mathbf{A}(S))- \inf_{h\in \mathcal{H}}R_D^{\alpha}(h)\big] \leq \epsilon_{\rm cons}(n),~\forall \alpha\in[0,1].
    \end{split}
\end{equation}
If an algorithm $\mathbf{A}$ satisfies Eq. \eqref{issue-definition2}, then the imbalance issue and the prior-unknown issue disappear. That is, $\mathbf{A}$ can simultaneously classify the ID data and detect the OOD data well. Based on the above discussion, we define the strong learnability of OOD detection under risk as follows:
\begin{Definition}[Strong Learnability of OOD Detection under Risk]\label{D2}
Given a domain space $\mathscr{D}_{XY}$ and a hypothesis space $\mathcal{H}\subset  \{ { h}: \mathcal{X}\rightarrow \mathcal{Y}_{\rm all}\}$, we say  OOD detection is \textbf{strongly learnable} in  $\mathscr{D}_{XY}$ for $\mathcal{H}$, if there exists an algorithm $\mathbf{A}: \cup_{n=1}^{+\infty}(\mathcal{X}\times\mathcal{Y})^n\rightarrow \mathcal{H}$ and a monotonically decreasing
sequence $\epsilon_{\rm cons}(n)$, such that $\epsilon_{\rm cons}(n)\rightarrow 0$, as $n\rightarrow +\infty$, and for any domain $D_{XY}\in \mathscr{D}_{XY}$,
\begin{equation*}
    \mathbb{E}_{S\sim D^n_{X_{\rm I}Y_{\rm I}}}\big[ R_D^{\alpha}(\mathbf{A}(S))- \inf_{h\in \mathcal{H}}R_D^{\alpha}(h)\big] \leq \epsilon_{\rm cons}(n), ~\forall \alpha \in[0,1].
\end{equation*}
\end{Definition}

\noindent \textbf{Remark.}
    In Theorem \ref{T1}, we have shown that the strong learnability of OOD detection under risk is equivalent to the learnability of OOD detection under risk, if the domain space $\mathscr{D}_{XY}$ is a \textit{prior-unknown space} (see Definition \ref{D3}). In this paper, we mainly discuss the learnability in the prior-unknown space. Therefore, \textit{when we mention that OOD detection is learnable under risk, we also mean that OOD detection is strongly learnable under risk}.

\subsection{Learnability under AUC} 

Based on AUC defined in Eq.~\eqref{auc}, OOD detection aims to select a ranking function $r \in \mathcal{R}$ with approximately maximal AUC, based on finite data. Generally, we expect the approximation to get better, with the increase in sample size. Algorithms achieving this are said to be consistent under AUC. Formally, we have: 
\begin{Definition}[Learnability of OOD Detection under AUC]\label{D0_auc}
Given a domain space $\mathscr{D}_{XY}$, a ranking function space $\mathcal{R}\subset  \{ r: \mathcal{X}\rightarrow \mathbb{R}\}$, we say  OOD detection is \textbf{learnable} in  $\mathscr{D}_{XY}$ for $\mathcal{R}$ under AUC, if there exists an algorithm $\mathbf{A}: \cup_{n=1}^{+\infty}(\mathcal{X}\times\mathcal{Y})^n\rightarrow \mathcal{R}$ and a monotonically decreasing
sequence $\epsilon_{\rm cons}(n)$, such that $\epsilon_{\rm cons}(n)\rightarrow 0$, as $n\rightarrow +\infty$, and for any domain $D_{XY}\in \mathscr{D}_{XY}$,
\begin{equation}\label{issue-definition1_auc}
    \mathbb{E}_{S\sim D^n_{X_{\rm I}Y_{\rm I}}}\big[\sup_{r\in \mathcal{R}}{\rm AUC}(r;D_{XY})- {\rm AUC}(\mathbf{A}(S);D_{XY})\big] \leq \epsilon_{\rm cons}(n).
\end{equation}
An algorithm $\mathbf{A}$ for which this holds is said to be ${\rm AUC}$ consistent with respect to $\mathscr{D}_{XY}$.
\end{Definition}
Definition \ref{D0_auc} is another version of Definition~\ref{D0}. Here, we use AUC instead of risk to evaluate the performance of the OOD detection. Note that the learnability of OOD detection under AUC is not influenced by the $\pi_{\rm out}$, as AUC is directly calculated by using $D_{X_{\rm I}}$ and $D_{X_{\rm O}}$. 

\subsection{Goal of Our Theory}
Note that the agnostic PAC learnability of supervised learning is distribution-free, \textit{i.e.}, the domain space $\mathscr{D}_{XY}$ consists of all domains. However, due to the absence of OOD data during the training process \citep{liang2018enhancing,ren2019likelihood,Fang2021learning}, it is obvious that the learnability of OOD detection is not distribution-free (\textit{i.e.}, Theorem \ref{T4} and {Theorem \ref{T4_auc}}). In fact, we discover that the learnability of OOD detection is deeply correlated with the relationship between the domain space $\mathscr{D}_{XY}$ and the hypothesis space $\mathcal{H}$ (or the ranking function space $\mathcal{R}$). That is, OOD detection is learnable only when the domain space $\mathscr{D}_{XY}$ and the hypothesis space $\mathcal{H}$ (or the ranking function space $\mathcal{R}$) satisfy some special conditions, \textit{e.g.,} Conditions \ref{C1}, \ref{Con2} (under risk), {Conditions} \ref{C1_auc} (under AUC). We present our goal as follows:
\\
$~~~~~~~~~$\fbox{
\parbox{0.81\textwidth}{
\textit{\textbf{Goal:} given a hypothesis space $\mathcal{H}$ $($or a ranking function space $\mathcal{R})$, and several representative domain spaces $\mathscr{D}_{XY}$, what are the \textbf{conditions} to ensure the learnability of OOD detection? Furthermore, if possible, we hope that these conditions are \textbf{necessary and sufficient} in some scenarios.}
}
}

\vspace{0.4em}
\noindent Therefore, compared to the agnostic PAC learnability of supervised learning, our theory doesn't focus on the distribution-free case, but focuses on discovering essential conditions to guarantee the learnability of OOD detection in several representative and practical domain spaces $\mathscr{D}_{XY}$. By these essential conditions, we can know \textit{when} OOD detection can be successful in real applications.
\section{Learning in Priori-unknown Spaces}

We first investigate a special space, called {prior-unknown} space 
and prove that if OOD detection is strongly learnable under risk or learnable under AUC in a space $\mathscr{D}_{XY}$, then one can discover a larger domain space, which is prior-unknown, to ensure the learnability of OOD detection under risk or AUC. These results imply that it is enough to study learnabiligy of OOD detection in the prior-unknown spaces. The prior-unknown space is as follows:

\begin{Definition}\label{D3}
Given a domain space $\mathscr{D}_{XY}$, we say $\mathscr{D}_{XY}$ is a priori-unknown space, if for any domain $D_{XY}\in \mathscr{D}_{XY}$ and any $\alpha \in [0,1)$, we have $D_{XY}^{\alpha}:=(1-\alpha)D_{X_{\rm I}Y_{\rm I}}+\alpha D_{X_{\rm O}Y_{\rm O}}\in \mathscr{D}_{XY}$.
\end{Definition}

\noindent Then the following theorem presents importance and necessity of  priori-unknown space.
\begin{restatable}{theorem}{thmCone}
\label{T1}
Given spaces $\mathscr{D}_{XY}$ and
$\mathscr{D}_{XY}'=\{D_{XY}^{\alpha}:\forall D_{XY}\in \mathscr{D}_{XY}, \forall \alpha\in [0,1)\}$, then
\\
1) $\mathscr{D}_{XY}'$ is a priori-unknown space and $\mathscr{D}_{XY}\subset \mathscr{D}_{XY}'$;\\
2) if $\mathscr{D}_{XY}$ is a priori-unknown space, then Definition \ref{D0} and Definition \ref{D2} are \textbf{equivalent};\\
3) OOD detection is strongly learnable in $\mathscr{D}_{XY}$ under risk \textbf{if and only if} OOD detection is learnable in $\mathscr{D}_{XY}'$ under risk; \\
4) OOD detection is learnable in $\mathscr{D}_{XY}$ under AUC \textbf{if and only if} OOD detection is learnable in $\mathscr{D}_{XY}'$ under AUC.
\end{restatable}
\noindent The second result of Theorem \ref{T1} bridges the learnability and strong learnability under risk, which implies that if an algorithm $\mathbf{A}$ is consistent with respect to a prior-unknown space, then this algorithm $\mathbf{A}$ can address the imbalance issue between ID and OOD distributions, and the priori-unknown issue well. The fourth result of Theorem \ref{T1} shows that the learnability of OOD detection under AUC is not influenced by the unknown class-prior probability $\pi^{\rm out}$. Based on Theorem \ref{T1}, we focus on our theory in the prior-unknown spaces. To demystify the learnability of OOD detection, we introduce five representative priori-unknown spaces:

\begin{itemize}
    \item Single-distribution space $\mathscr{D}_{XY}^{D_{XY}}$.  For a domain $D_{XY}$, $\mathscr{D}_{XY}^{D_{XY}}:=\{D_{XY}^{\alpha}: \forall \alpha \in [0,1)\}$.
    \item Total space $\mathscr{D}_{XY}^{\rm all}$, which consists of all domains.
    \item Separate space $\mathscr{D}_{XY}^{s}$, which consists of all domains that satisfy the separate condition, that is for any $D_{XY}\in \mathscr{D}_{XY}^{s}$,
    $
        {\rm supp} D_{X_{\rm O}}\cap   {\rm supp} D_{X_{\rm I}}=\emptyset,
    $
    where ${\rm supp}D$ means the support set of a distribution $D$.
    \item Finite-ID-distribution space $\mathscr{D}_{XY}^{F}$, which is a prior-unknown space satisfying that the number of distinct ID joint distributions $D_{X_{\rm I}Y_{\rm I}}$ in $\mathscr{D}_{XY}^{F}$ is finite, \textit{i.e.}, $|\{D_{X_{\rm I}Y_{\rm I}}:\forall D_{XY}\in \mathscr{D}_{XY}^{F}\}|<+\infty$.
    \item Density-based space $\mathscr{D}^{\mu,b}_{XY}$, which is a prior-unknown space consisting of some domains satisfying that: for any $D_{XY}$, there exists a density function $f$ with $1/b\leq f\leq b$ in ${\rm supp}\mu$ and $0.5*D_{X_{\rm I}}+0.5*D_{X_{\rm O}}= \int f {\rm d}\mu$, where $\mu$ is a measure defined over $\mathcal{X}$. Note that if $\mu$ is discrete, then $D_{X}$ is a discrete distribution; and if $\mu$ is the Lebesgue measure, then $D_{X}$ is a continuous distribution.
\end{itemize}

\noindent The above representative spaces widely exist in real applications. For example, 1) if the images from different semantic labels with different styles are clearly different, then those images can form a distribution belonging to a separate space $\mathscr{D}_{XY}^{s}$; and 2) when designing an algorithm, we only have finite ID datasets, \textit{e.g.}, CIFAR-10, MNIST, SVHN, and ImageNet, to build a model. Then, finite-ID-distribution space $\mathscr{D}_{XY}^{F}$ can handle this real scenario. Note that the single-distribution space is a special case of the finite-ID-distribution space. In this paper, we mainly discuss these five spaces.

\section{Impossibility Theorems for OOD Detection}

In this section, we first give a necessary condition for the learnability of OOD detection. Then, we show this necessary condition does not hold in the total space $\mathscr{D}_{XY}^{\rm all}$ and the separate space $\mathscr{D}_{XY}^{s}$. 

\subsection{Necessary Conditions for Learnability of OOD Detection} 

We first find a necessary condition for the learnability of OOD detection under risk (AUC), \textit{i.e.}, Condition \ref{C1} (Condition \ref{C1_auc}).
\begin{Condition}[Linear Condition under Risk]\label{C1}
For any $D_{XY}\in \mathscr{D}_{XY}$ and any $\alpha \in [0,1)$,
\begin{equation*}
   \inf_{h\in \mathcal{H}}R_D^{\alpha}(h)= (1-\alpha)\inf_{h\in \mathcal{H}}R_D^{\rm in}(h)+\alpha \inf_{h\in \mathcal{H}}R_D^{\rm out}(h).
\end{equation*}
\end{Condition}
The importance of Condition \ref{C1} is reflected by Theorem \ref{T3}, showing that Condition \ref{C1} is a \textit{necessary and sufficient} condition for the learnability of OOD detection under risk if the $\mathscr{D}_{XY}$ is the single-distribution space. 

\begin{restatable}{theorem}{thmCondoneIff}
\label{T3}
Given a hypothesis space $\mathcal{H}$ and a domain $D_{XY}$, OOD detection is learnable under risk in the single-distribution space $\mathscr{D}_{XY}^{D_{XY}}$ for $\mathcal{H}$ \textbf{if and only if} Condition \ref{C1} holds.
\end{restatable}
\noindent Theorem \ref{T3} implies that Condition \ref{C1} is important for the learnability of OOD detection under risk. Due to the simplicity of single-distribution space, Theorem \ref{T3} implies that  Condition \ref{C1} is the necessary condition for the learnability of OOD detection under risk in the prior-unknown space, see Lemma \ref{C1andC2} in Appendix \ref{SD}. Then, we focus on finding a necessary condition for the learnability of OOD detection under AUC. The condition is similar to Condition~\ref{C1} but replacing risk with AUC. Note that, for simplicity, in the following of this paper, we use ${\rm AUC}(r; D_{X_{\rm I}}, D_{X_{\rm O}})$ to present ${\rm AUC}(r;D_{XY})$.
\begin{Condition}[Linear Condition under AUC]\label{C1_auc}
For any $D_{XY}=\beta D_{X_{\rm I}Y_{\rm I}}+(1-\beta)D_{X_{\rm O}Y_{\rm O}},$ $D_{XY}'= \beta' D_{X_{\rm I}Y_{\rm I}}+(1-\beta')D_{X_{\rm O}Y_{\rm O}}' \in \mathscr{D}_{XY}$, then for any $\alpha \in [0,1)$,
\begin{equation*}
\begin{split}
 \alpha \sup_{r\in \mathcal{R}} {\rm AUC}(r; D_{X_{\rm I}}, D_{X_{\rm O}}) +(1-\alpha) \sup_{r\in \mathcal{R}} {\rm AUC}(r; D_{X_{\rm I}}, D_{X_{\rm O}}') =  \sup_{r\in \mathcal{R}} {\rm AUC}(r; D_{X_{\rm I}}, {D}_{X_{\rm O}}^{\alpha}),
       \end{split}
\end{equation*}
where ${D}_{X_{\rm O}}^{\alpha} = \alpha D_{X_{\rm O}} + (1-\alpha) D_{X_{\rm O}}'$.
\end{Condition}

\noindent The importance of Condition \ref{C1_auc} is reflected in Theorem \ref{T3_auc}, showing that Condition \ref{C1_auc} is a \textit{necessary} condition for the learnability of OOD detection under AUC if the $\mathscr{D}_{XY}$ is a simple distribution space. 

\begin{restatable}{theorem}{thmCondoneIffauc}
\label{T3_auc}
    Given a ranking function space $\mathcal{R}$ and a domain space $\mathscr{D}_{XY}$, if OOD detection is learnable under AUC for $\mathcal{R}$ in $\mathscr{D}_{XY}$, then for any $D_{XY}, D_{XY}'\in \mathscr{D}_{XY}$, the linear condition under AUC (i.e., Condition \ref{C1_auc}) holds.
\end{restatable}

\noindent Since the Condition~\ref{C1_auc} is a necessary condition for the learnability of OOD detection under AUC, this condition provides a new way to check if an OOD detection is learnable under AUC. Namely, if Condition~\ref{C1_auc} does not hold, OOD detection is not learnable under AUC.

\subsection{Impossibility Theorems under Risk} 

In this subsection, we first study whether Condition \ref{C1} holds in the total space $\mathscr{D}_{XY}^{\rm all}$. If Condition \ref{C1} does not hold, then OOD detection is not learnable under risk.
Theorem~\ref{T5} shows that Condition \ref{C1} is not always satisfied, 
especially, when there is an overlap between the ID and OOD distributions:
\begin{Definition}[Overlap Between ID and OOD]
\label{def:overlap}
We say a domain $D_{XY}$ has overlap between ID and OOD distributions, if there is a $\sigma$-finite measure $\tilde{\mu}$ such that $D_X$ is absolutely continuous with respect to $\tilde{\mu}$, and
$
 \tilde{\mu}(A_{\rm overlap})>0,~\textit{ where}
$
$A_{\rm overlap}= \{\mathbf{x}\in \mathcal{X}:f_{\rm I}(\mathbf{x})>0~ \textit{and}~ f_{\rm O}(\mathbf{x})>0\}$. Here $f_{\rm I}$ and $f_{\rm O}$ are the representers of $D_{X_{\rm I}}$ and $D_{X_{\rm O}}$ in Radon–Nikodym Theorem \citep{cohn2013measure}, 
\begin{equation*}
  D_{X_{\rm I}} = \int f_{\rm I} {\rm d}\tilde{\mu},~~~ D_{X_{\rm O}} = \int f_{\rm O} {\rm d}\tilde{\mu}.
\end{equation*}
\end{Definition}
\begin{restatable}{lemma}{thmImpOne}
\label{T5}
Given a hypothesis space $\mathcal{H}$ and  a prior-unknown space $\mathscr{D}_{XY}$, if there is $D_{XY}\in \mathscr{D}_{XY}$, which has overlap between ID and OOD, and $\inf_{h\in \mathcal{H}} R_{D}^{\rm in}(h)=0$, $\inf_{h\in \mathcal{H}} R_{D}^{\rm out}(h)=0$, then Condition \ref{C1} does not hold. Therefore, OOD detection is not learnable under risk in $\mathscr{D}_{XY}$ for $\mathcal{H}$.
\end{restatable}
Lemma~\ref{T5} clearly shows that under proper conditions, Condition \ref{C1} does not hold, if there exists a domain whose ID and OOD distributions have overlap. By Lemma \ref{T5}, we can obtain that the OOD detection is not learnable in the total space $\mathscr{D}^{\rm all}_{XY}$ for any non-trivial hypothesis space $\mathcal{H}$.

%
\begin{restatable}[Impossibility Theorem for Total Space under Risk]{theorem}{thmImpTotal}
\label{T4}
OOD detection is not learnable under risk in the total space $\mathscr{D}^{\rm all}_{XY}$ for $\mathcal{H}$, if $|\phi\circ\mathcal{H}|>1$, where $\phi$ maps ID labels to $1$ and maps OOD labels to $2$.
\end{restatable}

\noindent Since the overlaps between ID and OOD distributions may cause that Condition~\ref{C1} does not hold, we then consider studying
the learnability of OOD detection in the separate space $\mathscr{D}_{XY}^s$, where there are no overlaps between the ID
and OOD distributions. However, Theorem \ref{T12} shows that even if we consider the separate space, the OOD detection is still not learnable in some scenarios. Before introducing the impossibility theorem for separate space, \textit{i.e.}, Theorem \ref{T12}, we need a mild assumption:
\begin{assumption}[Separate Space for OOD under Risk]\label{ass1}
A hypothesis space $\mathcal{H}$ is separate for OOD data, if for each data point $\mathbf{x}\in \mathcal{X}$, there exists at least one hypothesis function $h_{\mathbf{x}}\in \mathcal{H}$ such that $h_{\mathbf{x}}(\mathbf{x})=K+1$.
\end{assumption}
Assumption \ref{ass1} means that every data point $\mathbf{x}$ has the possibility to be detected as OOD data. Assumption \ref{ass1} is mild and can be satisfied by many hypothesis spaces, {\textit{e.g.}, the FCNN-based hypothesis space} (Proposition \ref{Pr1} in Appendix \ref{SK}), score-based hypothesis space  (Proposition \ref{P2} in Appendix \ref{SK}) and universal kernel space. Next, we use \textit{Vapnik–Chervonenkis} (VC) dimension \citep{mohri2018foundations} to measure the size of hypothesis space, and study the learnability of OOD detection in $\mathscr{D}_{XY}^{s}$ based on the VC dimension. 

\begin{restatable}[Impossibility Theorem for Separate Space under Risk]{theorem}{thmImpSeptwo}
\label{T12}
 If Assumption \ref{ass1} holds, 
  ${\rm VCdim}(\phi\circ \mathcal{H})<+\infty$ and $\sup_{{ h}\in \mathcal{H}} |\{\mathbf{x}\in \mathcal{X}: { h}(\mathbf{x})\in \mathcal{Y}\}|=+\infty$,
  OOD detection is not learnable under risk in the separate space $\mathscr{D}_{XY}^{s}$ for $\mathcal{H}$, where $\phi$ maps ID labels to ${1}$ and maps OOD labels to $2$.
\end{restatable}


\noindent The finite VC dimension normally implies the learnability of supervised learning. However, in our results, the finite VC dimension cannot guarantee the learnability of OOD detection under risk in the separate space, which reveals the difficulty of the OOD detection.

\subsection{Impossibility Theorems under AUC} 

We then study whether Condition \ref{C1_auc} holds in the total space $\mathscr{D}_{XY}^{\rm all}$. If Condition \ref{C1_auc} does not hold, then OOD detection is not learnable under AUC. We first present Lemma~\ref{T5_auc} to point out when Condition \ref{C1_auc} does not hold. 

\begin{restatable}{lemma}{thmImpOneAuc}
\label{T5_auc}
Given a ranking function space $\mathcal{R}$, a domain space $\mathscr{D}_{XY}$ and $D_{XY}=\beta D_{X_{\rm I}Y_{\rm I}}+(1-\beta)D_{X_{\rm O}Y_{\rm O}},$ $D_{XY}'= \beta' D_{X_{\rm I}Y_{\rm I}}+(1-\beta')D_{X_{\rm O}Y_{\rm O}}' \in \mathscr{D}_{XY}$, let $P$ be the overlap set between $D_{X_{\rm I}}$ and $D_{X_{\rm O}}$ and $P'$ be the overlap set between $D_{X_{\rm I}}$ and $D_{X_{\rm O}}'$ based on the Definition~\ref{def:overlap}. If
\begin{align*}
        \sup_{r\in \mathcal{R}} {\rm AUC}(r; D_{X_{\rm I}},D_{X_{\rm O}}) &= \sup_{r\in \mathcal{R}_{\rm all}} {\rm AUC}(r; D_{X_{\rm I}},D_{X_{\rm O}})\\
        \sup_{r\in \mathcal{R}} {\rm AUC}(r; D_{X_{\rm I}},D_{X_{\rm O}}') &= \sup_{r\in \mathcal{R}_{\rm all}} {\rm AUC}(r; D_{X_{\rm I}},D_{X_{\rm O}}'),
\end{align*}
and $D_{X_{\rm I}}(P\cap P')<\min\{ D_{X_{\rm I}}(P), D_{X_{\rm I}}(P') \}$, then Condition \ref{C1_auc} does not hold, where $\mathcal{R}_{\rm all}$ is a ranking function space consisting of all ranking functions from $\mathcal{X}$ to $\mathbb{R}$. Therefore, OOD detection is not learnable under AUC in $\mathscr{D}_{XY}$ for $\mathcal{R}$.
\end{restatable}

\noindent Based on Lemma~\ref{T5_auc}, we know that, under proper conditions, Condition~\ref{C1_auc} does not hold once there is one domain whose ID and OOD distributions overlap. Then, based on Lemma~\ref{T5_auc}, 
we can obtain that the OOD detection is not learnable in the total space $\mathscr{D}^{\rm all}_{XY}$ for any non-trivial ranking function space $\mathcal{R}$.

\begin{restatable}[Impossibility Theorem for Total Space under AUC]{theorem}{thmImpTotalauc}
\label{T4_auc}
    Given ranking function space $\mathcal{R}$, if there exist $\mathbf{x}, \mathbf{x}'\in \mathcal{X}$ and $r, r' \in \mathcal{R}$ such that 
\begin{equation*}
r(\mathbf{x})>r(\mathbf{x}')~ \text{and}~ r'(\mathbf{x}')>r'(\mathbf{x}),
    \end{equation*}
    then the learnability of OOD detection under AUC is not distribution-free for $\mathcal{R}$.
\end{restatable}

\noindent From Lemma~\ref{T5_auc}, we know that the overlap between $D_{X_{\rm I}}$ and $D_{X_{\rm O}}$ is an important factor to influence the learnability of OOD detection under AUC. Thus, similar to the situation under risk, we want to study the learnability of OOD detection under AUC in separate space $\mathscr{D}_{XY}^s$ first. Before introducing the impossibility theorem for separate space, we need a mild assumption demonstrated below.

\begin{assumption}[Separate Space for OOD under AUC]\label{ass1_auc}
A ranking function space $\mathcal{R}$ is called separate ranking function space, if for any $\mathbf{x}\in \mathcal{X}$, there exists $r_{\mathbf{x}}\in \mathcal{R}$ such that $r_{\mathbf{x}}(\mathbf{x})<r_{\mathbf{x}}(\mathbf{x}')$, for any $\mathbf{x}'\in \mathcal{X}-\{\mathbf{x}\}$.
\end{assumption}

\noindent Note that, the above assumption is weak and can be satisfied by some well-known spaces (see Propositions~\ref{prop_fcnn_rank1} and \ref{prop_score_rank}). The above assumption means that, for any data point $\mathbf{x}$, its ranking can be the lowest one compared to other data points in the space $\mathcal{X}$. Finally, we use \textit{Vapnik–Chervonenkis} (VC) dimension \citep{mohri2018foundations} to help measure the size of ranking function space, and study the learnability of OOD detection under AUC in $\mathscr{D}_{XY}^{s}$ with the help the VC dimension. 

\begin{restatable}[Impossibility Theorem for Separate Space under AUC]{theorem}{thmImpSeptwoauc}
\label{T12_auc}
Given a separate ranking function space $\mathcal{R}$, if ${\rm VC}[\phi\circ \mathcal{R}]=d<+\infty$ and $|\mathcal{X}|\geq (28d+14)\log(14d+7)$,
then OOD detection is not learnable under AUC in $\mathscr{D}_{XY}^s$ for $\mathcal{R}$, where
$
    \phi\circ \mathcal{R} = \{\mathbf{1}_{r(\mathbf{x})>r(\mathbf{x}')}: r\in \mathcal{R}\}.
$
\end{restatable}

\noindent Based on Theorem~\ref{T12_auc}, we obtain a similar result to the learnability of OOD detection under risk: the finite VC dimension cannot guarantee the learnability of OOD detection
under AUC in the separate space, which further reveals the difficulty of OOD detection. Although the above impossibility theorems (under risk and AUC) are frustrating, there is still room to discuss the conditions in Theorem~\ref{T12} and Theorem~\ref{T12_auc}, and to find out the proper conditions for ensuring the learnability of OOD detection under risk and AUC in the separate space (see the following section).

\section{When OOD Detection Can Be Successful}\label{S6}
Here, we discuss when the OOD detection can be learnable under risk/AUC in different spaces. 
We first study the separate space $\mathscr{D}_{XY}^s$.

\subsection{OOD Detection in the Separate Space} 

Both Theorem \ref{T12} and Theorem~\ref{T12_auc} have indicated that ${\rm VCdim}(\phi\circ \mathcal{H})=+\infty$ or $\sup_{{ h}\in \mathcal{H}} |\{\mathbf{x}\in \mathcal{X}: { h}(\mathbf{x})\in \mathcal{Y}\}|<+\infty$ (or $\sup_{{ r}\in \mathcal{R}} |\{\mathbf{x}\in \mathcal{X}: {r}(\mathbf{x})\in \mathcal{R}\}|<+\infty$ under AUC metric) is necessary to ensure the learnability of OOD detection under risk or AUC in $\mathscr{D}_{XY}^{s}$ if Assumption \ref{ass1} or Assumption \ref{ass1_auc} holds. However, generally, hypothesis spaces generated by feed-forward neural networks with proper activation functions have finite VC dimension \citep{Peter2019Nearly,Karpinski1997polynomial}. 
Therefore, we study the learnability of OOD detection in the case that $|\mathcal{X}|<+\infty$, which implies that $\sup_{{ h}\in \mathcal{H}} |\{\mathbf{x}\in \mathcal{X}: { h}(\mathbf{x})\in \mathcal{Y}\}|<+\infty$ under risk metric or $\sup_{{ r}\in \mathcal{R}} |\{\mathbf{x}\in \mathcal{R}: {r}(\mathbf{x})\in \mathcal{X}\}|<+\infty$ under AUC metric. Additionally, {Theorem}~\ref{T24} also implies that $|\mathcal{X}|<+\infty$ is the necessary and sufficient condition for the learnability of OOD detection under risk in a separate space, when the hypothesis space is generated by FCNN. Hence, $|\mathcal{X}|<+\infty$ may be necessary in the space $\mathscr{D}_{XY}^s$.
\\

\paragraph{Learnability under Risk.} For simplicity, we first discuss the case that $K = 1$, \textit{i.e.}, the one-class novelty detection.
We show the {necessary and sufficient} condition for the learnability of OOD detection under risk in $\mathscr{D}_{XY}^{s}$, when $|\mathcal{X}|<+\infty$.

%
\begin{restatable}{theorem}{thmPosbSep}
\label{T13}
Let $K=1$ and $|\mathcal{X}|<+\infty$. Suppose that Assumption \ref{ass1} holds and the constant function $h^{\rm in}:=1 \in \mathcal{H}$. Then OOD detection is learnable under risk in $\mathscr{D}_{XY}^{s}$ for $\mathcal{H}$ \textbf{if and only if} $\mathcal{H}_{\rm all}-\{h^{\rm out}\} \subset \mathcal{H}$, where $\mathcal{H}_{\rm all}$ is the hypothesis space consisting of all hypothesis functions, and $h^{\rm out}$ is a constant function that $h^{\rm out}:=2$, here $1$ represents ID data and $2$ represents OOD data.
\end{restatable}

\noindent The condition $h^{\rm in}\in \mathcal{H}$ presented in Theorem \ref{T13} is mild. {Many practical hypothesis spaces satisfy this condition, {\textit{e.g.}, the FCNN-based hypothesis space} (Proposition \ref{Pr1} in Appendix \ref{SK}), score-based hypothesis space  (Proposition \ref{P2} in Appendix \ref{SK}) and universal kernel-based hypothesis space.} Theorem \ref{T13} implies that if $K=1$ and OOD detection is learnable under risk in $\mathscr{D}^{s}_{XY}$ for $\mathcal{H}$, then the hypothesis space $\mathcal{H}$ should contain almost all hypothesis functions, implying that if the OOD detection can be learnable under risk in the distribution-agnostic case, then a large-capacity model is necessary.
\\
\vspace{-0.7em}
\\
\noindent Next, we extend Theorem \ref{T13} to a general case, \textit{i.e.}, $K>1$. 
When $K>1$, we will first use a binary classifier $h^b$ to classify the ID and OOD data. Then, for the ID data identified by $h^b$, an ID hypothesis function $h^{\rm in}$ will be used to classify them into corresponding ID classes. 
We state this strategy as follows: given a hypothesis space $\mathcal{H}^{\rm in}$ for ID distribution and a binary classification hypothesis space $\mathcal{H}^{\rm b}$ introduced in Section \ref{S3}, we use $\mathcal{H}^{\rm in}$ and $\mathcal{H}^{\rm b}$ to construct an OOD detection's hypothesis space $\mathcal{H}$, which consists of all hypothesis functions $h$ satisfying the following condition: there exist $h^{\rm in}\in \mathcal{H}^{\rm in}$ and $h^{\rm b}\in \mathcal{H}^b$ such that $\forall \mathbf{x}\in \mathcal{X}$,
\begin{equation}\label{Eq.dot}
    h(\mathbf{x}) = i,~~\textnormal{ if}~h^{\rm in}(\mathbf{x})=i~ \textnormal{ and }~h^{\rm b}(\mathbf{x})=1; \textnormal{ otherwise},  ~h(\mathbf{x}) =  K+1.
\end{equation}
We use $\mathcal{H}^{\rm in} \bullet \mathcal{H}^{\rm b} $ to represent a hypothesis space consisting of all $h$ defined in Eq. (\ref{Eq.dot}). In addition, we also need an additional condition for the loss function $\ell$, shown as follows:

\begin{Condition}\label{C3}
$
\ell(y_2,y_1)\leq \ell(K+1,y_1)$, for any in-distribution labels $y_1$ and $y_2\in \mathcal{Y}.
$
\end{Condition}
\begin{restatable}{theorem}{thmPosbMultione}
\label{T17}
Let $|\mathcal{X}|<+\infty$ and $\mathcal{H}=\mathcal{H}^{\rm in} \bullet \mathcal{H}^{\rm b}$. If $\mathcal{H}_{\rm all}-\{h^{\rm out}\} \subset \mathcal{H}^{\rm b}$ and Condition \ref{C3} holds, then OOD detection is learnable under risk in $\mathscr{D}_{XY}^{s}$ for $\mathcal{H}$, where $\mathcal{H}_{\rm all}$ and $h^{\rm out}$ are defined in Theorem \ref{T13}.
\end{restatable}



\paragraph{Learnability under AUC.} Then, we study the learnability of OOD detection under AUC in the separate space. Here we require to introduce a basic assumption in learning theory for AUC---\textit{AUC-based Realizability Assumption}, \textit{i.e.}, for any $D_{XY}\in \mathscr{D}_{XY}$, there exists $r^*\in \mathcal{R}$ such that ${\rm AUC}(r^*;D_{X_{\rm I}},D_{X_{\rm O}})=1$ (see Appendix \ref{SB.2}). Based on this AUC-based Realizability Assumption, we prove the following theorem. 

\begin{restatable}{theorem}{thmPosbSepauc}
\label{T13_auc}
Given a separate ranking function space $\mathcal{R}$, if $|\mathcal{X}|<+\infty$,
then OOD detection is learnable under AUC in the separate space $\mathscr{D}_{XY}^{s}$ for $\mathcal{R}$ \textbf{if and only if} AUC-based Realizability Assumption holds. 
\end{restatable}

\noindent Theorem \ref{T13_auc} indicates the significance of AUC-based Realizability Assumption in OOD detection under AUC, which also means that a large ranking function space is essential for the success of OOD detection under AUC. 

\subsection{OOD Detection in the Finite-ID-Distribution Space}

Since researchers can only collect finite ID datasets as the training data in the process of algorithm design, it is worthy to study the learnability of OOD detection under risk in the finite-ID-distribution space $\mathscr{D}_{XY}^F$. We first show two necessary concepts below. 
\begin{Definition}[ID Consistency] Given a domain space $\mathscr{D}_{XY}$, we say any two domains $D_{XY}\in \mathscr{D}_{XY}$ and $D_{XY}'\in \mathscr{D}_{XY}$ are ID consistency, if $D_{X_{\rm I}Y_{\rm I}}=D_{X_{\rm I}Y_{\rm I}}'$. We use $\sim$ to represent the ID consistency, i.e., $D_{XY}\sim D_{XY}'$ if and only if $D_{XY}$ and $D_{XY}'$ are ID consistency.
\end{Definition}
It is easy to check that the ID consistency $\sim$ is an equivalence relation. Therefore, we define the set $[D_{XY}]:=\{D_{XY}'\in \mathscr{D}_{XY}:D_{XY}\sim D_{XY}'\}$ as the equivalence class regarding $\mathscr{D}_{XY}$.


\begin{Condition}[Compatibility]\label{Con2}
    For any equivalence class $[D_{XY}']$ with respect to $\mathscr{D}_{XY}$ and any $\epsilon>0$, there exists a hypothesis function $h_{\epsilon}\in \mathcal{H}$ such that for any domain $D_{XY}\in [D_{XY}']$,
 \begin{equation*}
 h_{\epsilon}\in \{ h' \in \mathcal{H}: R_D^{\rm out}(h') \leq \inf_{h\in \mathcal{H}} R_D^{\rm out}(h)+\epsilon\} \cap   \{ h' \in \mathcal{H}: R_D^{\rm in}(h') \leq \inf_{h\in \mathcal{H}} R_D^{\rm in}(h)+\epsilon\}.
 \end{equation*}
\end{Condition}

\noindent In Appendix \ref{SD}, Lemma \ref{Lemma1} has implied that Condition \ref{Con2} is a general version of Condition \ref{C1}. Next, Theorem \ref{T-SET} shows that Condition \ref{Con2} is the \textit{necessary and sufficient condition} in $\mathscr{D}_{XY}^F$.

%
\begin{restatable}{theorem}{thmPosbtennew}\label{T-SET}
Suppose that $\mathcal{X}$ is bounded. OOD detection is learnable under risk in $\mathscr{D}_{XY}^{F}$ for $\mathcal{H}$ \textbf{if and only if} the compatibility condition (i.e., Condition \ref{Con2}) holds. Furthermore, the learning rate $\epsilon_{\rm cons}(n)$ can attain ${O}(1/\sqrt{n^{1-\theta}})$, for any $\theta\in (0,1)$.
\end{restatable}

\noindent Theorem~\ref{T-SET} shows that, in the process of algorithm design, OOD detection cannot be successful without the compatibility condition if we use risk to evaluate the performance. Theorem \ref{T-SET} also implies that Condition \ref{Con2} is essential for the learnability of OOD detection under risk. This motivates us to study whether OOD detection can be successful in more general spaces (\textit{e.g.}, the density-based space), when the compatibility condition holds. 
\\
\vspace{-0.8em}
\\
\noindent As for the learnability of OOD detection under AUC in the finite-ID-distribution space, since Condition~\ref{C1_auc} only considers linearity between OOD distributions instead of OOD and ID distributions as shown in  Condition~\ref{C1}. To further reveal the learnability of OOD detection under AUC in the finite-ID-distribution space, we might need to discover a new condition for compatibility w.r.t. OOD and ID distributions to extend Condition~\ref{C1_auc}.


\subsection{OOD Detection in the Density-based Space} 

\paragraph{Learnability under Risk.} To ensure that Condition \ref{Con2} holds, we consider a basic assumption in learning theory---\textit{Risk-based Realizability Assumption} (see Appendix \ref{SB.2}), \textit{i.e.}, for any $D_{XY}\in \mathscr{D}_{XY}$, there exists $h^*\in \mathcal{H}$ such that $R_D(h^*)=0$. We discover that in the density-based space $\mathscr{D}_{XY}^{\mu,b}$, Risk-based Realizability Assumption can conclude the compatibility condition (Condition \ref{Con2}). Based on this observation, we prove the following theorem:

\begin{restatable}{theorem}{thmPosbtennewtwo}\label{T-SET2}
Given a density-based space $\mathscr{D}_{XY}^{\mu,b}$, if $\mu(\mathcal{X})<+\infty$, the Risk-based Realizability Assumption holds, then when $\mathcal{H}$ has finite Natarajan dimension \citep{shalev2014understanding}, OOD detection is learnable in $\mathscr{D}_{XY}^{\mu,b}$ for $\mathcal{H}$. Furthermore, the learning rate $\epsilon_{\rm cons}(n)$ can attain ${O}(1/\sqrt{n^{1-\theta}})$, for any $\theta\in (0,1)$.
\end{restatable}

\noindent To further investigate the importance and necessary of Risk-based Realizability Assumption, Theorem \ref{T24.3} has indicated that in some practical scenarios, Risk-based Realizability Assumption is the necessary and sufficient condition for the learnability of OOD detection under risk in the density-based space. Therefore, Risk-based Realizability Assumption may be indispensable for the learnability of OOD detection under risk in some practical scenarios.

\paragraph{Learnability under AUC.} To study the learnability of OOD detection under AUC in the density-based space, we first need to introduce a constant-closure assumption for $\mathcal{R}$.

\begin{assumption}
We say a ranking function space $\mathcal{R}$ is 
constant closure, if for any $r\in \mathcal{R}$, the constant function space $\overline{r(\mathcal{X})}:= {\{c: r(\mathbf{x})=c,~\forall \mathbf{x}\in \mathcal{X}, r\in \mathcal{R}\}}\subset \mathcal{R}$.
\end{assumption}
\noindent Note that, the above assumption is weak and can be satisfied by some well-known ranking function space (see Propositions~\ref{prop_fcnn_rank1} and \ref{prop_score_rank}). Based on this assumption, we give a sufficient condition for learnability of OOD detection under AUC in the density-based space:
\begin{restatable}{theorem}{thmPosbtennewtwoauc}\label{T-SET2_auc}
Suppose that $\mathcal{R}$ is constant closure, separate, and $\mu(\mathcal{X})<+\infty$. Given a density-based space $\mathscr{D}_{XY}^{\mu,b}$, if the AUC-based Realizability Assumption holds, then when ${\rm VC}[\phi\circ \mathcal{R}]<+\infty$, OOD detection is learnable under AUC in $\mathscr{D}_{XY}^{\mu,b}$ for $\mathcal{R}$, where $
    \phi\circ \mathcal{R} = \{\mathbf{1}_{r_1(\mathbf{x})>r_2(\mathbf{x}')}: r_1,r_2\in \mathcal{R}\}.
$ Furthermore, the learning rate $\epsilon_{\rm cons}(n)$ can attain ${O}(1/\sqrt{n^{1-\theta}})$, for any $\theta\in (0,1)$.
\end{restatable}

\noindent Based on Theorem~\ref{T-SET2_auc}, we find that the AUC-based Realizability Assumption is also important for the learnability of OOD detection under AUC in the density-based space.



\section{Connecting Theory to Practice}\label{S7}
In Section \ref{S6}, we have shown the successful scenarios where OOD detection problem can be addressed in theory under risk or AUC metric. In this section, we will discuss how the proposed theory is applied to two representative hypothesis spaces---neural-network-based spaces and score-based spaces. 

\subsection{Key Concepts Regarding Fully-connected Neural Networks}

\paragraph{Fully-connected Neural Networks.}
Given a sequence $\mathbf{q}=(l_1,l_2,...,l_g)$, where $l_i$ and $g$ are positive integers and $g>2$, we use $g$ to represent the \textbf{\textit{depth}} of neural network and use $l_i$ to represent the \textbf{\textit{width}} of the $i$-th layer. After the activation function $\sigma$ is selected\footnote{We consider the \emph{rectified linear unit} (ReLU) function as the default activation function $\sigma$, which is defined by $\sigma(x)=\max \{x,0 \}$,~$\forall$ $ x\in \mathbb{R}$. 
\textit{{We will not repeatedly mention the definition of $\sigma$ in the rest of our paper}}.
}, we can obtain the architecture of FCNN according to the sequence $\mathbf{q}$. Let $\mathbf{f}_{\mathbf{w},\mathbf{b}}$ be the function generated by FCNN with weights $\mathbf{w}$ and bias $\mathbf{b}$. An FCNN-based scoring function space is defined as:
$
    \mathcal{F}_{\mathbf{q}}^{\sigma}:=\{\mathbf{f}_{\mathbf{w},\mathbf{b}}:\forall~ \textnormal{weights}~\mathbf{w},~\forall~  \textnormal{bias}~ \mathbf{b}\}.
$ In addition, for simplicity, given any two sequences $\mathbf{q}=(l_1,...,l_g)$ and $\mathbf{q}'=(l_1',...,l_{g'}')$, we use the notation
$
    \mathbf{q}\lesssim\mathbf{q}'
$
to represent the  following equations and inequalities:
\begin{equation*}
1)~ g \leq g', l_1=l_1', l_g=l_{g'}'; ~~~~~2)~ l_i\leq l'_i,~\forall i=1,...,g-1;~~\textnormal{and}~~~3)~
l_{g-1}\leq l'_i, ~\forall i=g,...,g'-1.
\end{equation*}
Lemma \ref{L7-contain} shows $\mathbf{q}\lesssim\mathbf{q}'\Rightarrow \mathcal{F}_{\mathbf{q}}^{\sigma}\subset \mathcal{F}_{\mathbf{q}'}^{\sigma}$. We use $\lesssim$ to compare the sizes of FCNNs.

\paragraph{FCNN-based Hypothesis Space.}
Let $l_g=K+1$. The FCNN-based scoring function space $\mathcal{F}_{\mathbf{q}}^{\sigma}$ can induce an FCNN-based hypothesis space. For any $\mathbf{f}_{\mathbf{w},\mathbf{b}}\in \mathcal{F}_{\mathbf{q}}^{\sigma}$, the induced hypothesis function  is:
\begin{equation*}
    h_{\mathbf{w},\mathbf{b}}:=\argmax_{k\in\{1,...,K+1\}} {f}^k_{\mathbf{w},\mathbf{b}}, ~\textnormal{where}~{f}^k_{\mathbf{w},\mathbf{b}}~\textnormal{is ~the~}k\textnormal{-th}~\textnormal{coordinate~of}~\mathbf{f}_{\mathbf{w},\mathbf{b}}.
\end{equation*}
 Then, the FCNN-based hypothesis space is defined as
$
  \mathcal{H}_{\mathbf{q}}^{\sigma}:= \{h_{\mathbf{w},\mathbf{b}}:\forall~ \textnormal{weights}~\mathbf{w},~\forall~\textnormal{bias}~ \mathbf{b}\}.
$

\paragraph{FCNN-based Ranking Function Space.}
Then, based on the definition of FCNN, we show that, given a specific $\mathbf{q}$, under some mild conditions, $\mathcal{F}_{\mathbf{q}}^{\sigma}$ is the separate and constant closure ranking function space.
\begin{restatable}{Proposition}{propFcnnRankingOne}\label{prop_fcnn_rank1}
    Let $\mathcal{X}$ be a bounded feature space. Given $\mathbf{q}=(l_1,...,l_{g-1},1)$, then 
    \begin{itemize}
        \item if  some $s$ with $1<s<g$, $d=l_1\leq l_2\leq...\leq l_s$, and $l_s \geq 2d$, $\mathcal{F}_{\mathbf{q}}^{\sigma}$ is the separate ranking function space;
        \item $\mathcal{F}_{\mathbf{q}}^{\sigma}$ is constant closure;
        \item $\{\mathbf{1}_{f_1(\mathbf{x})<f_2(\mathbf{x}')}: f_1, f_2\in \mathcal{F}_{\mathbf{q}}^{\sigma}\}$ has finite VC dimension.
    \end{itemize}
\end{restatable}




\paragraph{Score-based Hypothesis Space.}
Many OOD detection algorithms detect OOD data by using a score-based strategy. That is, given a threshold $\lambda$, a scoring function space $\mathcal{F}_{l}\subset \{\mathbf{f}:\mathcal{X}\rightarrow \mathbb{R}^l\}$ and a score function $E: \mathcal{F}_l\rightarrow \mathbb{R}$, then $\mathbf{x}$ is regarded as ID data if and only if $E(\mathbf{f}(\mathbf{x}))\geq \lambda$. 
We introduce several representative score functions $E$ as follows: for any $\mathbf{f}=[f^1,...,f^l]^{\top}\in \mathcal{F}_{l}$,

\begin{itemize}
    \item softmax-based function \citep{Hendrycks2017abaseline}: $\lambda\in (\frac{1}{l},1)$ and $T>0$,
    \begin{equation}\label{score1}
    E(\mathbf{f}) =\max_{k\in\{1,...,l\}}  \frac{\exp{(f^k)}}{\sum_{c=1}^l \exp{(f^c)}};
    \end{equation}
    \item temperature-scaled function \citep{liang2018enhancing}: $\lambda\in (\frac{1}{l},1)$ and $T>0$,
    \begin{equation}\label{score2}
    E(\mathbf{f}) =\max_{k\in\{1,...,l\}}  \frac{\exp{(f^k/T)}}{\sum_{c=1}^l \exp{(f^c/T)}};
    \end{equation}
    \item energy-based function \citep{liu2020energy}: $\lambda\in (0,+\infty) $ and $T>0$,
\begin{equation}
    E(\mathbf{f}) =T\log \sum_{c=1}^l  \exp{(f^c/T)}.~~~ \label{score3}
\end{equation}
\end{itemize}


\noindent Using $E$, $\lambda$ and $\mathbf{f}\in \mathcal{F}_{\mathbf{q}}^{\sigma}$, we have a classifier: $h^{\lambda}_{\mathbf{f},E}(\mathbf{x})=1$, if $E(\mathbf{f}(\mathbf{x}))\geq \lambda$; otherwise, $h^{\lambda}_{\mathbf{f},E}(\mathbf{x})=2$, where $1$ represents the ID data and $2$ represents the OOD data. Hence, a binary classification hypothesis space $\mathcal{H}^b$, which consists of all $h^{\lambda}_{\mathbf{f},E}$, is generated. We define $
    \mathcal{H}^{{\sigma},\lambda}_{{\mathbf{q}},E} := \{h^{\lambda}_{\mathbf{f},E}:\forall \mathbf{f}\in \mathcal{F}_{\mathbf{q}}^{\sigma}\}
$.



\paragraph{Score-based Ranking Function Space.} Similar to the FCNN-based ranking function space, for several representative score functions $E$ (\textit{e.g.,}, Eqs~\eqref{score1}, \eqref{score2}, and \eqref{score3}), the FCNN-based score ranking function space ${E}\circ \mathcal{F}_{\mathbf{q}}^{\sigma}$ is separate, which is evidence that Assumption \ref{ass1_auc} is weak and can be easily satisfied.

\begin{restatable}{Proposition}{propScoreRanking}\label{prop_score_rank}
    Given $\mathbf{q}=(l_1,...,l_{g-1},l)$ and $\mathbf{q}'=(l_1,...,l_{g-1},1)$, let $\mathcal{R}={E}\circ \mathcal{F}_{\mathbf{q}}^{\sigma}$, then
    \begin{itemize}
        \item  if $\mathcal{F}_{\mathbf{q}'}^{\sigma}$ is a separate ranking function space, $\mathcal{R}$ is the separate ranking function space;
        \item $\mathcal{R}$ is constant closure;
        \item $\{\mathbf{1}_{r_1(\mathbf{x})<r_2(\mathbf{x}')}, r_1, r_2\in \mathcal{R}\}$ has finite VC dimension,
    \end{itemize}
    where $E$ is Eq. \eqref{score1}, \eqref{score2} or \eqref{score3}.
\end{restatable}

\noindent According to the previous section, we find that FCNN-based ranking function space and FCNN-based score ranking function space can satisfy almost all conditions in theorems.




\subsection{Learnability of OOD Detection in Different Spaces}

Next, we present applications of our theory regarding the above practical and important hypothesis spaces and ranking function spaces.


\begin{restatable}{theorem}{thmAppFCNN}
\label{T24}
 Suppose that Condition \ref{C3} holds and the hypothesis space $\mathcal{H}$ is FCNN-based or score-based, i.e., $\mathcal{H}=\mathcal{H}_{\mathbf{q}}^{\sigma}$ or  $\mathcal{H}=\mathcal{H}^{\rm in}\bullet \mathcal{H}^{\rm b}$, where  $\mathcal{H}^{\rm in}$ is an ID hypothesis space, $\mathcal{H}^{\rm b}=\mathcal{H}^{{\sigma},\lambda}_{{\mathbf{q}},E}$ and $\mathcal{H}=\mathcal{H}^{\rm in}\bullet \mathcal{H}^{\rm b}$ is introduced below Eq. \eqref{Eq.dot}, here
 $E$ is Eq. \eqref{score1}, \eqref{score2} or \eqref{score3}. Then
\\
\vspace{-0.6em}
\\
$~~~~~~~~~~~$\fbox{
\parbox{0.8\textwidth}{
There is a sequence $\mathbf{q}=(l_1,...,l_g)$  such that OOD detection is learnable under risk in the separate space $\mathscr{D}^s_{XY}$ for $\mathcal{H}$ \textbf{if and only if} $|\mathcal{X}|<+\infty$.
 }}
 \\
 \\
Furthermore, if $|\mathcal{X}|<+\infty$, then there exists a sequence $\mathbf{q}=(l_1,...,l_g)$ such that for any sequence $\mathbf{q}'$ satisfying that $\mathbf{q}\lesssim\mathbf{q}'$, OOD detection is learnable under risk in  $\mathscr{D}^s_{XY}$ for $\mathcal{H}$.
\end{restatable}

\noindent If we consider the ranking function space, we can obtain a similar theoretical result.

\begin{restatable}{theorem}{thmAppFCNNauc}
\label{T24_auc}
 Suppose the ranking function space $\mathcal{R}$ is separate, and FCNN-based or score-based, i.e., $\mathcal{R}=\mathcal{F}_{\mathbf{q}}^{\sigma}$ or  $\mathcal{R}=E\circ \mathcal{F}_{\mathbf{q}}^{\sigma}$, where $E$ is Eq. \eqref{score1}, \eqref{score2} or \eqref{score3}. Then
\\
\vspace{-0.6em}
\\
$~~~~~~~~~~~$\fbox{
 \parbox{0.8\textwidth}{
There is a sequence $\mathbf{q}=(l_1,...,l_g)$  such that OOD detection is AUC learnable in the separate space $\mathscr{D}^s_{XY}$ for $\mathcal{R}$ \textbf{if and only if} $|\mathcal{X}|<+\infty$.
 }}
 \\
 \\
Furthermore, if $|\mathcal{X}|<+\infty$, then there is a sequence $\mathbf{q}=(l_1,...,l_g)$ such that for any sequence $\mathbf{q}'$ satisfying that $\mathbf{q}\lesssim\mathbf{q}'$, OOD detection is learnable under AUC in  $\mathscr{D}^s_{XY}$ for $\mathcal{R}$.
\end{restatable}

\noindent Theorems~\ref{T24} and \ref{T24_auc} state that 1) when the hypothesis space or ranking function space is FCNN-based or score-based, the finite feature space is the necessary and sufficient condition for the learnability of OOD detection (under risk or AUC) in the separate space; and 2) a larger architecture of FCNN has a greater probability to achieve the learnability of OOD detection in the separate space.
Note that when we select  Eqs. \eqref{score1}, \eqref{score2}, or \eqref{score3} as the score function $E$, Theorems~\ref{T24} and \ref{T24_auc} also show that the selected score functions $E$ can guarantee the learnability of OOD detection (under risk or AUC), which is a theoretical support for the representative works \citep{liang2018enhancing,liu2020energy,Hendrycks2017abaseline}. Furthermore, Theorems \ref{T24.3} and \ref{T24.3_auc} also offer theoretical supports for these works in the density-based space. 

%


\begin{restatable}{theorem}{thmAppFCNNtwo}
\label{T24.3}
 Suppose that each domain $D_{XY}$ in $\mathscr{D}_{XY}^{\mu,b}$  is attainable, i.e., $\argmin_{h\in \mathcal{H}}R_{D}(h)\neq \emptyset$ (the finite discrete domains satisfy this). Let $K=1$ and the hypothesis space $\mathcal{H}$ be score-based $($$\mathcal{H}=\mathcal{H}_{\mathbf{q},E}^{\sigma,\lambda}$, where $E$ is in Eq. \eqref{score1}, \eqref{score2}, or \eqref{score3}$)$ or  FCNN-based $($$\mathcal{H}=\mathcal{H}_{\mathbf{q}}^{\sigma}$$)$.
If $\mu(\mathcal{X})<+\infty$, then the following four conditions are \textbf{equivalent}:
\\
$~~~~~~~~~~~~~~~~~~~~~~$\fbox{
 \parbox{0.6\textwidth}{
 Learnability in $\mathscr{D}_{XY}^{\mu,b}$ for  $\mathcal{H}$ $\iff$ Condition \ref{C1} 
 $\iff$
Risk-based Realizability Assumption $\iff$ Condition \ref{Con2}
}}
\end{restatable}
\vspace{-1em}

\begin{restatable}{theorem}{thmAppFCNNtwoauc}
\label{T24.3_auc}
 Suppose that the ranking function space $\mathcal{R}$ is separate and score-based $($$\mathcal{R}=E \circ \mathcal{F}_{\mathbf{q}}^{\sigma})$ or FCNN-based $($$\mathcal{R}= \mathcal{F}_{\mathbf{q}}^{\sigma})$, where $E$ is Eq. \eqref{score1}, \eqref{score2} or \eqref{score3}.
If $\mu(\mathcal{X})<+\infty$, then the following three conditions satisfy:
\\
\fbox{
 \parbox{0.98\textwidth}{
~~AUC-based Realizability Assumption $\Rightarrow$ Learnability in $\mathscr{D}_{XY}^{\mu,b}$ for  $\mathcal{R}$ $\Rightarrow$ Condition~\ref{C1_auc}
}}
\end{restatable}

\noindent Compared to Theorem~\ref{T24.3}, Theorem~\ref{T24.3_auc} cannot obtain the equivalence among Realizability Assumption, Learnability in $\mathscr{D}_{XY}^{\mu,b}$ and linear condition under AUC. The main reason is that Condition~\ref{C1_auc} is a weaker necessary condition for learnability of OOD detection under AUC than Condition~\ref{C1} for learnability of OOD detection under risk. We need to discover a strong necessary condition for learnability of OOD detection under AUC to obtain a similar equivalence that appeared in Theorem~\ref{T24.3}\footnote{A stronger necessary condition normally means that this necessary condition is closer to the necessary and sufficient condition.}. 

\subsection{Overlap and Benefits of Multi-class Case}
We investigate when the hypothesis space is FCNN-based or score-based, what will happen if there exists an overlap between the ID and OOD distributions?

\begin{restatable}{theorem}{thmImpOverlapFCNN}
\label{overlapcase}
Let $K=1$ and the hypothesis space $\mathcal{H}$ be score-based $($$\mathcal{H}=\mathcal{H}_{\mathbf{q},E}^{\sigma,\lambda}$, where $E$ is in Eq. \eqref{score1}, \eqref{score2}, or \eqref{score3}$)$ or FCNN-based $($$\mathcal{H}=\mathcal{H}_{\mathbf{q}}^{\sigma}$$)$.  Given a prior-unknown space $\mathscr{D}_{XY}$, if there exists a domain $D_{XY}\in \mathscr{D}_{XY}$, which has an overlap between ID and OOD distributions (see Definition \ref{def:overlap}), then OOD detection is not learnable under risk in $\mathscr{D}_{XY}$ for  $\mathcal{H}$.
\end{restatable}

\noindent When $K=1$ and the hypothesis space is FCNN-based or score-based, Theorem \ref{overlapcase} shows that overlap between ID and OOD distributions is the sufficient condition for the unlearnability of OOD detection under risk.
Theorem \ref{overlapcase} takes roots in the conditions $\inf_{h\in \mathcal{H}} R_{D}^{\rm in}(h)=0$ and $\inf_{h\in \mathcal{H}} R_{D}^{\rm out}(h)=0$. However, when $K>1$, we can ensure $\inf_{h\in \mathcal{H}}R_{D}^{\rm in}(h)>0$ if ID distribution $D_{X_{\rm I}Y_{\rm I}}$ has overlap between ID classes. By this observation, we conjecture that when $K>1$, OOD detection is learnable in some special cases where overlap exists, even if the hypothesis space is FCNN-based or score-based. As for the ranking function space, we can obtain a corresponding but weaker theoretical result shown below.

\begin{restatable}{theorem}{thmImpOverlapFCNNauc}
\label{overlapcase_auc}
Let the separate ranking function space $\mathcal{R}$ be FCNN-based or score-based (where the score function $E$ is Eq. \eqref{score1}, \eqref{score2}, or \eqref{score3}). Suppose that $D_{XY},D_{XY}'\in \mathscr{D}_{XY}$ are discrete distributions with $D_{X_{\rm I}Y_{\rm I}}=D_{X_{\rm I}Y_{\rm I}}$ and $D_{X_{\rm O}}=\delta_{\mathbf{x}}, D_{X_{\rm O}}'=\delta_{\mathbf{x}'}$. If $D_{X_{\rm O}}=\delta_{\mathbf{x}}, D_{X_{\rm O}}'=\delta_{\mathbf{x}'}$ have overlaps with $D_{X_{\rm I}Y_{\rm I}}$ and $D_{X_{\rm O}}\neq D_{X_{\rm O}}'$, then
 OOD detection is not learnable under AUC in $\mathscr{D}_{XY}$ for $\mathcal{R}$.
\end{restatable}

\section{Discussion}\label{sec:discuss}

\paragraph{Connections between Theorems under Risk and AUC.} In our previous conference paper \citep{fang2022is}, we mainly focus on the learnability of OOD detection under risk. However, in practice, AUC-related metrics are often used to evaluate the performance of OOD detection algorithms \citep{lin2021mood, Huang2021on, Fort2021exploring, ming2022spurious,DBLP:journals/corr/abs-2204-05306}. To fill this gap, we take a further step: investigating the learnability of OOD detection under AUC. Figure~\ref{fig:connection} illustrates the connections between the main theoretical results of learnability of OOD detection under risk and AUC. In the most of domain spaces considered in this paper, we can obtain similar theoretical results under AUC compared to theoretical results under risk. However, when considering the finite-ID-distribution space, we cannot get a corresponding theorem for the learnability of OOD detection under AUC. 
The main reason is that Condition~\ref{C1_auc} is a weaker necessary condition for learnability of OOD detection under AUC than Condition~\ref{C1} for learnability of OOD detection under risk. Thus, additional information might be required to obtain a stronger necessary condition for AUC learnability. The influence of Condition~\ref{C1_auc} also appears in Theorem~\ref{T24.3_auc} where we cannot obtain a strong theoretical result like Theorem~\ref{T24.3}. To conclude, since Condition~\ref{C1} is stronger (i.e., it is closer to necessary and sufficient condition) than Condition~\ref{C1_auc}, the theoretical results under risk are stronger than those under AUC. In the future, we need to discover a stronger necessary condition for learnability of OOD detection under AUC.


\begin{figure*}[!t]
    \centering
    {\includegraphics[width=0.7\linewidth]{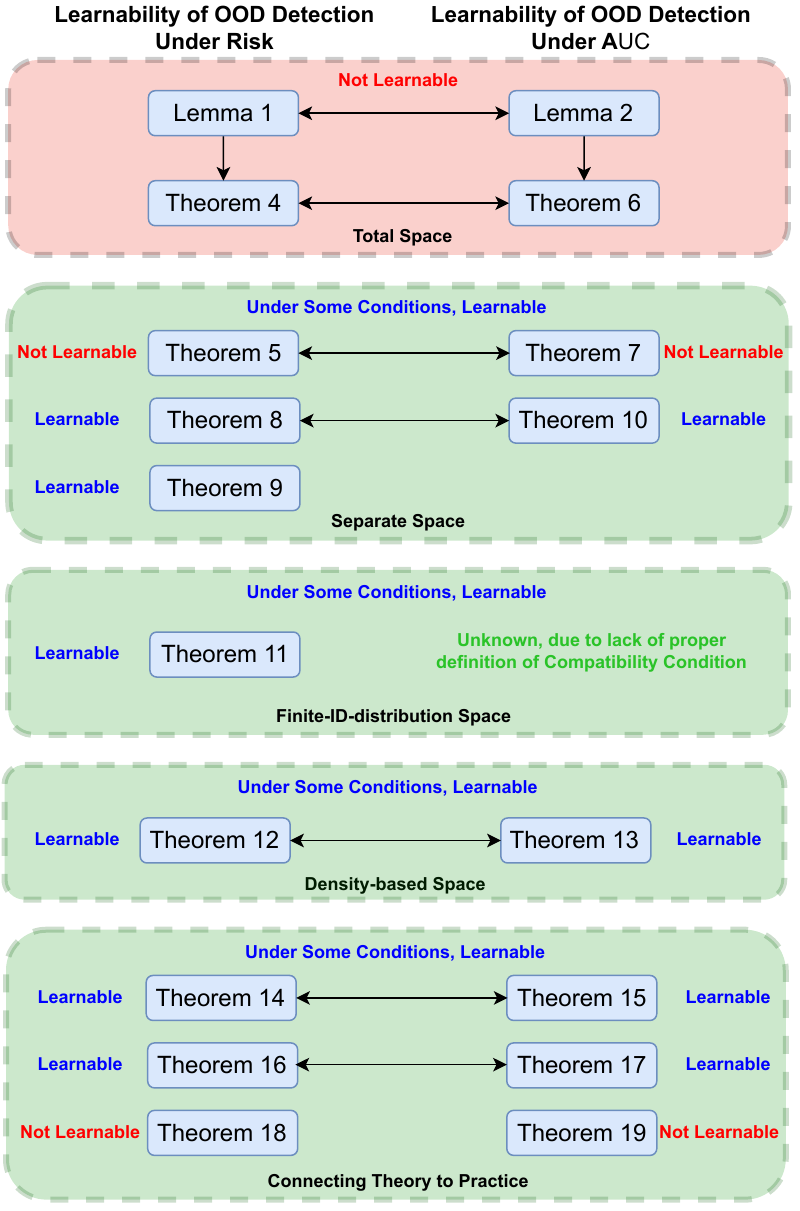}}
    \caption{Connections among main theoretical results in this paper. Compared to risk, AUC has a more strict requirement for the classification. Perfect classification (i.e., accuracy is 100\%) does not imply perfect AUC, which is the reason why theories regarding risk and AUC are very different.}\label{fig:connection}
    \vspace{-1em}
\end{figure*}

\paragraph{Understanding Far-OOD Detection.} Many existing works \citep{Hendrycks2017abaseline,DBLP:journals/corr/abs-2204-05306} study the far-OOD detection issue. Existing benchmarks include 1) MNIST \citep{DBLP:journals/spm/Deng12} as ID dataset, and Texture \citep{kylberg2011kylberg}, CIFAR-$10$ \citep{krizhevsky2010convolutional} or Place$365$ \citep{DBLP:journals/pami/ZhouLKO018} as OOD datasets; and  2) CIFAR-$10$ \citep{krizhevsky2010convolutional} as ID dataset, and MNIST \citep{DBLP:journals/spm/Deng12}, or Fashion-MNIST \citep{DBLP:journals/pami/ZhouLKO018} as OOD datasets. 
In far-OOD case, we find that the ID and OOD datasets have different semantic labels and different styles. From the theoretical view, we can define far-OOD detection tasks as follows: for $\tau>0$, a domain space $\mathscr{D}_{XY}$ is $\tau$-far-OOD, if for any domain $D_{XY}\in \mathscr{D}_{XY}$, 
\begin{equation*}
    {\rm dist}({{\rm supp} D_{X_{\rm O}}},{{\rm supp} D_{X_{\rm I}}})>\tau.
\end{equation*}
Theorems \ref{T17}, \ref{T-SET} and \ref{T24} imply that under appropriate hypothesis space, $\tau$-far-OOD detection is learnable under risk. Theorem~\ref{T24_auc} implies that under appropriate hypothesis space, $\tau$-far-OOD detection is learnable under AUC. In Theorem \ref{T17}, the condition $|\mathcal{X}|<+\infty$ is necessary for the separate space. However, one can prove that in the  far-OOD case, when $\mathcal{H}^{\rm in}$ is agnostic PAC learnable for ID distribution, the results in  Theorem \ref{T17} still holds, if the condition $|\mathcal{X}|<+\infty$ is replaced by a weaker condition that $\mathcal{X}$ is compact. In addition, it is notable that when $\mathcal{H}^{\rm in}$ is agnostic PAC learnable for ID distribution and $\mathcal{X}$ is compact, the KNN-based OOD
detection algorithm \citep{sun2022knn} is consistent in the  $\tau$-far-OOD case.

\paragraph{Understanding Near-OOD Detection.} When the ID and OOD datasets have similar semantics or styles,  OOD detection tasks become more challenging. \citep{ren2021asimplefix,DBLP:conf/nips/FortRL21} consider this issue and name it near-OOD detection. Existing benchmarks include  1) MNIST \citep{DBLP:journals/spm/Deng12} as ID dataset, and Fashion-MNIST \citep{DBLP:journals/pami/ZhouLKO018} or Not-MNIST \citep{notmnist} as OOD datasets; and 2) CIFAR-$10$ \citep{krizhevsky2010convolutional} as ID dataset, and CIFAR-$100$ \citep{cifar} as OOD dataset. From the theoretical view, some near-OOD tasks may imply the overlap condition, \textit{i.e.} Definition \ref{def:overlap}. Therefore, Lemma \ref{T5} and Theorem \ref{overlapcase} imply that near-OOD detection may be not learnable under risk, and Lemma \ref{T5_auc} and Theorem \ref{overlapcase_auc} imply that near-OOD detection may be not learnable under AUC.  
Developing a theory to understand the feasibility of near-OOD detection is an \textit{open question}.

\section{Related Work}

\textbf{OOD Detection Algorithms.} 
We will briefly review many representative OOD detection algorithms in three categories.
1) Classification-based methods use an ID classifier to detect OOD data \citep{Hendrycks2017abaseline}\footnote{Note that, some methods assume that OOD data are available in advance \citep{Hendrycks2019deep,Dhamija2018reducing}. 
However, the exposure of OOD data is a strong assumption \citep{Yang2021generalized}. We do not consider this situation in our paper.}. Representative works consider using the maximum softmax score \citep{Hendrycks2017abaseline}, temperature-scaled score \citep{ren2019likelihood} and energy-based score \citep{liu2020energy,Wang2021can} to identify OOD data. 
2) Density-based methods aim to estimate an ID distribution and identify the low-density area as OOD data \citep{zong2018deep}. 3) The recent development of generative models provides promising ways to make them successful in OOD detection \citep{Pidhorskyi2018generative,Nalisnick2019do,ren2019likelihood,Kingma2018glow,Xiao2020likelihood}.
Distance-based methods are based on the assumption that OOD data should be relatively far away from the centroids of ID classes \citep{lee2018asimple}, including Mahalanobis distance \citep{lee2018asimple,ren2021asimplefix}, cosine similarity \citep{zaeemzadeh2021out}, and kernel similarity \citep{Amersfoort2020uncertainty}.
\\
\vspace{-0.8em}
\\
\noindent Early works consider using the maximum softmax score to express the ID-ness \citep{Hendrycks2017abaseline}. Then, temperature scaling functions are used to amplify the separation between the ID and OOD data \citep{ren2019likelihood}. Recently, researchers propose hyperparameter-free energy scores to improve the OOD uncertainty estimation \citep{liu2020energy,Wang2021can}. 
Additionally, researchers also consider using the gradients to help improve the performance of OOD detection \citep{Huang2021on}.
\\
\vspace{-0.8em}
\\
\noindent Except for the above algorithms, researchers also study the situation, where auxiliary OOD data can be obtained during the training process \citep{Hendrycks2019deep,Dhamija2018reducing}. These methods are called outlier exposure, and have much better performance than the above methods due to the appearance of OOD data. However, the exposure of OOD data is a strong assumption \citep{Yang2021generalized}. Thus, researchers also consider generating OOD data to help the separation of OOD and ID data \citep{Vernekar2019out}. In this paper, we do not make an assumption that OOD data are available during training, since this assumption may not hold in real world.

\paragraph{OOD Detection Theory.}
\cite{zhang2021understanding} rejects the typical set hypothesis, the claim that relevant OOD distributions can lie in high likelihood regions of data distribution, as implausible.
\cite{zhang2021understanding} argues that minimal density estimation errors can lead to OOD detection failures {without assuming} an overlap between ID and OOD distributions. Compared to \cite{zhang2021understanding}, our theory focuses on the PAC learnable theory of OOD detection. If detectors are generated by FCNN, our theory (Theorem~\ref{overlapcase}) shows that the overlap is the sufficient condition to the failure of learnability of OOD detection, which is complementary to \cite{zhang2021understanding}. In addition, we identify several necessary and sufficient conditions for the learnability of OOD detection, which opens a door to studying OOD detection in theory.
Beyond \cite{zhang2021understanding}, \cite{morteza2022provable} paves a new avenue to designing provable OOD detection algorithms. Compared to \cite{morteza2022provable}, our paper aims to characterize the learnability of OOD detection to answer the question: is OOD detection PAC learnable?


\paragraph{Open-set Learning Theory.} \cite{liu2018open} is the first to propose the agnostic PAC guarantees for open-set detection. Unfortunately, the test data must be used during the training process. \cite{Fang2020Open} considers the open-set domain adaptation (OSDA)  \citep{Luo2020progressive} and proposes the first learning bound for OSDA. \cite{Fang2020Open} mainly depends on the positive-unlabeled learning techniques \citep{kiryo2017positive,Ishida2018binary,chen2021large}. However, similar to \cite{liu2018open}, the test data must be available during training. To study open-set learning (OSL) \textit{without accessing the test data} during training, \cite{Fang2021learning} proposes and studies the {almost} PAC learnability for OSL, which is motivated by transfer learning \citep{What_Transferred_Dong_CVPR2020,9695325}. Recently, \cite{wang2022openauc} proposes a novel AUC-based OOD detection objective named OpenAUC \citep{10214340,10264106} as objective function to learn open-set predictors, and builds a corresponding AUC-based open-set learning theory. In our paper, we study the PAC learnability for OOD detection, which is an {open problem} proposed by \cite{Fang2021learning}.

\paragraph{Learning Theory for Classification with Reject Option.} Many works \citep{DBLP:journals/tit/Chow70,DBLP:journals/corr/abs-2101-12523} also investigate the \emph{classification with reject option} (CwRO) problem, which is similar to OOD detection in some cases.  \cite{cortes2016learning,DBLP:conf/nips/CortesDM16,DBLP:conf/nips/NiCHS19,DBLP:conf/icml/CharoenphakdeeC21,DBLP:journals/jmlr/BartlettW08} study the learning theory and propose the agnostic PAC learning bounds for CwRO. However, compared to our work regarding OOD detection, existing CwRO theories mainly focus on how the ID risk (\textit{i.e.}, the risk that ID data is wrongly classified) is influenced by special rejection rules. Our theory not only focuses on the ID risk, but also pays attention to the OOD risk.

\paragraph{Robust Statistics.} In the field of robust statistics \citep{rousseeuw2011robust}, researchers aim to propose estimators and testers that can mitigate the negative effects of outliers (similar to OOD data). The proposed estimators are supposed to be independent of the potentially high dimensionality of the data \citep{ronchetti2009robust,diakonikolas2020outlier,diakonikolas2021outlier}. Existing works \citep{diakonikolas2021outlier_robust,cheng2021outlier,diakonikolas2022outlier} in the field have identified and resolved the statistical limits of outlier robust statistics by constructing estimators and proving impossibility results. In the future, it is a promising and interesting research direction to study the robustness of OOD detection based on robust statistics.

\paragraph{PQ Learning Theory.} Under some conditions, PQ learning theory \citep{DBLP:conf/nips/GoldwasserKKM20,DBLP:conf/alt/KalaiK21}  can be regarded
as the PAC theory for OOD detection in the semi-supervised
or transductive learning cases, \emph{i.e., test data are required during the training process.} Additionally,  PQ learning theory in \cite{DBLP:conf/nips/GoldwasserKKM20,DBLP:conf/alt/KalaiK21} aims to give the PAC estimation under Realizability Assumption \citep{shalev2014understanding}. Our theory focuses on the
PAC theory in different cases, which is more difficult and more practical than PAC theory under Realizability Assumption.  


\section{Conclusions and Future Works}

OOD detection has become an important technique to increase the reliability of machine learning. 
However, its theoretical foundation is merely investigated, which hinders real-world applications of OOD detection algorithms.
This paper is the \emph{first} to provide the PAC theory for OOD detection in terms of two commonly used metrics: risk and AUC.
Our results imply that a universally consistent algorithm might not exist for all scenarios in OOD detection. 
Yet, we still discover some scenarios where OOD detection is learnable under risk or AUC metrics. 
Our theory reveals many necessary and sufficient conditions for the learnability of OOD detection under risk or AUC, hence \emph{paving a road} to studying the learnability of OOD detection. In the future,  we will focus on studying the robustness of OOD detection based on robust statistics  \citep{diakonikolas2021outlier_robust,DBLP:journals/corr/abs-1911-05911}.

\acks{JL and ZF were supported by the Australian Research Council (ARC) under
FL190100149. YL is supported by National Science Foundation (NSF) Award No. IIS-2237037.  FL was supported by the ARC with grant numbers DP230101540 and DE240101089, and the NSF\&CSIRO Responsible AI program with grant number 2303037. ZF would also like to thank Prof. Peter Bartlett, Dr. Tongliang Liu and Dr. Zhiyong Yang for productive discussions. }


\vskip 0.2in
\bibliography{sample}

\newpage

\appendix

\DoToC
\newpage

\section{Notations}\label{SB}
\subsection{Main Notations and Their Descriptions}\label{SB.1}
In this section, we summarize important notations in Table \ref{Table_conncept}.
\begin{table}[t]
\caption{{Main notations and their descriptions.}}
\begin{center}\label{Table_conncept}
\scriptsize
\begin{tabular}{p{5cm}p{8.2cm}}
\hline
Notation & ~~~~~~~~~~Description  \\ \hline
$\bullet$ \textbf{Spaces and Labels} & \\
$d$ and $\mathcal{X}\subset \mathbb{R}^d$&  the feature dimension of data point and feature space\\
$\mathcal{Y}$&  ID label space $\{1,...,K\}$\\
$K+1$&  $K+1$ represents the OOD labels\\
$\mathcal{Y}_{\rm all}$&  $\mathcal{Y}\cup \{K+1\}$\\
$\bullet$ \textbf{Distributions} & \\
$X_{\rm I}$, $X_{\rm O}$, $Y_{\rm I}$, $Y_{\rm O}$ & ID feature, OOD feature, ID label, OOD label random variables\\
$D_{X_{\rm I}Y_{\rm I}}$, $D_{X_{\rm O}Y_{\rm O}}$&ID joint distribution and OOD joint distribution\\
$D_{XY}^{\alpha}$& $D_{XY}^{\alpha}=(1-\alpha)D_{X_{\rm I}Y_{\rm I}}+\alpha D_{X_{\rm O}Y_{\rm O}},~~\forall \alpha\in [0,1]$\\
$\pi^{\rm out}$& class-prior probability for OOD distribution\\
$D_{XY}$& $D_{XY}=(1-\pi^{\rm out})D_{X_{\rm I}Y_{\rm I}}+\pi^{\rm out} D_{X_{\rm O}Y_{\rm O}}$, called domain
\\
$D_{X_{\rm I}}, D_{X_{\rm O}},D_X$& marginal distributions for $D_{X_{\rm I}Y_{\rm I}}$, $D_{X_{\rm O}Y_{\rm O}}$ and $D_{XY}$, respectively
\\
$\bullet$ \textbf{Domain Spaces} & \\
$\mathscr{D}_{XY}$& domain space consisting of some domains\\
$\mathscr{D}_{XY}^{\rm all}$& total space\\
$\mathscr{D}_{XY}^{s}$& seperate space\\
$\mathscr{D}_{XY}^{D_{XY}}$& single-distribution space\\ 
$\mathscr{D}_{XY}^{F}$& finite-ID-distribution space\\
$\mathscr{D}_{XY}^{\mu,b}$& density-based space\\
$\bullet$ \textbf{Loss Function, Function Spaces} & \\
$\ell(\cdot,\cdot)$& loss: $\mathcal{Y}_{\rm all}\times \mathcal{Y}_{\rm all}\rightarrow \mathbb{R}_{\geq 0}$:  $\ell(y_1,y_2)=0$ if and only if $y_1=y_2$\\
$\mathcal{H}$& hypothesis space 
\\ 
$\mathcal{H}^{\rm in}$& ID hypothesis space 
\\ 
$\mathcal{H}^{\rm b}$& hypothesis space in binary
classification 
\\ 
$\mathcal{F}_l$& scoring function space consisting some $l$ dimensional vector-valued functions
\\ 
$\mathcal{R}$& ranking function space consisting some ranking functions
\\ 
$\bullet$ \textbf{Risks, Partial Risks and AUC} & \\
$R_D(h)$& risk corresponding to $D_{XY}$\\
$R_D^{\rm in}(h)$& partial risk corresponding to $D_{X_{\rm I}Y_{\rm I}}$\\
$R_D^{\rm out}(h)$& partial risk corresponding to $D_{X_{\rm O}Y_{\rm O}}$\\
$R_D^{\alpha}(h)$& $\alpha$-risk corresponding to $D_{XY}^{\alpha}$
\\
${\rm AUC}(r;D_{XY})$& AUC corresponding to $D_{XY}$
\\
${\rm AUC}(r;D_{X_{\rm I}},D_{X_{\rm O}})$& AUC corresponding to $D_{X_{\rm I}},D_{X_{\rm O}}$
\\
$\bullet$ \textbf{Fully-Connected Neural Networks} & \\
$\mathbf{q}$&  a sequence $(l_1,...,l_g)$ to represent the architecture of FCNN\\
$\sigma$&  activation function. In this paper, we use ReLU function\\
$\mathcal{F}_{\mathbf{q}}^{\sigma}$&  FCNN-based scoring function space\\
$\mathcal{H}_{\mathbf{q}}^{\sigma}$&  FCNN-based hypothesis space\\
$\mathbf{f}_{\mathbf{w},\mathbf{b}}$& FCNN-based scoring function, which is from $\mathcal{F}_{\mathbf{q}}^{\sigma}$\\
${h}_{\mathbf{w},\mathbf{b}}$& FCNN-based hypothesis function, which is from $\mathcal{H}_{\mathbf{q}}^{\sigma}$\\
$\bullet$ \textbf{Score-based Hypothesis Space} & \\
$E$ & scoring function\\
$\lambda$& threshold\\
$\mathcal{H}_{\mathbf{q},E}^{\sigma, \lambda}$&  score-based hypothesis
space---a binary classification space\\
${h}_{\mathbf{f},E}^{\lambda}$&  score-based hypothesis
function---a binary classifier\\
\hline
\end{tabular}
\end{center}
\end{table}
\\
Given $\mathbf{f}=[f^1,...,f^l]^{\top}$, for any $\mathbf{x}\in \mathcal{X}$,
\begin{equation*}
    \argmax_{k
    \in \{1,...,l\}} f^k(\mathbf{x}):=\max \{k\in \{1,...,l\}: f^k(\mathbf{x}) \geq f^i(\mathbf{x}), \forall i=1,...,l\},
\end{equation*}
where $f^k$ is the $k$-th coordinate of $\mathbf{f}$ and $f^i$ is the $i$-th coordinate of $\mathbf{f}$. 
The above definition about $\argmax$ aims to overcome some special cases. For example, there exist ${k}_1$, ${k}_2$ ($k_1< k_2$) such that $f^{k_1}(\mathbf{x})=f^{k_2}(\mathbf{x})$ and $f^{k_1}(\mathbf{x})> f^{i}(\mathbf{x})$, $f^{k_2}(\mathbf{x})> f^{i}(\mathbf{x})$, $\forall i\in \{1,...,l\}{- }\{k_1,k_2\}$. 
Then, according to the above definition, $k_2=\argmax_{k\in\{1,...,l\}} f^k(\mathbf{x})$.
\subsection{Realizability Assumptions}\label{SB.2}

\begin{assumption}[Risk-based Realizability Assumption]\label{Reaass}
A domain space $\mathscr{D}_{XY}$ and hypothesis space $\mathcal{H}$ satisfy the Risk-based Realizability Assumption, if for each domain $D_{XY}\in \mathscr{D}_{XY}$, there exists at least one hypothesis function $h^{*}\in \mathcal{H}$ such that $R_D(h^*)=0$.
\end{assumption}

\begin{assumption}[AUC-based Realizability Assumption]\label{Reaass_auc}
A domain space $\mathscr{D}_{XY}$ and ranking function space $\mathcal{R}$ satisfy the AUC-based Realizability Assumption under AUC, if for each domain $D_{XY}\in \mathscr{D}_{XY}$, there exists at least one hypothesis function $r^{*}\in \mathcal{R}$ such that ${\rm AUC}(h^*;D_{XY})=1$.
\end{assumption}

\subsection{Learnability and PAC learnability}\label{SB.3}
Here we give a proof to show that Learnability given in Definition \ref{D0} and PAC learnability are equivalent.
\\

\noindent \textbf{First}, we prove that Learnability concludes the PAC learnability.
\\

\noindent According to Definition \ref{D0},
\begin{equation*}
    \mathbb{E}_{S\sim D^n_{X_{\rm I}Y_{\rm I}}} R_{D}(\mathbf{A}(S))\leq \inf_{h\in \mathcal{H}}R_{D}(h) + \epsilon_{\rm cons}(n),
\end{equation*}
which implies that
\begin{equation*}
    \mathbb{E}_{S\sim D^n_{X_{\rm I}Y_{\rm I}}} [R_{D}(\mathbf{A}(S))-\inf_{h\in \mathcal{H}}R_{D}(h)]\leq   \epsilon_{\rm cons}(n).
\end{equation*}
Note that $R_{D}(\mathbf{A}(S))-\inf_{h\in \mathcal{H}}R_{D}(h) \geq 0$. Therefore, by Markov's inequality, we have \begin{equation*}
    \mathbb{P}(R_{D}(\mathbf{A}(S))-\inf_{h\in \mathcal{H}}R_{D}(h)<\epsilon)>1-  \mathbb{E}_{S\sim D^n_{X_{\rm I}Y_{\rm I}}} [R_{D}(\mathbf{A}(S))-\inf_{h\in \mathcal{H}}R_{D}(h)]/\epsilon \geq 1- \epsilon_{\rm cons}(n)/\epsilon.
\end{equation*}
 Because $\epsilon_{\rm cons}(n)$ is monotonically decreasing, we can find a smallest $m$ such that $\epsilon_{\rm cons}(m) \geq \epsilon\delta$ and $\epsilon_{\rm cons}(m-1) < \epsilon\delta$, for $\delta\in (0,1)$. We define that $m(\epsilon,\delta)=m$. Therefore, for any $\epsilon>0$ and $\delta\in (0,1)$, there exists a function $m(\epsilon,\delta)$ such that when $n>m(\epsilon,\delta)$, with the probability at least $1-\delta$, we have
 \begin{equation*}
     R_{D}(\mathbf{A}(S))-\inf_{h\in \mathcal{H}}R_{D}(h)<\epsilon,
 \end{equation*}
 which is the definition of  PAC learnability.
 \\

\noindent \textbf{Second}, we prove that the PAC learnability concludes Learnability.
\\

\noindent PAC-learnability: for any $\epsilon>0$ and $0<\delta<1$, there exists a function $m(\epsilon,\delta)>0$ such that when the sample size $n>m(\epsilon,\delta)$, we have that with the probability at least $1-\delta>0$,
\begin{equation*}
    R_D(\mathbf{A}(S))-\inf_{h\in \mathcal{H}}  R_D(h) \leq \epsilon.
\end{equation*}

\noindent Note that the loss $\ell$ defined in Section \ref{S3} has upper bound (because $\mathcal{Y}\cup \{K+1\}$ is a finite set). We assume the upper bound of $\ell$ is $M$. Hence, according to the definition of PAC-learnability, when the sample size $n>m(\epsilon,\delta)$, we have that 
\begin{equation*}
 \mathbb{E}_S [R_D(\mathbf{A}(S))-\inf_{h\in \mathcal{H}}  R_D(h)] \leq \epsilon(1-\delta)+2M\delta < \epsilon+2M\delta.
\end{equation*}
If we set $\delta = \epsilon$, then when the sample size $n>m(\epsilon,\epsilon)$, we have that 
\begin{equation*}
 \mathbb{E}_S [R_D(\mathbf{A}(S))-\inf_{h\in \mathcal{H}}  R_D(h)] < (2M+1)\epsilon,
\end{equation*}
this implies that
\begin{equation*}
 \lim_{n\rightarrow +\infty}\mathbb{E}_S [R_D(\mathbf{A}(S))-\inf_{h\in \mathcal{H}}  R_D(h)]=0,
\end{equation*}
which implies the Learnability in Definition \ref{D0}. We have completed this proof.
\newpage
\section{Proof of Theorem \ref{T1}}\label{SC}

\thmCone*
\begin{proof}[Proof of Theorem \ref{T1}.] 
$~$

\noindent \textbf{Proof of the First Result.}

\noindent To prove that $\mathscr{D}_{XY}'$ is a priori-unknown space, we need to show that for any $D_{XY}^{\alpha'} \in \mathscr{D}_{XY}'$, then $D_{XY}^{\alpha} \in \mathscr{D}_{XY}'$ for any $\alpha\in [0,1)$.
\\
\\
\noindent According to the definition of $\mathscr{D}_{XY}'$, for any $D_{XY}^{\alpha'} \in \mathscr{D}_{XY}'$, we can find a domain $D_{XY}\in \mathscr{D}_{XY}$, which can be written as $D_{XY}=(1-\pi^{\rm out})D_{X_{\rm I}Y_{\rm I}}+ \pi^{\rm out}D_{X_{\rm O}Y_{\rm O}}$ (here $\pi^{\rm out}\in [0,1)$) such that
\begin{equation*}
    D_{XY}^{\alpha'} = (1-\alpha')D_{X_{\rm I}Y_{\rm I}}+\alpha' D_{X_{\rm O}Y_{\rm O}}.
\end{equation*}

\noindent Note that $D_{XY}^{\alpha} = (1-\alpha)D_{X_{\rm I}Y_{\rm I}}+\alpha D_{X_{\rm O}Y_{\rm O}}$. Therefore, based on the definition of $\mathscr{D}_{XY}'$, for any $\alpha\in [0,1)$, $ D_{XY}^{\alpha}\in \mathscr{D}_{XY}'$, which implies that $\mathscr{D}_{XY}'$ is a prior-known space. Additionally, for any $D_{XY}\in \mathscr{D}_{XY}$, we can rewrite $D_{XY}$ as $D_{XY}^{\pi_{\rm out}}$, thus $D_{XY}=D_{XY}^{\pi_{\rm out}}\in \mathscr{D}_{XY}'$, which implies that $\mathscr{D}_{XY}\subset \mathscr{D}_{XY}'$.
\\

\noindent \textbf{Proof of the Second Result.}

\noindent \textbf{First,} we prove that Definition \ref{D0} concludes Definition \ref{D2}, if $\mathscr{D}_{XY}$ is a prior-unknown space:

\noindent\fbox{
\parbox{0.98\textwidth}{
~~~~$\mathscr{D}_{XY}$ is a priori-unknown space, and OOD detection  is learnable in $\mathscr{D}_{XY}$ for $\mathcal{H}$. \\
$~~~~~~~~~~~~~~~~~~~~~~~~~~~~~~~~~~~~~~~~~~~~~~~~~~~~~~~~~~~~{\Huge \Downarrow}$
\\
OOD detection is strongly learnable in $\mathscr{D}_{XY}$ for $\mathcal{H}$: there exist an algorithm $\mathbf{A}: \cup_{n=1}^{+\infty}(\mathcal{X}\times\mathcal{Y})^n\rightarrow \mathcal{H}$, and a monotonically decreasing sequence $\epsilon(n)$, such that
$\epsilon(n)\rightarrow 0$, as $n\rightarrow +\infty$
\begin{equation*}
\begin{split}
 \mathbb{E}_{S\sim D^n_{X_{\rm I}Y_{\rm I}}} \big[R^{\alpha}_D(\mathbf{A}(S))& -\inf_{h\in \mathcal{H}}R^{\alpha}_D(h)\big] \leq \epsilon(n), ~~~\forall \alpha\in [0,1],~\forall D_{XY}\in \mathscr{D}_{XY}.
 \end{split}
\end{equation*}
}
}
\\

\noindent In the priori-unknown space, for any $D_{XY}\in \mathscr{D}_{XY}$, we have that for any $\alpha\in [0,1)$,
\begin{equation*}
    D_{XY}^{\alpha}=(1-\alpha)D_{X_{\rm I}Y_{\rm I}}+\alpha D_{X_{\rm O}Y_{\rm O}}\in \mathscr{D}_{XY}.
\end{equation*}
Then, according to the definition of learnability of OOD detection, we have an algorithm $\mathbf{A}$ and a monotonically decreasing sequence
$\epsilon_{\rm cons}(n)\rightarrow 0$, as $n\rightarrow +\infty$, such that for any $\alpha\in [0,1)$,
\begin{equation*}
    \mathbb{E}_{S\sim D^n_{X_{\rm I}Y_{\rm I}}} R_{D^{\alpha}}(\mathbf{A}(S))\leq \inf_{h\in \mathcal{H}}R_{D^{\alpha}}(h) + \epsilon_{\rm cons}(n),~~(\textnormal{{by~the~property~of~priori}-{\rm unknown~space}})
\end{equation*}
where 
\begin{equation*}
     R_{D^{\alpha}}(\mathbf{A}(S))=\int_{\mathcal{X}\times \mathcal{Y}_{\rm all}} \ell(\mathbf{A}(S)(\mathbf{x}),y){\rm d}D^{\alpha}_{XY}(\mathbf{x},y),~~~R_{D^{\alpha}}(h)=\int_{\mathcal{X}\times \mathcal{Y}_{\rm all}} \ell(h(\mathbf{x}),y){\rm d}D^{\alpha}_{XY}(\mathbf{x},y).
\end{equation*}
Since $R_{D^{\alpha}}(\mathbf{A}(S))=R_{D}^{\alpha}(\mathbf{A}(S))$ and $R_{D^{\alpha}}(h)=R_{D}^{\alpha}(h)$, we have that
\begin{equation}\label{Eq1}
    \mathbb{E}_{S\sim D^n_{X_{\rm I}Y_{\rm I}}} R_{D}^{\alpha}(\mathbf{A}(S))\leq \inf_{h\in \mathcal{H}}R_{D}^{\alpha}(h) + \epsilon_{\rm cons}(n),~~~ \forall \alpha\in[0,1).
\end{equation}
{Next},
we consider the case that $\alpha=1$. Note that
\begin{equation}\label{Riskinequality1}
   \liminf_{\alpha\rightarrow 1} \inf_{h\in\mathcal{H}} R_{D}^{\alpha}(h) \geq \liminf_{\alpha \rightarrow 1} \alpha \inf_{h\in\mathcal{H}} R_{D}^{\rm out}(h) =  \inf_{h\in\mathcal{H}} R_{D}^{\rm out}(h).
\end{equation}
Then, we assume that $h_{\epsilon} \in \mathcal{H}$ satisfies that
\begin{equation*}
    R_{D}^{\rm out}(h_{\epsilon})-\inf_{h\in\mathcal{H}} R_{D}^{\rm out}(h) \leq \epsilon.
\end{equation*}
It is obvious that
\begin{equation*}
    R_{D}^{\alpha}(h_{\epsilon}) \geq  \inf_{h\in\mathcal{H}} R_{D}^{\alpha}(h).
\end{equation*}
Let $\alpha\rightarrow 1$. Then, for any $\epsilon>0$,
\begin{equation*}
    R_{D}^{\rm out}(h_{\epsilon})=\lim_{\alpha\rightarrow 1} R_{D}^{\alpha}(h_{\epsilon})=\limsup_{\alpha\rightarrow 1}R_{D}^{\alpha}(h_{\epsilon}) \geq \limsup_{\alpha\rightarrow 1} \inf_{h\in\mathcal{H}} R_{D}^{\alpha}(h),
\end{equation*}
which implies that
\begin{equation}\label{Riskinequality2}
    \inf_{h\in\mathcal{H}} R_{D}^{\rm out}(h) = \lim_{\epsilon\rightarrow 0}  R_{D}^{\rm out}(h_{\epsilon})\geq \lim_{\epsilon\rightarrow 0}\limsup_{\alpha\rightarrow 1} \inf_{h\in\mathcal{H}} R_{D}^{\alpha}(h)= \limsup_{\alpha\rightarrow 1} \inf_{h\in\mathcal{H}} R_{D}^{\alpha}(h).
\end{equation}
Combining Eq. (\ref{Riskinequality1}) with Eq. (\ref{Riskinequality2}), we have
\begin{equation}\label{Riskinequality3}
    \inf_{h\in\mathcal{H}} R_{D}^{\rm out}(h) = \limsup_{\alpha\rightarrow 1} \inf_{h\in\mathcal{H}} R_{D}^{\alpha}(h) = \liminf_{\alpha\rightarrow 1} \inf_{h\in\mathcal{H}} R_{D}^{\alpha}(h),
\end{equation}
which implies that
\begin{equation}\label{Riskinequality4}
    \inf_{h\in\mathcal{H}} R_{D}^{\rm out}(h) = \lim_{\alpha\rightarrow 1} \inf_{h\in\mathcal{H}} R_{D}^{\alpha}(h).
\end{equation}
Note that
\begin{equation*}
     \mathbb{E}_{S\sim D^n_{X_{\rm I}Y_{\rm I}}} R_{D}^{\alpha}(\mathbf{A}(S)) = (1-\alpha)\mathbb{E}_{S\sim D^n_{X_{\rm I}Y_{\rm I}}} R_{D}^{\rm in}(\mathbf{A}(S))+\alpha \mathbb{E}_{S\sim D^n_{X_{\rm I}Y_{\rm I}}} R_{D}^{\rm out}(\mathbf{A}(S)).
\end{equation*}
Hence, Lebesgue's Dominated Convergence Theorem \citep{cohn2013measure} implies that
\begin{equation}\label{Riskinequality5}
    \lim_{\alpha\rightarrow 1}\mathbb{E}_{S\sim D^n_{X_{\rm I}Y_{\rm I}}} R_{D}^{\alpha}(\mathbf{A}(S)) = \mathbb{E}_{S\sim D^n_{X_{\rm I}Y_{\rm I}}} R_{D}^{\rm out}(\mathbf{A}(S)).
\end{equation}

\noindent Using Eq. (\ref{Eq1}), we have that
\begin{equation}\label{Eq18}
   \lim_{\alpha \rightarrow 1} \mathbb{E}_{S\sim D^n_{X_{\rm I}Y_{\rm I}}} R_{D}^{\alpha}(\mathbf{A}(S))\leq  \lim_{\alpha \rightarrow 1} \inf_{h\in \mathcal{H}}R_{D}^{\alpha}(h) + \epsilon_{\rm cons}(n).
\end{equation}
Combining Eq. (\ref{Riskinequality4}), Eq. (\ref{Riskinequality5}) with Eq. (\ref{Eq18}), we obtain that
\begin{equation*}
  \mathbb{E}_{S\sim D^n_{X_{\rm I}Y_{\rm I}}} R_{D}^{\rm out}(\mathbf{A}(S))\leq  \inf_{h\in \mathcal{H}}R_{D}^{\rm out}(h) + \epsilon_{\rm cons}(n).
\end{equation*}
Since $R_D^{\rm out}(\mathbf{A}(S))=R_{D}^{1}(\mathbf{A}(S))$ and $R_D^{\rm out}(h)=R_{D}^{1}(h)$, we obtain that \begin{equation}\label{Eq2}
  \mathbb{E}_{S\sim D^n_{X_{\rm I}Y_{\rm I}}} R_{D}^{1}(\mathbf{A}(S))\leq  \inf_{h\in \mathcal{H}}R_{D}^{1}(h) + \epsilon_{\rm cons}(n).
\end{equation}
\\
Combining Eq. (\ref{Eq1}) and Eq. (\ref{Eq2}), we have proven that: if the  domain space $\mathscr{D}_{XY}$ is a priori-unknown space, then OOD detection  is learnable in $\mathscr{D}_{XY}$ for $\mathcal{H}$.\\
$~~~~~~~~~~~~~~~~~~~~~~~~~~~~~~~~~~~~~~~~~~~~~~~~~~~~~~~~{\Huge \Downarrow}$
\\
OOD detection is strongly learnable in $\mathscr{D}_{XY}$ for $\mathcal{H}$: there exist an algorithm $\mathbf{A}: \cup_{n=1}^{+\infty}(\mathcal{X}\times\mathcal{Y})^n\rightarrow \mathcal{H}$, and a monotonically decreasing sequence $\epsilon(n)$, such that
$\epsilon(n)\rightarrow 0$, as $n\rightarrow +\infty$,
\begin{equation*}
\begin{split}
 \mathbb{E}_{S\sim D^n_{X_{\rm I}Y_{\rm I}}} R^{\alpha}_D(\mathbf{A}(S))&\leq \inf_{h\in \mathcal{H}}R^{\alpha}_D(h) + \epsilon(n),~~~\forall \alpha\in [0,1],~\forall D_{XY}\in \mathscr{D}_{XY}.
 \end{split}
\end{equation*}

\noindent \textbf{Second}, we prove that Definition \ref{D2} concludes Definition \ref{D0}:
\\
\fbox{
\parbox{1\textwidth}{
OOD detection is strongly learnable in $\mathscr{D}_{XY}$ for $\mathcal{H}$: there exist an algorithm $\mathbf{A}: \cup_{n=1}^{+\infty}(\mathcal{X}\times\mathcal{Y})^n\rightarrow \mathcal{H}$, and a monotonically decreasing sequence $\epsilon(n)$, such that
$\epsilon(n)\rightarrow 0$, as $n\rightarrow +\infty$,
\begin{equation*}
\begin{split}
 \mathbb{E}_{S\sim D^n_{X_{\rm I}Y_{\rm I}}} \big[R^{\alpha}_D(\mathbf{A}(S))&- \inf_{h\in \mathcal{H}}R^{\alpha}_D(h)\big] \leq \epsilon(n),~~~\forall \alpha\in [0,1],~\forall D_{XY}\in \mathscr{D}_{XY}.
 \end{split}
\end{equation*}
$~~~~~~~~~~~~~~~~~~~~~~~~~~~~~~~~~~~~~~~~~~~~~~~~~~~~~~~~~~~~~~~{\Huge \Downarrow}$\\
${~~~~~~~~~~~~~~~~~~~~~~~~~~~~~~~~~}$OOD detection is learnable in $\mathscr{D}_{XY}$ for $\mathcal{H}$.
}
}

\noindent If we set $\alpha=\pi^{\rm out}$, then $
 \mathbb{E}_{S\sim D^n_{X_{\rm I}Y_{\rm I}}} R^{\alpha}_D(\mathbf{A}(S))\leq \inf_{h\in \mathcal{H}}R^{\alpha}_D(h) + \epsilon(n)
$ implies that
\begin{equation*}
\begin{split}
    \mathbb{E}_{S\sim D^n_{X_{\rm I}Y_{\rm I}}} R_D(\mathbf{A}(S)) \leq \inf_{h\in \mathcal{H}} R_D(h) + \epsilon(n),
    \end{split}
\end{equation*}
which means that OOD detection  is learnable in $\mathscr{D}_{XY}$ for $\mathcal{H}$. We have completed this proof.
\\

\noindent \textbf{Proof of the Third Result.} 

\noindent The third result is a simple conclusion of the second result. Hence, we omit it.
\\

\noindent \textbf{Proof of the Fourth Result.}

\noindent The fourth result is a simple conclusion of the property of AUC metric. Hence, we omit it.

\end{proof}
\section{Proof of Theorem \ref{T3}}\label{SD}

Before introducing the proof of Theorem \ref{T3}, we extend Condition \ref{C1} to a general version (Condition \ref{C2}). Then, Lemma  \ref{C1andC2} proves that Conditions \ref{C1} and \ref{C2} are the necessary conditions for the learnability of OOD detection. First, we provide the details of Condition \ref{C2}.

\noindent  Let  $\Delta_l^{\rm o} =\{(\lambda_1,...,\lambda_l): \sum_{j=1}^l \lambda_j<1~ {\rm and}~ \lambda_j \geq 0, \forall j=1,...,l\}$, where $l$ is a positive integer. Next, we introduce an important definition as follows:
\begin{Definition}[OOD Convex Decomposition and Convex Domain]\label{OODCD}
Given any domain $D_{XY}\in \mathscr{D}_{XY}$, we say joint distributions $Q_1,...,Q_l$, which are defined over $\mathcal{X}\times \{K+1\}$, are the OOD convex decomposition for $D_{XY}$, if 
\begin{equation*}
D_{XY}=(1-\sum_{j=1}^l \lambda_j)D_{X_{\rm I}Y_{\rm I}}+\sum_{j=1}^l \lambda_j Q_j,
\end{equation*}
for some  $(\lambda_1,...,\lambda_l)\in \Delta_l^{\rm o}$. We also say domain $D_{XY}\in \mathscr{D}_{XY}$ is an OOD convex domain corresponding to OOD convex decomposition $Q_1,...,Q_l$, if for any $(\alpha_1,...,\alpha_l)\in \Delta_l^{\rm o}$, 
\begin{equation*}
   (1-\sum_{j=1}^l \alpha_j)D_{X_{\rm I}Y_{\rm I}}+\sum_{j=1}^l \alpha_j Q_j\in \mathscr{D}_{XY}.
\end{equation*}
\end{Definition}
We extend the  linear condition (Condition \ref{C1}) to a multi-linear scenario.
\begin{Condition}[Multi-linear Condition]\label{C2}
    For each OOD convex domain $D_{XY}\in \mathscr{D}_{XY}$ corresponding to OOD convex decomposition $Q_1,...,Q_l$, the following function 
    \begin{equation*}
    f_{D,Q}(\alpha_1,...,\alpha_{l}):= \inf_{h\in \mathcal{H}}  \Big ((1-\sum_{j=1}^{l}\alpha_j) R_D^{\rm in}(h) + \sum_{j=1}^l \alpha_j R_{Q_j}(h) \Big ),~~~\forall (\alpha_1,...,\alpha_{l})\in\Delta_l^{\rm o}
    \end{equation*}
  satisfies that
\begin{equation*}
    f_{D,Q}(\alpha_1,...,\alpha_{l})= (1-\sum_{j=1}^{l}\alpha_j) f_{D,Q}(\mathbf{0})+\sum_{j=1}^{l}\alpha_j f_{D,Q}({\bm \alpha}_j),
    \end{equation*}
    where $\mathbf{0}$ is the $1\times l$ vector, whose elements are $0$, and ${\bm \alpha}_j$ is the $1\times l$ vector, whose $j$-th element is $1$ and other elements are $0$.
\end{Condition}
This is a more general condition compared to Condition \ref{C1}. When $l=1$ and the domain space $\mathscr{D}_{XY}$ is a priori-unknown space, Condition \ref{C2} degenerates into  Condition \ref{C1}.
Lemma \ref{C1andC2} shows that Condition \ref{C2} is necessary for the learnability of OOD detection.
\begin{lemma}\label{C1andC2}
Given a priori-unknown space $\mathscr{D}_{XY}$ and a hypothesis space $\mathcal{H}$, if OOD detection  is learnable in $\mathscr{D}_{XY}$ for $\mathcal{H}$, then Conditions \ref{C1} and \ref{C2} hold.
\end{lemma}
\begin{proof}[Proof of Lemma \ref{C1andC2}]
$~$

\noindent Since Condition \ref{C1} is a special case of Condition \ref{C2}, we only need to prove that Condition \ref{C2} holds. 
\\
\\
\noindent  For any OOD convex domain $D_{XY}\in \mathscr{D}_{XY}$ corresponding to OOD convex decomposition $Q_1,...,Q_l$,  and any $(\alpha_1,...,\alpha_l) \in \Delta_l^{\rm o}$, we set
 \begin{equation*}
     Q^{\bm \alpha}= \frac{1}{\sum_{i=1}^{l}\alpha_i} \sum_{j=1}^l \alpha_j Q_j.
 \end{equation*}
Then, we define
\begin{equation*}
    D_{XY}^{\bm \alpha}=(1-\sum_{i=1}^{l}\alpha_i)D_{X_{\rm I}Y_{\rm I}}+(\sum_{i=1}^{l}\alpha_i)  Q^{\bm \alpha},~\textnormal{which~belongs~to~}\mathscr{D}_{XY}.
\end{equation*}
Let
\begin{equation*}
     R^{\bm \alpha}_D(h)=\int_{\mathcal{X}\times \mathcal{Y}_{\rm all}} \ell(h(\mathbf{x}),y){\rm d}D_{XY}^{\bm \alpha}(\mathbf{x},y).
\end{equation*}

\noindent Since OOD detection is learnable in $\mathscr{D}_{XY}$ for $\mathcal{H}$,  there exist an algorithm $\mathbf{A}: \cup_{n=1}^{+\infty}(\mathcal{X}\times\mathcal{Y})^n\rightarrow \mathcal{H}$, and a monotonically decreasing sequence $\epsilon(n)$, such that
$\epsilon(n)\rightarrow 0$, as $n\rightarrow +\infty$, and
\begin{equation*}
 0\leq  \mathbb{E}_{S\sim D^n_{X_{\rm I}Y_{\rm I}}} R^{\bm \alpha}_D(\mathbf{A}(S)) - \inf_{h\in \mathcal{H}}R^{\bm \alpha}_D(h)\leq \epsilon(n).
\end{equation*}

\noindent Note that 
\begin{equation*}
\begin{split}
    &\mathbb{E}_{S\sim D^n_{X_{\rm I}Y_{\rm I}}} R^{\bm \alpha}_D(\mathbf{A}(S))=(1-\sum_{j=1}^{l}\alpha_j) \mathbb{E}_{S\sim D^n_{X_{\rm I}Y_{\rm I}}} R^{\rm in}_D(\mathbf{A}(S)) +\sum_{j=1}^{l}\alpha_j \mathbb{E}_{S\sim D^n_{X_{\rm I}Y_{\rm I}}}R_{Q_j}(\mathbf{A}(S)),
    \end{split}
\end{equation*}
and
\begin{equation*}
    \inf_{h\in \mathcal{H}}R^{\bm \alpha}_D(h) = f_{D,Q}(\alpha_1,...,\alpha_{l}),
\end{equation*}
where
\begin{equation*}
    R_{Q_j}(\mathbf{A}(S)) = \int_{\mathcal{X}\times \{K+1\}} \ell(\mathbf{A}(S)(\mathbf{x}),y){\rm d}Q_j(\mathbf{x},y).
\end{equation*}
Therefore, we have that for any $ (\alpha_1,...,\alpha_l)\in \Delta_l^{\rm o}$,
\begin{equation}\label{Eq::converge}
\begin{split}
  &\big |(1-\sum_{j=1}^{l}\alpha_j) \mathbb{E}_{S\sim D^n_{X_{\rm I}Y_{\rm I}}} R^{\rm in}_D(\mathbf{A}(S)) +\sum_{j=1}^{l}\alpha_j \mathbb{E}_{S\sim D^n_{X_{\rm I}Y_{\rm I}}}R_{Q_j}(\mathbf{A}(S)) - f_{D,Q}(\alpha_1,...,\alpha_{l})\big |\leq \epsilon(n).
  \end{split}
\end{equation}
Let 
\begin{equation*}
g_n(\alpha_1,...,\alpha_l)=(1-\sum_{j=1}^{l}\alpha_j) \mathbb{E}_{S\sim D^n_{X_{\rm I}Y_{\rm I}}} R^{\rm in}_D(\mathbf{A}(S)) +\sum_{j=1}^{l}\alpha_j \mathbb{E}_{S\sim D^n_{X_{\rm I}Y_{\rm I}}}R_{Q_j}(\mathbf{A}(S)).
\end{equation*}
\noindent Note that Eq. \eqref{Eq::converge} implies that
\begin{equation}\label{aim5}
\begin{split}
    \lim_{n\rightarrow +\infty } g_n(\alpha_1,...,\alpha_l) &= f_{D,Q}(\alpha_1,...,\alpha_{l}),~~\forall (\alpha_1,...,\alpha_l)\in \Delta_l^{\rm o},
   \\ 
   \lim_{n\rightarrow +\infty } g_n(\mathbf{0}) &= f_{D,Q}(\mathbf{0}).
    \end{split}
\end{equation}

\noindent \textbf{Step 1.} Since ${\bm \alpha}_j\notin \Delta_l^{\rm o}$, we need to prove that
\begin{equation}
    \lim_{n\rightarrow +\infty} \mathbb{E}_{S\sim D^n_{X_{\rm I}Y_{\rm I}}} R_{Q_j}(\mathbf{A}(S))  = f({\bm \alpha}_j), \textit{i.e.},  \lim_{n\rightarrow +\infty} g_n({\bm \alpha}_j)  = f({\bm \alpha}_j),
\end{equation}
where ${\bm \alpha}_j$ is the $1\times l$ vector, whose $j$-th element is $1$ and other elements are $0$.

\noindent Let $\tilde{D}_{XY}=0.5*D_{X_{\rm I}Y_{\rm I}}+0.5* Q_j$. The second result of Theorem \ref{T1} implies that
\begin{equation*}
\begin{split}
 \mathbb{E}_{S\sim D^n_{X_{\rm I}Y_{\rm I}}} R^{\rm out}_{\tilde{D}}(\mathbf{A}(S))\leq \inf_{h\in \mathcal{H}}R^{\rm out}_{\tilde{D}}(h) + \epsilon(n).
 \end{split}
\end{equation*}
Since $R^{\rm out}_{\tilde{D}}(\mathbf{A}(S))=R_{Q_j}(\mathbf{A}(S))$ and $R^{\rm out}_{\tilde{D}}(h)=R_{Q_j}(h)$,
\begin{equation*}
    \mathbb{E}_{S\sim D^n_{X_{\rm I}Y_{\rm I}}} R_{Q_j}(\mathbf{A}(S))\leq \inf_{h\in \mathcal{H}}R_{Q_j}(h) + \epsilon(n).
\end{equation*}
Note that $ \inf_{h\in \mathcal{H}}R_{Q_j}(h)  \leq \mathbb{E}_{S\sim D^n_{X_{\rm I}Y_{\rm I}}} R_{Q_j}(\mathbf{A}(S))$. We have
\begin{equation}\label{aim1}
   0\leq  \mathbb{E}_{S\sim D^n_{X_{\rm I}Y_{\rm I}}} R_{Q_j}(\mathbf{A}(S))- \inf_{h\in \mathcal{H}}R_{Q_j}(h)\leq \epsilon(n).
\end{equation}
Eq. \eqref{aim1} implies that
\begin{equation}\label{aim2}
   \lim_{n\rightarrow +\infty} \mathbb{E}_{S\sim D^n_{X_{\rm I}Y_{\rm I}}} R_{Q_j}(\mathbf{A}(S))= \inf_{h\in \mathcal{H}}R_{Q_j}(h).
\end{equation}

\noindent We note that $\inf_{h\in \mathcal{H}}R_{Q_j}(h)=f_{D,Q}({\bm \alpha}_j)$. Therefore,
\begin{equation}\label{aim3}
   \lim_{n\rightarrow +\infty} \mathbb{E}_{S\sim D^n_{X_{\rm I}Y_{\rm I}}} R_{Q_j}(\mathbf{A}(S))= f_{D,Q}({\bm \alpha}_j),~ \textit{i.e.}, \lim_{n\rightarrow +\infty} g_n({\bm \alpha}_j)  = f({\bm \alpha}_j).
\end{equation}

\noindent \textbf{Step 2.} It is easy to check that for any $(\alpha_1,...,\alpha_l)\in \Delta_l^{\rm o}$,
\begin{equation}\label{finalequ}
\begin{split}
    \lim_{n\rightarrow +\infty } g_n(\alpha_1,...,\alpha_l) &=   \lim_{n\rightarrow +\infty } \big ((1-\sum_{j=1}^{l}\alpha_j)g_n(\mathbf{0})+ \sum_{j=1}^{l}\alpha_j g_n({\bm \alpha}_j) \big)\\&=(1-\sum_{j=1}^{l}\alpha_j) \lim_{n\rightarrow +\infty } g_n(\mathbf{0})+ \sum_{j=1}^{l}\alpha_j \lim_{n\rightarrow +\infty } g_n({\bm \alpha}_j).
    \end{split}
\end{equation}
According to Eq. \eqref{aim5} and Eq. \eqref{aim3}, we have
\begin{equation}\label{finalequ1}
\begin{split}
    & \lim_{n\rightarrow +\infty } g_n(\alpha_1,...,\alpha_l)=f_{D,Q}(\alpha_1,...,\alpha_{l}),~~\forall (\alpha_1,...,\alpha_l)\in \Delta_l^{\rm o},
    \\ &
    \lim_{n\rightarrow +\infty } g_n(\mathbf{0})=f_{D,Q}(\mathbf{0}),
       \\ &  \lim_{n\rightarrow +\infty} g_n({\bm \alpha}_j)  = f({\bm \alpha}_j),
     \end{split}
\end{equation}
 Combining Eq. (\ref{finalequ1}) with Eq. (\ref{finalequ}), we complete the proof.
\end{proof}

\begin{lemma}\label{Lemma1}
\begin{equation*}
 \inf_{h\in \mathcal{H}}R_D^{\alpha}(h)=(1-\alpha)\inf_{h\in \mathcal{H}}R_D^{\rm in}(h)+\alpha\inf_{h\in \mathcal{H}}R_D^{\rm out}(h),~\forall \alpha\in [0,1),
 \end{equation*}
 \textbf{if and only if} for any $\epsilon>0$,
 \begin{equation*}
     \{ h' \in \mathcal{H}: R_D^{\rm in}(h') \leq \inf_{h\in \mathcal{H}} R_D^{\rm in}(h)+2\epsilon\} \cap  \{ h' \in \mathcal{H}: R_D^{\rm out}(h') \leq \inf_{h\in \mathcal{H}} R_D^{\rm out}(h)+2\epsilon\}\neq \emptyset.
 \end{equation*}
 \end{lemma}
 \begin{proof}[Proof of Lemma \ref{Lemma1}]
For the sake of convenience, we set $f_D(\alpha)=\inf_{h\in \mathcal{H}}R_D^{\alpha}(h)$, for any $\alpha \in [0,1]$.

\noindent  \textbf{First}, we prove that $f_D(\alpha)= (1-\alpha)f_D(0)+\alpha f_D(1)$, ~$\forall \alpha\in [0,1)$ implies \begin{equation*}
     \{ h' \in \mathcal{H}: R_D^{\rm in}(h') \leq \inf_{h\in \mathcal{H}} R_D^{\rm in}(h)+2\epsilon\} \cap  \{ h' \in \mathcal{H}: R_D^{\rm out}(h') \leq \inf_{h\in \mathcal{H}} R_D^{\rm out}(h)+2\epsilon \}\neq \emptyset.
 \end{equation*}
 For any $\epsilon>0$ and $0\leq \alpha<1$, we can find $h_{\epsilon}^{\alpha}\in \mathcal{H}$ satisfying that
 \begin{equation*}
     R_D^{\alpha}(h_{\epsilon}^{\alpha}) \leq \inf_{h\in \mathcal{H}} R_D^{\alpha}(h)+\epsilon.
 \end{equation*}
 Note that
 \begin{equation*}
     \inf_{h\in \mathcal{H}} R_D^{\alpha}(h)= \inf_{h\in \mathcal{H}} \Big ((1-\alpha)R_D^{\rm in}(h)+\alpha R_D^{\rm out}(h) \Big )\geq (1-\alpha)\inf_{h\in \mathcal{H}} R_D^{\rm in}(h)+\alpha \inf_{h\in \mathcal{H}} R_D^{\rm out}(h).
 \end{equation*}
 Therefore,
 \begin{equation}\label{L1.eq1}
(1-\alpha)\inf_{h\in \mathcal{H}} R_D^{\rm in}(h)+\alpha \inf_{h\in \mathcal{H}} R_D^{\rm out}(h)   \leq  \inf_{h\in \mathcal{H}} R_D^{\alpha}(h) \leq R_D^{\alpha}(h_{\epsilon}^{\alpha}) \leq \inf_{h\in \mathcal{H}} R_D^{\alpha}(h)+\epsilon.
 \end{equation}
 Note that $f_D(\alpha)= (1-\alpha)f_D(0)+\alpha f_D(1), \forall \alpha\in [0,1)$, \textit{i.e.},
 \begin{equation}\label{L1.eq2}
 \inf_{h\in \mathcal{H}} R_D^{\alpha}(h)= (1-\alpha)\inf_{h\in \mathcal{H}} R_D^{\rm in}(h)+\alpha \inf_{h\in \mathcal{H}} R_D^{\rm out}(h), \forall \alpha\in [0,1).
 \end{equation}
 
\noindent  Using Eqs. (\ref{L1.eq1}) and (\ref{L1.eq2}), we have that for any $0\leq \alpha<1$,
\begin{equation}\label{AboveEq}
\epsilon \geq \big |R^{\alpha}_D(h_{\epsilon}^{\alpha})-\inf_{h\in \mathcal{H}} R_D^{\alpha}(h)\big |=  \big |(1-\alpha)\big (R_D^{\rm in}(h_{\epsilon}^{\alpha})-\inf_{h\in \mathcal{H}} R_D^{\rm in}(h)\big)+\alpha\big (R_D^{\rm out}(h_{\epsilon}^{\alpha})-\inf_{h\in \mathcal{H}} R_D^{\rm out}(h)\big) \big |.
\end{equation}
Since $R_D^{\rm out}(h_{\epsilon}^{\alpha})-\inf_{h\in \mathcal{H}} R_D^{\rm out}(h)\geq 0$ and $R_D^{\rm in}(h_{\epsilon}^{\alpha})-\inf_{h\in \mathcal{H}} R_D^{\rm in}(h)\geq 0$, Eq. (\ref{AboveEq}) implies that: for any $0< \alpha<1$,
\begin{equation*}
\begin{split}
  &  R_D^{\rm in}(h_{\epsilon}^{\alpha}) \leq \inf_{h\in \mathcal{H}} R_D^{\rm in}(h)+\epsilon /(1-\alpha),
  \\ &
  R_D^{\rm out}(h_{\epsilon}^{\alpha}) \leq \inf_{h\in \mathcal{H}} R_D^{\rm out}(h)+\epsilon /\alpha.
\end{split}
\end{equation*}
Therefore,
\begin{equation*}
  h_{\epsilon}^\alpha \in \{ h' \in \mathcal{H}: R_D^{\rm in}(h') \leq \inf_{h\in \mathcal{H}} R_D^{\rm in}(h)+\epsilon/(1-\alpha)\} \cap  \{ h' \in \mathcal{H}: R_D^{\rm out}(h') \leq \inf_{h\in \mathcal{H}} R_D^{\rm out}(h)+\epsilon/\alpha \}.
 \end{equation*}
 If we set $\alpha=0.5$, we obtain that for any $\epsilon>0$,
\begin{equation*}
 \{ h' \in \mathcal{H}: R_D^{\rm in}(h') \leq \inf_{h\in \mathcal{H}} R_D^{\rm in}(h)+2\epsilon \}\cap  \{ h' \in \mathcal{H}: R_D^{\rm out}(h') \leq \inf_{h\in \mathcal{H}} R_D^{\rm out}(h)+2\epsilon \}\neq \emptyset.
 \end{equation*}
 \\
\textbf{Second}, we prove that for any $\epsilon>0$, if 
\begin{equation*}
     \{ h' \in \mathcal{H}: R_D^{\rm in}(h') \leq \inf_{h\in \mathcal{H}} R_D^{\rm in}(h)+2\epsilon\} \cap  \{ h' \in \mathcal{H}: R_D^{\rm out}(h') \leq \inf_{h\in \mathcal{H}} R_D^{\rm out}(h)+2\epsilon\}\neq \emptyset,
 \end{equation*}
 then $f_D(\alpha)= (1-\alpha)f_D(0)+\alpha f_D(1)$, for any $\alpha\in [0,1)$.
 \\
 \\
Let $h_{\epsilon}\in  \{ h' \in \mathcal{H}: R_D^{\rm in}(h') \leq \inf_{h\in \mathcal{H}} R_D^{\rm in}(h)+2\epsilon\} \cap  \{ h' \in \mathcal{H}: R_D^{\rm out}(h') \leq \inf_{h\in \mathcal{H}} R_D^{\rm out}(h)+2\epsilon\}$.

Then,
\begin{equation*}
 \inf_{h\in \mathcal{H}} R_D^{\alpha}(h) \leq   R_D^{\alpha}(h_{\epsilon})\leq (1-\alpha) \inf_{h\in \mathcal{H}} R_D^{\rm in}(h)+\alpha \inf_{h\in \mathcal{H}} R_D^{\rm out}(h)+2\epsilon \leq \inf_{h\in \mathcal{H}} R_D^{\alpha}(h)+2\epsilon,
\end{equation*}
which implies that $|f_D(\alpha)-(1-\alpha)f_D(0)-\alpha f_D(1)|\leq 2\epsilon$.

As $\epsilon \rightarrow 0$, $|f_D(\alpha)-(1-\alpha)f_D(0)-\alpha f_D(1)|\leq 0$. We have completed the proof.
 \end{proof}
\thmCondoneIff*
\begin{proof}[Proof of Theorem \ref{T3}]
Based on Lemma \ref{C1andC2}, we obtain that
 Condition \ref{C1} is the necessary condition for the learnability of OOD detection in the single-distribution space $\mathscr{D}_{XY}^{D_{XY}}$. Next, it suffices to prove that Condition \ref{C1} is the sufficient  condition for the learnability of OOD detection in the single-distribution space $\mathscr{D}_{XY}^{D_{XY}}$.
 We use Lemma \ref{Lemma1} to prove the sufficient condition. 
\\
\\
\noindent  Let $\mathscr{F}$ be the infinite sequence set that consists of all infinite sequences, whose coordinates are hypothesis functions, \textit{i.e.},
 \begin{equation*}
 \mathscr{F}=\{{\bm h}=(h_1,...,h_n,...): \forall h_n\in \mathcal{H}, n=1,....,+\infty\}.
 \end{equation*}
For each ${\bm h}\in \mathscr{F}$, there is a corresponding algorithm $\mathbf{A}_{\bm h}$\footnote{In this paper, we regard an algorithm as a mapping from $\cup_{n=1}^{+\infty}(\mathcal{X}\times\mathcal{Y})^n$ to $\mathcal{H}$ or $\mathcal{R}$. So we can design an algorithm like this.}: $\mathbf{A}_{\bm h}(S)=h_n,~{\rm if}~|S|=n$. $\mathscr{F}$ generates an algorithm class $\mathscr{A}=\{\mathbf{A}_{\bm h}: \forall {\bm h}\in \mathscr{F}\}$. We select a consistent algorithm from the algorithm class $\mathscr{A}$. 
\\
\\
\noindent We construct a special infinite sequence $\tilde{{\bm h}}=(\tilde{h}_1,...,\tilde{h}_n,...)\in \mathscr{F}$. For each positive integer $n$, we select $\tilde{h}_n$ from $  \{ h' \in \mathcal{H}: R_D^{\rm in}(h') \leq \inf_{h\in \mathcal{H}} R_D^{\rm in}(h)+2/n\} \cap  \{ h' \in \mathcal{H}: R_D^{\rm out}(h') \leq \inf_{h\in \mathcal{H}} R_D^{\rm out}(h)+2/n\}$ (the existence of $\tilde{h}_n$ is based on Lemma \ref{Lemma1}). It is easy to check that
 \begin{equation*}
 \begin{split}
  & \mathbb{E}_{S\sim D^n_{X_{\rm I}Y_{\rm I}}}  R_D^{\rm in}(\mathbf{A}_{\tilde{{\bm h}}}(S))\leq  \inf_{h\in \mathcal{H}} R_D^{\rm in}(h)+2/n.
   \\
   & \mathbb{E}_{S\sim D^n_{X_{\rm I}Y_{\rm I}}}  R_D^{\rm out}(\mathbf{A}_{\tilde{{\bm h}}}(S))\leq  \inf_{h\in \mathcal{H}} R_D^{\rm out}(h)+2/n.
    \end{split}
 \end{equation*}
 Since $(1-\alpha)\inf_{h\in \mathcal{H}} R_D^{\rm in}(h)+\alpha \inf_{h\in \mathcal{H}} R_D^{\rm out}(h)\leq \inf_{h\in \mathcal{H}} R_D^{\rm \alpha}(h)$, we obtain that for any $\alpha\in [0,1]$, 
 \begin{equation*}
 \begin{split}
  & \mathbb{E}_{S\sim D^n_{X_{\rm I}Y_{\rm I}}}  R_D^{\alpha}(\mathbf{A}_{\tilde{{\bm h}}}(S))\leq  \inf_{h\in \mathcal{H}} R_D^{\alpha}(h)+2/n.
    \end{split}
 \end{equation*}
We have completed this proof.
\end{proof}

\section{Proof of Theorem \ref{T3_auc}}\label{SD_auc}

\thmCondoneIffauc*
\begin{proof}[Proof of Theorem \ref{T3_auc}] According to Definition \ref{D0_auc}, we assume that $\mathbf{A}$ is the AUC learnable algorithm: there exists learning rate $\epsilon(n)$ such that for any $D_{XY}\in \mathscr{D}_{XY}$,
\begin{equation}\label{issue-definition1_auc-T3}
    \mathbb{E}_{S\sim D^n_{X_{\rm I}Y_{\rm I}}}\big[\sup_{r\in \mathcal{R}}{\rm AUC}(r;D_{XY})- {\rm AUC}(\mathbf{A}(S);D_{XY})\big] \leq \epsilon(n).
\end{equation}
For any $D_{XY}=\beta D_{X_{\rm I}Y_{\rm I}}+(1-\beta)D_{X_{\rm O}Y_{\rm O}},$ $D_{XY}'= \beta' D_{X_{\rm I}Y_{\rm I}}+(1-\beta')D_{X_{\rm O}Y_{\rm O}}' \in \mathscr{D}_{XY}$, we set ${D}_{X_{\rm O}}^{\alpha} = \alpha D_{X_{\rm O}} + (1-\alpha) D_{X_{\rm O}}'$. Then it is clear that for any $\tilde{\beta}\in (0, 1]$,
\begin{equation*}
\begin{split}
   & {\rm AUC}(\mathbf{A}(S); \tilde{\beta} D_{X_{\rm I}Y_{\rm I}} +(1-\tilde{\beta} )D_{X_{\rm O}Y_{\rm O}}^{\alpha})\\  =&  {\rm AUC}(\mathbf{A}(S); D_{X_{\rm I}}, D_{X_{\rm O}}^{\alpha})\\  =&\alpha {\rm AUC}(\mathbf{A}(S); D_{X_{\rm I}}, D_{X_{\rm O}})+(1-\alpha){\rm AUC}(\mathbf{A}(S); D_{X_{\rm I}}, D_{X_{\rm O}}')\\  =& \alpha {\rm AUC}(\mathbf{A}(S); D_{XY})+(1-\alpha){\rm AUC}(\mathbf{A}(S); D_{XY}').
    \end{split}
\end{equation*}
Therefore,
\begin{equation*}
\begin{split}
      & \mathbb{E}_{S\sim D^n_{X_{\rm I}Y_{\rm I}}}\big[\sup_{r\in \mathcal{R}}{\rm AUC}(r;\tilde{\beta} D_{X_{\rm I}Y_{\rm I}} +(1-\tilde{\beta} )D_{X_{\rm O}Y_{\rm O}}^{\alpha})- {\rm AUC}(\mathbf{A}(S);\tilde{\beta} D_{X_{\rm I}Y_{\rm I}} +(1-\tilde{\beta} )D_{X_{\rm O}Y_{\rm O}}^{\alpha})\big] \\
      \leq & \mathbb{E}_{S\sim D^n_{X_{\rm I}Y_{\rm I}}}\big[\alpha \sup_{r\in \mathcal{R}}{\rm AUC}(r;D_{XY}) +(1-\alpha) \sup_{r\in \mathcal{R}}{\rm AUC}(r;D_{XY}')
      \\
     &~~~~~~~~~~~~ -\alpha {\rm AUC}(\mathbf{A}(S); D_{XY})-(1-\alpha){\rm AUC}(\mathbf{A}(S); D_{XY}') ]\leq \epsilon(n),
      \end{split}
\end{equation*}
which implies that
\begin{equation*}
    \sup_{r\in \mathcal{R}}{\rm AUC}(r;D_{X_{\rm I}},D_{X_{\rm O}}^{\alpha}) = \lim_{n \rightarrow +\infty} \mathbb{E}_{S\sim D^n_{X_{\rm I}Y_{\rm I}}} {\rm AUC}(\mathbf{A}(S);D_{X_{\rm I}},D_{X_{\rm O}}^{\alpha}).
\end{equation*}
Note that
\begin{equation*}
\begin{split}
   & \lim_{n \rightarrow +\infty} \mathbb{E}_{S\sim D^n_{X_{\rm I}Y_{\rm I}}} {\rm AUC}(\mathbf{A}(S);D_{X_{\rm I}},D_{X_{\rm O}}^{\alpha})
   \\ = & \alpha \lim_{n \rightarrow +\infty} \mathbb{E}_{S\sim D^n_{X_{\rm I}Y_{\rm I}}} {\rm AUC}(\mathbf{A}(S);D_{X_{\rm I}},D_{X_{\rm O}})+(1-\alpha)  \lim_{n \rightarrow +\infty} \mathbb{E}_{S\sim D^n_{X_{\rm I}Y_{\rm I}}} {\rm AUC}(\mathbf{A}(S);D_{X_{\rm I}},D_{X_{\rm O}}')
   \\
   = &\alpha \sup_{r\in \mathcal{R}}{\rm AUC}(r;D_{X_{\rm I}},D_{X_{\rm O}})+(1-\alpha)  \sup_{r\in \mathcal{R}} {\rm AUC}(r;D_{X_{\rm I}},D_{X_{\rm O}}').
    \end{split}
\end{equation*}
Therefore,
\begin{equation*}
     \sup_{r\in \mathcal{R}}{\rm AUC}(r;D_{X_{\rm I}},D_{X_{\rm O}}^{\alpha}) = \alpha \sup_{r\in \mathcal{R}}{\rm AUC}(r;D_{X_{\rm I}},D_{X_{\rm O}})+(1-\alpha)  \sup_{r\in \mathcal{R}} {\rm AUC}(r;D_{X_{\rm I}},D_{X_{\rm O}}').
\end{equation*}
\end{proof}

\section{Proofs of Theorem \ref{T5} and Theorem \ref{T4} }\label{SF}
\subsection{Proof of Theorem \ref{T5}}
\thmImpOne*

\begin{proof}[Proof of Theorem \ref{T5}]
We \textbf{first} explain how we get $f_{\rm I}$ and $f_{\rm O}$ in Definition {\ref{def:overlap}}. Since $D_X$ is absolutely continuous respect to $\mu$ ($D_X\ll \mu$), then $D_{X_{\rm I}}\ll \mu$ and $D_{X_{\rm O}}\ll \mu$. By Radon-Nikodym Theorem \citep{cohn2013measure}, we know there exist two non-negative functions defined over $\mathcal{X}$: $f_{\rm I}$
and $f_{\rm O}$ such that for any $\mu$-measurable set $A\subset \mathcal{X}$,
\begin{equation*}
    D_{X_{\rm I}}(A)=\int_{A} f_{\rm I}(\mathbf{x}){\rm d} \mu(\mathbf{x}),~~D_{X_{\rm O}}(A)=\int_{A} f_{\rm O}(\mathbf{x}){\rm d} \mu(\mathbf{x}).
\end{equation*}

\noindent \textbf{Second}, we prove that for any $\alpha\in (0,1)$, $\inf_{h\in \mathcal{H}} R_{D}^{\alpha}(h)>0$.

\noindent We define $A_{m}=\{\mathbf{x}\in \mathcal{X}: f_{\rm I}(\mathbf{x})\geq \frac{1}{m}~ {\rm and} ~f_{\rm O}(\mathbf{x})\geq \frac{1}{m}\}$. It is clear that 
\begin{equation*}
    \cup_{m=1}^{+\infty} A_m =\{\mathbf{x}\in \mathcal{X}: f_{\rm I}(\mathbf{x})>0~ {\rm and} ~ f_{\rm O}(\mathbf{x})>0\}=A_{\rm overlap},
\end{equation*}
and 
\begin{equation*}
    A_m \subset A_{m+1}.
\end{equation*}

\noindent Therefore, 
\begin{equation*}
\lim_{m\rightarrow +\infty}\mu(A_m)=\mu(A_{\rm overlap})>0,
\end{equation*}
which implies that there exists $m_0$ such that
\begin{equation*}
    \mu(A_{m_0})>0.
\end{equation*}

\noindent For any $\alpha\in(0,1)$, we define 
$
    c_{\alpha}= \min_{y_1\in \mathcal{Y}_{\rm all}} \big((1-\alpha)\min_{y_2\in \mathcal{Y}} \ell(y_1,y_2)+\alpha \ell(y_1,K+1)\big).
$
It is clear that $c_{\alpha}>0$ for $\alpha\in(0,1)$. Then, for any $h\in \mathcal{H}$,
\begin{equation*}
\begin{split}
    &~~~~~~R_D^{\alpha}(h)\\&=\int_{\mathcal{X}\times \mathcal{Y}_{\rm all}} \ell(h(\mathbf{x}),y){\rm d} D^{\alpha}_{XY}(\mathbf{x},y)\\&=\int_{\mathcal{X}\times \mathcal{Y}} (1-\alpha) \ell(h(\mathbf{x}),y){\rm d} D_{X_{\rm I}Y_{\rm I}}(\mathbf{x},y)+\int_{\mathcal{X}\times \{K+1\}} \alpha \ell(h(\mathbf{x}),y){\rm d} D_{X_{\rm O}Y_{\rm O}}(\mathbf{x},y)\\&\geq \int_{A_{m_0}\times \mathcal{Y}} (1-\alpha) \ell(h(\mathbf{x}),y){\rm d} D_{X_{\rm I}Y_{\rm I}}(\mathbf{x},y)+\int_{A_{m_0}\times \{K+1\}} \alpha \ell(h(\mathbf{x}),y){\rm d} D_{X_{\rm O}Y_{\rm O}}(\mathbf{x},y)\\ & = \int_{A_{m_0}} \big((1-\alpha)\int_{ \mathcal{Y}}  \ell(h(\mathbf{x}),y){\rm d} D_{Y_{\rm I}|X_{\rm I}}(y|\mathbf{x})\big){\rm d}D_{X_{\rm I}}(\mathbf{x})\\&+ \int_{A_{m_0}} \alpha \ell(h(\mathbf{x}),K+1){\rm d} D_{X_{\rm O}}(\mathbf{x})\\ & \geq \int_{A_{m_0}} (1-\alpha)\min_{y_2\in \mathcal{Y}}\ell(h(\mathbf{x}),y_2){\rm d}D_{X_{\rm I}}(\mathbf{x})+ \int_{A_{m_0}} \alpha \ell(h(\mathbf{x}),K+1){\rm d} D_{X_{\rm O}}(\mathbf{x})
    \\ & \geq \int_{A_{m_0}} (1-\alpha)\min_{y_2\in \mathcal{Y}}\ell(h(\mathbf{x}),y_2)f_{\rm I}(\mathbf{x}){\rm d}\mu(\mathbf{x})+ \int_{A_{m_0}} \alpha \ell(h(\mathbf{x}),K+1)f_{\rm O}(\mathbf{x}){\rm d} \mu(\mathbf{x})\\ & \geq \frac{1}{m_0} \int_{A_{m_0}} (1-\alpha)\min_{y_2\in \mathcal{Y}}\ell(h(\mathbf{x}),y_2){\rm d}\mu(\mathbf{x})+ \frac{1}{m_0}\int_{A_{m_0}} \alpha \ell(h(\mathbf{x}),K+1){\rm d} \mu (\mathbf{x})\\ & = \frac{1}{m_0} \int_{A_{m_0}} \big( (1-\alpha)\min_{y_2\in \mathcal{Y}}\ell(h(\mathbf{x}),y_2)+ \alpha \ell(h(\mathbf{x}),K+1) \big) {\rm d} \mu (\mathbf{x}) \geq \frac{c_{\alpha}}{m_0}\mu(A_{m_0})>0.
    \end{split}
\end{equation*}
Therefore,
\begin{equation*}
    \inf_{h\in \mathcal{H}} R_D^{\alpha}(h)\geq \frac{c_{\alpha}}{m_0}\mu(A_{m_0})>0.
\end{equation*}

\noindent \textbf{Third}, Condition \ref{C1} indicates that  $\inf_{h\in \mathcal{H}} R_{D}^{\alpha}(h)=(1-\alpha)\inf_{h\in \mathcal{H}}R_{D}^{\rm in}(h)+\alpha \inf_{h\in \mathcal{H}}R_{D}^{\rm in}(h) = 0 $ (here we have used conditions $\inf_{h\in \mathcal{H}}R_D^{\rm in}(h)=0$ and $\inf_{h\in \mathcal{H}}R_D^{\rm out}(h)=0$), which contradicts with $\inf_{h\in \mathcal{H}} R_{D}^{\alpha}(h)>0$ ($\alpha\in (0,1)$). Therefore, Condition \ref{C1} does not hold. Using Lemma \ref{C1andC2}, we obtain that OOD detection in $\mathscr{D}_{XY}$ is not learnable  for $\mathcal{H}$.
\end{proof}

\subsection{Proof of Theorem \ref{T4}}

\thmImpTotal*
\begin{proof}[Proof of Theorem \ref{T4}]
We need to prove that OOD detection is not learnable in the total space $\mathscr{D}_{XY}^{\rm all}$ for $\mathcal{H}$, if $\mathcal{H}$ is non-trivial, \textit{i.e.},
$
\{\mathbf{x}\in \mathcal{X}:\exists h_1,h_2\in \mathcal{H}, \textnormal{s.t.}~ h_1(\mathbf{x})\in \mathcal{Y}, h_2(\mathbf{x})=K+1\}\neq \emptyset.
$ The main idea is to construct a domain $D_{XY}$ satisfying that:
\\
\\
1) the ID and OOD distributions have overlap (Definition {\ref{def:overlap}}); 
\\
\\
2) $R_{D}^{\rm in}(h_1)=0$, $ R_{D}^{\rm out}(h_2)=0$.
\\
\\
\noindent According to the condition that $\mathcal{H}$ is non-trivial, we know that there exist $h_1, h_2\in \mathcal{H}$ such that $h_1(\mathbf{x}_1)\in \mathcal{Y}, h_2(\mathbf{x}_1)=K+1$, for some $\mathbf{x}_1\in \mathcal{X}$. We set $D_{XY}= 0.5*\delta_{(\mathbf{x}_1,h_1(\mathbf{x}_1))}+0.5*\delta_{(\mathbf{x}_1,h_2(\mathbf{x}_1))}$, where $\delta$ is the Dirac measure. It is easy to check that $R_{D}^{\rm in}(h_1)=0$, $R_{D}^{\rm out}(h_2)=0$, which implies that $\inf_{h\in \mathcal{H}}R_{D}^{\rm in}(h)=0$ and $\inf_{h\in \mathcal{H}}R_{D}^{\rm out}(h)=0$. In addition, the ID distribution $\delta_{(\mathbf{x}_1,h_1(\mathbf{x}_1))}$ and OOD distribution $\delta_{(\mathbf{x}_1,h_2(\mathbf{x}_1))}$ have overlap $\mathbf{x}_1$.
By using Theorem \ref{T5}, we have completed this proof.
\end{proof}

\section{Proof of Theorem \ref{T12}}\label{SH}
Before proving Theorem \ref{T12}, we need three important lemmas.
\begin{lemma}\label{lemma10} Suppose that $D_{XY}$ is a domain with OOD convex decomposition $Q_1,...,Q_l$ (convex decomposition is given by Definition \ref{OODCD} in Appendix \ref{SD}), and $D_{XY}$ is a finite discrete distribution, then (the definition of $f_{D,Q}$ is given in Condition \ref{C2})
 \begin{equation*}
    f_{D,Q}(\alpha_1,...,\alpha_{l})= (1-\sum_{j=1}^{l}\alpha_j) f_{D,Q}(\mathbf{0})+\sum_{j=1}^{l}\alpha_j f_{D,Q}({\bm \alpha}_j),~~~\forall (\alpha_1,...,\alpha_{l})\in \Delta_l^{\rm o}, 
    \end{equation*}
    \textbf{if and only if} 
    \begin{equation*}
     \argmin_{h\in \mathcal{H}} R_D(h) = \bigcap_{j=1}^l \argmin_{h\in \mathcal{H}} R_{Q_j}(h) \bigcap \argmin_{h\in \mathcal{H}} R_D^{\rm in}(h),
    \end{equation*}
    where $\mathbf{0}$ is the $1\times l$ vector, whose elements are $0$, and ${\bm \alpha}_j$ is the $1\times l$ vector, whose $j$-th element is $1$ and other elements are $0$, and 
 \begin{equation*}   
 R_{Q_j}(h) = \int_{\mathcal{X}\times \{K+1\}} \ell(h(\mathbf{x}),y){\rm d}Q_j(\mathbf{x},y).
\end{equation*}
\end{lemma}
\begin{proof}[Proof of Lemma \ref{lemma10}] To better understand this proof, we recall the definition of  $f_{D,Q}(\alpha_1,...,\alpha_{l})$:
   \begin{equation*}
    f_{D,Q}(\alpha_1,...,\alpha_{l})= \inf_{h\in \mathcal{H}}  \Big ((1-\sum_{j=1}^{l}\alpha_j) R_D^{\rm in}(h) + \sum_{j=1}^l \alpha_j R_{Q_j}(h) \Big ),~~~\forall (\alpha_1,...,\alpha_{l})\in\Delta_l^{\rm o}
    \end{equation*}
\textbf{First}, we prove that if \begin{equation*}
    f_{D,Q}(\alpha_1,...,\alpha_{l})= (1-\sum_{j=1}^{l}\alpha_j) f_{D,Q}(\mathbf{0})+\sum_{j=1}^{l}\alpha_j f_{D,Q}({\bm \alpha}_j),~~~\forall (\alpha_1,...,\alpha_{l})\in \Delta_l^{\rm o},
    \end{equation*}
   then,
   \begin{equation*}
     \argmin_{h\in \mathcal{H}} R_D(h) = \bigcap_{j=1}^l \argmin_{h\in \mathcal{H}}R_{Q_j}(h) \bigcap \argmin_{h\in \mathcal{H}} R_D^{\rm in}(h).
    \end{equation*}
    
\noindent    
Let $D_{XY}=(1-\sum_{j=1}^l \lambda_j)D_{X_{\rm I}Y_{\rm I}}+\sum_{j=1}^l \lambda_j Q_j$, for some $(\lambda_1,...,\lambda_l)\in \Delta_l^{\rm o}$. Since $D_{XY}$ has finite support set, we have
\begin{equation*}
     \argmin_{h\in \mathcal{H}} R_D(h) = \argmin_{h\in \mathcal{H}}\Big ((1-\sum_{j=1}^l \lambda_j) R_D^{\rm in}(h) + \sum_{j=1}^l \lambda_j R_{Q_j}(h)\Big) \neq \emptyset.
\end{equation*}

\noindent We can find that $h_{0}\in   \argmin_{h\in \mathcal{H}}\Big ((1-\sum_{j=1}^l \lambda_j) R_D^{\rm in}(h) + \sum_{j=1}^l \lambda_j R_{Q_j}(h)\Big)$. Hence,
    \begin{equation}\label{threeinequality-1}
     (1-\sum_{j=1}^l \lambda_j) R_D^{\rm in}(h_{0}) + \sum_{j=1}^l \lambda_j R_{Q_j}(h_{0})  = \inf_{h\in \mathcal{H}}  \Big ((1-\sum_{j=1}^l \lambda_j) R_D^{\rm in}(h) + \sum_{j=1}^l \lambda_j R_{Q_j}(h) \Big ) .
    \end{equation}
    Note that the condition $f_{D,Q}(\alpha_1,...,\alpha_{l})= (1-\sum_{j=1}^{l}\alpha_j) f_{D,Q}(\mathbf{0})+\sum_{j=1}^{l}\alpha_j f_{D,Q}({\bm \alpha}_j)$ implies 
    \begin{equation}\label{threeinequality-2}
    (1-\sum_{j=1}^l \lambda_j)\inf_{h\in \mathcal{H}}  R_D^{\rm in}(h) + \sum_{j=1}^{l} \lambda_j \inf_{h\in \mathcal{H}}  R_{Q_j}(h)   = \inf_{h\in \mathcal{H}}  \Big ( (1-\sum_{j=1}^l \lambda_j) R_D^{\rm in}(h) + \sum_{j=1}^{l} \lambda_j R_{Q_j}(h) \Big ).
    \end{equation}
    Therefore, Eq. \eqref{threeinequality-1} and Eq. \eqref{threeinequality-2} imply that
    \begin{equation}\label{threeinequality}
    \begin{split}
   (1-\sum_{j=1}^l \lambda_j)\inf_{h\in \mathcal{H}}  R_D^{\rm in}(h) + \sum_{j=1}^{l} \lambda_j \inf_{h\in \mathcal{H}}  R_{Q_j}(h) = (1-\sum_{j=1}^l \lambda_j) R_D^{\rm in}(h_{0}) + \sum_{j=1}^{l} \lambda_j R_{Q_j}(h_{0}).
   \end{split}
    \end{equation}

\noindent   
Since $ R_D^{\rm in}(h_{0})\geq \inf_{h\in \mathcal{H}}  R_D^{\rm in}(h)$ and $R_{Q_j}(h_{0})\geq \inf_{h\in \mathcal{H}}  R_{Q_j}^{\rm in}(h)$, for $j=1,...,l$, then using Eq. \eqref{threeinequality}, we have that
\begin{equation*}
\begin{split}
&R_D^{\rm in}(h_{0})= \inf_{h\in \mathcal{H}}  R_D^{\rm in}(h),
\\
&
R_{Q_j}(h_{0})=  \inf_{h\in \mathcal{H}}  R_{Q_j}(h),~~~\forall j=1,...,l,
\end{split}
\end{equation*}
which implies that
\begin{equation*}
h_{0} \in     \bigcap_{j=1}^{l} \argmin_{h\in \mathcal{H}} R_{Q_j}(h) \bigcap   \argmin_{h\in \mathcal{H}} R_D^{\rm in}(h).
    \end{equation*}
    Therefore, 
    \begin{equation}\label{subsetOOD1}
       \argmin_{h\in \mathcal{H}} R_D(h)\subset  \bigcap_{j=1}^{l} \argmin_{h\in \mathcal{H}} R_{Q_j}(h) \bigcap   \argmin_{h\in \mathcal{H}} R_D^{\rm in}(h).
    \end{equation}
    
 \noindent Additionally, using 
 \begin{equation*}
 f_{D,Q}(\alpha_1,...,\alpha_{l})= (1-\sum_{j=1}^{l}\alpha_j) f_{D,Q}(\mathbf{0})+\sum_{j=1}^{l}\alpha_j f_{D,Q}({\bm \alpha}_j),~\forall (\alpha_1,...,\alpha_{l})\in \Delta_l^{\rm o},
 \end{equation*}
 we obtain that for any $h'\in    \bigcap_{j=1}^{l} \argmin_{h\in \mathcal{H}} R_{Q_j}(h) \bigcap   \argmin_{h\in \mathcal{H}} R_D^{\rm in}(h)$, 
 \begin{equation*}
    \begin{split}
 \inf_{h\in \mathcal{H}}   R_D(h) =& \inf_{h\in \mathcal{H}}  \Big ((1-\sum_{j=1}^{l}\lambda_j) R_D^{\rm in}(h) + \sum_{j=1}^l \lambda_j R_{Q_j}(h) \Big )\\=&(1-\sum_{j=1}^{l}\lambda_j)  \inf_{h\in \mathcal{H}} R_D^{\rm in}(h) + \sum_{j=1}^l \lambda_j  \inf_{h\in \mathcal{H}}  R_{Q_j}(h)\\ =  & (1-\sum_{j=1}^{l}\lambda_j) R_D^{\rm in}(h') + \sum_{j=1}^l \lambda_j R_{Q_j}(h')=R_D(h'),
    \end{split}
    \end{equation*}
    which implies that 
    \begin{equation*}
        h'\in \argmin_{h\in \mathcal{H}} R_D(h).
    \end{equation*}
    Therefore,  
    \begin{equation}\label{subsetOOD2}
    \bigcap_{j=1}^{l} \argmin_{h\in \mathcal{H}} R_{Q_j}(h) \bigcap   \argmin_{h\in \mathcal{H}} R_D^{\rm in}(h) \subset \argmin_{h\in \mathcal{H}} R_D(h).
    \end{equation}
    \\
    \noindent Combining Eq. \eqref{subsetOOD1} with Eq. \eqref{subsetOOD2},  we obtain that 
    \begin{equation*}
    \bigcap_{j=1}^{l} \argmin_{h\in \mathcal{H}} R_{Q_j}(h) \bigcap   \argmin_{h\in \mathcal{H}} R_D^{\rm in}(h)= \argmin_{h\in \mathcal{H}} R_D(h).
    \end{equation*}

\noindent \textbf{Second}, we prove that if
    \begin{equation*}
  \argmin_{h\in \mathcal{H}} R_{D}(h) =    \bigcap_{j=1}^{l}  \argmin_{h\in \mathcal{H}} R_{Q_j}(h) \bigcap  \argmin_{h\in \mathcal{H}} R_D^{\rm in}(h),
    \end{equation*}
   then,
   \begin{equation*}
    f_{D,Q}(\alpha_1,...,\alpha_{l})= (1-\sum_{j=1}^{l}\alpha_j) f_{D,Q}(\mathbf{0})+\sum_{j=1}^{l}\alpha_j f_{D,Q}({\bm \alpha}_j),~~~\forall (\alpha_1,...,\alpha_{l})\in \Delta_l^{\rm o}.
    \end{equation*}

\noindent    We set
    \begin{equation*}
        h_{0} \in \bigcap_{j=1}^{l}  \argmin_{h\in \mathcal{H}} R_{Q_j}(h) \bigcap  \argmin_{h\in \mathcal{H}} R_D^{\rm in}(h),
    \end{equation*}
    then, for any $(\alpha_1,...,\alpha_{l})\in \Delta_l^{\rm o}$,
    \begin{equation*}
    \begin{split}
(1-\sum_{j=1}^{l}\alpha_j)  \inf_{h\in \mathcal{H}} R_D^{\rm in}(h) + \sum_{j=1}^l \alpha_j  \inf_{h\in \mathcal{H}}  R_{Q_j}(h)  &\leq  \inf_{h\in \mathcal{H}}  \Big ((1-\sum_{j=1}^{l}\alpha_j) R_D^{\rm in}(h) + \sum_{j=1}^l \alpha_j R_{Q_j}(h) \Big )\\  & \leq    (1-\sum_{j=1}^{l}\alpha_j) R_D^{\rm in}(h_{0}) + \sum_{j=1}^l \alpha_j R_{Q_j}(h_{0})\\  &=(1-\sum_{j=1}^{l}\alpha_j)  \inf_{h\in \mathcal{H}} R_D^{\rm in}(h) + \sum_{j=1}^l \alpha_j  \inf_{h\in \mathcal{H}}  R_{Q_j}(h).
    \end{split}
    \end{equation*}
    Therefore, for any $(\alpha_1,...,\alpha_{l})\in \Delta_l^{\rm o}$,
      \begin{equation*}
    \begin{split}
& (1-\sum_{j=1}^{l}\alpha_j)  \inf_{h\in \mathcal{H}} R_D^{\rm in}(h) + \sum_{j=1}^l \alpha_j  \inf_{h\in \mathcal{H}}  R_{Q_j}(h) = \inf_{h\in \mathcal{H}}  \Big ((1-\sum_{j=1}^{l}\alpha_j) R_D^{\rm in}(h) + \sum_{j=1}^l \alpha_j R_{Q_j}(h) \Big ),
    \end{split}
    \end{equation*}
    which implies that: for any $(\alpha_1,...,\alpha_{l})\in \Delta_l^{\rm o}$, 
    \begin{equation*}
        f_{D,Q}(\alpha_1,...,\alpha_{l})= (1-\sum_{j=1}^{l}\alpha_j) f_{D,Q}(\mathbf{0})+\sum_{j=1}^{l}\alpha_j f_{D,Q}({\bm \alpha}_j).
    \end{equation*}
    \noindent We have completed this proof.
\end{proof}
\newpage

\begin{lemma}
\label{T9}
Suppose that Assumption \ref{ass1} holds.
If there is a finite discrete domain $D_{XY}\in \mathscr{D}_{XY}^{s}$ such that
$
    \inf_{h\in \mathcal{H}}R_{D}^{\rm out}({\bm h})>0,
$ then OOD detection is not learnable in  $\mathscr{D}_{XY}^{s}$ for $\mathcal{H}$.
\end{lemma}
\begin{proof}[Proof of Lemma \ref{T9}]
Suppose that ${\rm supp} D_{X_{\rm O}}=\{\mathbf{x}_1^{\rm out},...,\mathbf{x}_l^{\rm out}\}$, then it is clear that $D_{XY}$ has OOD convex decomposition $\delta_{\mathbf{x}_1^{\rm out}},...,\delta_{\mathbf{x}_l^{\rm out}}$, where $\delta_{\mathbf{x}}$ is the dirac measure whose support set is $\{\mathbf{x}\}$.
\\
\\
\noindent Since $\mathcal{H}$ is the separate space for OOD (\textit{i.e.}, Assumption \ref{ass1} holds), then $\forall j=1,...,l$, 
\begin{equation*}
\inf_{h\in \mathcal{H}} R_{\delta_{\mathbf{x}_j^{\rm out}}}(h)=0,
\end{equation*}
where 
\begin{equation*}
    R_{\delta_{\mathbf{x}_j^{\rm out}}}(h) = \int_{\mathcal{X}} \ell(h(\mathbf{x}),K+1){\rm d}\delta_{\mathbf{x}_j^{\rm out}}(\mathbf{x}).
\end{equation*}
\\
\noindent This implies that: if $\bigcap_{j=1}^l \argmin_{h\in \mathcal{H}} R_{\delta_{\mathbf{x}_j^{\rm out}}}(h)\neq \emptyset$, then for $\forall h'\in \bigcap_{j=1}^l \argmin_{h\in \mathcal{H}} R_{\delta_{\mathbf{x}_j^{\rm out}}}(h)$,
\begin{equation*}
   h'(\mathbf{x}_i^{\rm out}) = K+1, ~\forall i=1,...,l.
\end{equation*}
Therefore, if $ \bigcap_{j=1}^l \argmin_{h\in \mathcal{H}} R_{\delta_{\mathbf{x}_j^{\rm out}}}(h) \bigcap \argmin_{h\in \mathcal{H}} R_D^{\rm in}(h)\neq \emptyset$, \\$~~~~~~~~~~~~~~~~~~$then for any $h^*\in \bigcap_{j=1}^l \argmin_{h\in \mathcal{H}} R_{\delta_{\mathbf{x}_j^{\rm out}}}(h) \bigcap \argmin_{h\in \mathcal{H}} R_D^{\rm in}(h)$, we have that
\begin{equation*}
   h^*(\mathbf{x}_i^{\rm out}) = K+1,~ \forall i=1,...,l.
\end{equation*}
\\
\noindent \textbf{Proof by Contradiction}: {assume} OOD detection is learnable in  $\mathscr{D}_{XY}^s$ for $\mathcal{H}$, then  Lemmas \ref{C1andC2} and \ref{lemma10} imply that
\begin{equation*}
     \bigcap_{j=1}^l \argmin_{h\in \mathcal{H}} R_{\delta_{\mathbf{x}_j^{\rm out}}}(h) \bigcap \argmin_{h\in \mathcal{H}} R_D^{\rm in}(h)= \argmin_{h\in \mathcal{H}} R_D(h) \neq \emptyset.
    \end{equation*}
    Therefore, for any $h^*\in  \argmin_{h\in \mathcal{H}} R_D(h)$, we have that 
    \begin{equation*}
   h^*(\mathbf{x}_i^{\rm out}) = K+1,~\forall i=1,...,l,
\end{equation*}
which implies that  for any $h^*\in  \argmin_{h\in \mathcal{H}} R_D(h)$, we have
$
   R_D^{\rm out}(h^*) = 0,
$
which implies that $\inf_{h\in \mathcal{H}} R_D^{\rm out}(h)=0$.
It is clear that $\inf_{h\in \mathcal{H}} R_D^{\rm out}(h)=0$ is \textbf{inconsistent} with the condition $\inf_{h\in \mathcal{H}} R_D^{\rm out}(h)>0$. Therefore, 
OOD detection is not learnable in  $\mathscr{D}_{XY}^s$ for $\mathcal{H}$.
\end{proof}

\begin{lemma}\label{LemmaforTheorem5}
If Assumption \ref{ass1} holds, 
  ${\rm VCdim}(\phi\circ \mathcal{H})=v<+\infty$ and $\sup_{{ h}\in \mathcal{H}} |\{\mathbf{x}\in \mathcal{X}: { h}(\mathbf{x})\in \mathcal{Y}\}| > m$ such that $v<m$, then
  OOD detection is not learnable in $\mathscr{D}_{XY}^{s}$ for $\mathcal{H}$, where $\phi$ maps ID's labels to ${1}$ and maps OOD's labels to $2$.
\end{lemma}

\begin{proof}[Proof of Lemma \ref{LemmaforTheorem5}]
Due to $\sup_{{ h}\in \mathcal{H}} |\{\mathbf{x}\in \mathcal{X}: {\bm h}(\mathbf{x})\in \mathcal{Y}\}| > m$, we can obtain a set
\begin{equation*}
   C=\{\mathbf{x}_1,...,\mathbf{x}_m,\mathbf{x}_{m+1}\},
\end{equation*}
which satisfies that there exists $\tilde{h}\in \mathcal{H}$ such that $\tilde{h}(\mathbf{x}_i)\in \mathcal{Y}$ for any $i=1,...,m,m+1$.

\noindent Let $\mathcal{H}_C^{\phi}=\{(\phi\circ h(\mathbf{x}_1),...,\phi\circ h(\mathbf{x}_m),\phi\circ h(\mathbf{x}_{m+1}):h\in \mathcal{H}\}$. It is clear that 
\begin{equation*}
(1,1,...,1)=(\phi\circ \tilde{h}(\mathbf{x}_1),...,\phi\circ \tilde{h}(\mathbf{x}_m),\phi\circ \tilde{h}(\mathbf{x}_{m+1}))\in\mathcal{H}_C^{\phi}, 
\end{equation*}
where $(1,1,...,1)$ means all elements are $1$.
\\
\\
\noindent Let $\mathcal{H}_{m+1}^{\phi}=\{(\phi\circ h(\mathbf{x}_1),...,\phi\circ h(\mathbf{x}_m),\phi\circ h(\mathbf{x}_{m+1}):h~\textnormal{is~any~hypothesis~function~from~}\mathcal{X}~\textnormal{to}~\mathcal{Y}_{\rm all}\}$.

\noindent Clearly, $\mathcal{H}_C^{\phi}\subset \mathcal{H}_{m+1}^{\phi}$ and $|\mathcal{H}_{m+1}^{\phi}|=2^{m+1}$. Sauer-Shelah-Perles Lemma (Lemma 6.10 in \citep{shalev2014understanding}) implies that 
\begin{equation*}
    |\mathcal{H}^{\phi}_C|\leq \sum_{i=0}^v \tbinom{m+1}{i}.
\end{equation*}
Since $
    \sum_{i=0}^v \tbinom{m+1}{i}<2^{m+1}-1
$ (because $v<m$), we obtain that $|\mathcal{H}^{\phi}_C|\leq2^{m+1}-2$. Therefore, $\mathcal{H}^{\phi}_C\cup \{(2,2...,2)\}$ is a proper subset of $\mathcal{H}_{m+1}^{\phi}$, where $(2,2,...,2)$ means that all elements are $2$.  Note that $(1,1...,1)$ (all elements are 1) also belongs to $\mathcal{H}_C^{\phi}$. Hence, $\mathcal{H}^{\phi}_C\cup \{(2,2...,2)\}\cup \{(1,1...,1)\}$ is a proper subset of $\mathcal{H}_{m+1}^{\phi}$, which implies that we can obtain a hypothesis function $h'$ satisfying that:
\begin{equation*}
\begin{split}
&1) (\phi\circ h'(\mathbf{x}_1),...,\phi\circ h'(\mathbf{x}_m),\phi\circ h'(\mathbf{x}_{m+1}))\notin \mathcal{H}^{\phi}_C;
\\
&2) \textnormal{~ There~ exist}~  \mathbf{x}_j,\mathbf{x}_p\in C \textnormal{~ such ~that}~ \phi\circ h'(\mathbf{x}_j)=2~ {\rm and}~ \phi\circ h'(\mathbf{x}_p)=1.
\end{split}
\end{equation*}

\noindent Let $C_{\rm I}=C\cap \{\mathbf{x}\in \mathcal{X}:\phi\circ h'(\mathbf{x})=1\}$ and $C_{\rm O}=C\cap \{\mathbf{x}\in \mathcal{X}:\phi\circ h'(\mathbf{x})=2\}$.
\\
\\
\noindent Then, we construct a special domain $D_{XY}$:
\begin{equation*}
    D_{XY} = 0.5*D_{X_{\rm I}}*D_{Y_{\rm I}|X_{\rm I}}+0.5* D_{X_{\rm O}}*D_{Y_{\rm O}|X_{\rm O}},~\textnormal{where}
\end{equation*}
\begin{equation*}
    D_{X_{\rm I}}= \frac{1}{|C_{\rm I}|}\sum_{\mathbf{x}\in C_{\rm I}} \delta_{\mathbf{x}}~~~\textnormal{and}~~~D_{Y_{\rm I}|X_{\rm I}}(y|\mathbf{x})=1,~\textnormal{if}~~\tilde{h}(\mathbf{x})=y~~\textnormal{and}~~\mathbf{x}\in C_{\rm I};
\end{equation*}
and
\begin{equation*}
    D_{X_{\rm O}}= \frac{1}{|C_{\rm O}|}\sum_{\mathbf{x}\in C_{\rm O}} \delta_{\mathbf{x}}~~~\textnormal{and}~~~D_{Y_{\rm O}|X_{\rm O}}(K+1|\mathbf{x})=1,~\textnormal{if}~~\mathbf{x}\in C_{\rm O}.
\end{equation*}

\noindent Since $D_{XY}$ is a finite discrete distribution and $(\phi\circ h'(\mathbf{x}_1),...,\phi\circ h'(\mathbf{x}_m),\phi\circ h'(\mathbf{x}_{m+1}))\notin \mathcal{H}^{\phi}_C$, it is clear that $\argmin_{h\in \mathcal{H}} R_D(h)\neq \emptyset$ and $\inf_{h\in \mathcal{H}} R_D(h)>0$. 
 Additionally, $R_D^{\rm in}(\tilde{h})=0$. Therefore, $\inf_{h\in \mathcal{H}}R_D^{\rm in}(h)=0$.
\\
\\
\noindent \textbf{Proof by Contradiction}: {suppose} that  OOD detection is learnable in  $\mathscr{D}_{XY}^s$ for $\mathcal{H}$, then  Lemma \ref{C1andC2} implies that
\begin{equation*}
    \inf_{h\in \mathcal{H}} R_D(h) = 0.5* \inf_{h\in \mathcal{H}} R_D^{\rm in}(h)+ 0.5*\inf_{h\in \mathcal{H}} R_D^{\rm out}(h).
    \end{equation*}
 Therefore, if  OOD detection is learnable in  $\mathscr{D}_{XY}^s$ for $\mathcal{H}$, then $\inf_{h\in \mathcal{H}} R_D^{\rm out}(h)>0$.

\noindent
 Until now, we have constructed a domain $D_{XY}$ (defined over $\mathcal{X}\times \mathcal{Y}_{\rm all}$) with finite support and satisfying that $\inf_{h\in \mathcal{H}} R_D^{\rm out}(h)>0$. Note that $\mathcal{H}$ is the separate space for OOD data (Assumption \ref{ass1} holds).  Using Lemma \ref{T9}, we know that OOD detection is not learnable in  $\mathscr{D}_{XY}^s$ for $\mathcal{H}$, which is \textbf{inconsistent} with our assumption that   OOD detection is learnable in  $\mathscr{D}_{XY}^s$ for $\mathcal{H}$. Therefore,  OOD detection is not learnable in  $\mathscr{D}_{XY}^s$ for $\mathcal{H}$. We have completed the proof.
\end{proof}
\thmImpSeptwo*
\begin{proof}[Proof of Theorem \ref{T12}]
Let ${\rm VCdim}(\phi\circ \mathcal{H})=v$. Since $\sup_{{ h}\in \mathcal{H}} |\{\mathbf{x}\in \mathcal{X}: {\bm h}(\mathbf{x})\in \mathcal{Y}\}|=+\infty$, it is clear that $\sup_{{ h}\in \mathcal{H}} |\{\mathbf{x}\in \mathcal{X}: {\bm h}(\mathbf{x})\in \mathcal{Y}\}|>v$. Using Lemma \ref{LemmaforTheorem5}, we complete this proof.
\end{proof}

\section{Proofs of Lemma \ref{T5_auc} and Theorem \ref{T4_auc} }\label{SF_auc}

\subsection{Proof of Lemma \ref{T5_auc}}
\begin{Definition}
    Given a ranking function space $\mathcal{R}$, the corresponding hypothesis space $\mathcal{G}$ consists of all $g_r$ satisfying that there exists a $r\in \mathcal{R}$ such that
    \begin{equation*}
        g_r(\mathbf{x},\mathbf{x}') = {\rm sign}(r(\mathbf{x})-r(\mathbf{x}')).
    \end{equation*}
\end{Definition}
\begin{Definition}[Equivalence]\label{D30}
 Given measures $\mu_1, \mu_2$,   we say two ranking functions $r\in \mathcal{R}$ and $r'\in \mathcal{R}$ are AUC equivalent over $\mu_1,\mu_2$, i.e., $f \sim f'$ w.r.t. $\mu_1,\mu_2$, if and only if the corresponding hypothesis functions $g_f = g_{f'}$ a.e. $\mu_1\times \mu_2$.
\end{Definition}

\begin{lemma}\label{L1_auc}
Assume that $D_{X_{\rm I}} = \int g_{X_{\rm I}} {\rm d}\mu $ and $D_{X_{\rm O}} = \int g_{X_{\rm O}} {\rm d}\mu $, then 
\begin{equation*}
 \sup_{r\in \mathcal{R}_{\rm all}} {\rm AUC}(f;D_{X_{\rm I}},D_{X_{\rm O}}) =  \frac{1}{2}\mathbb{E}_{\mathbf{x}\sim \mu}\mathbb{E}_{\mathbf{x}'\sim \mu}\max \{ g_{X_{\rm I}}(\mathbf{x})g_{X_{\rm O}}(\mathbf{x}'), g_{X_{\rm I}}(\mathbf{x}')g_{X_{\rm O}}(\mathbf{x})\}.
\end{equation*}
\end{lemma}
\begin{proof}
  Let $D(X) = \int g {\rm d}\mu$, $D_{X_{\rm I}} = \int g_{X_{\rm I}} {\rm d}\mu $ and $D_{X_{\rm O}} = \int g_{X_{\rm O}} {\rm d}\mu $ satisfying that $g_{X_{\rm I}}+g_{X_{\rm O}} = 2g$. {Let} $P=\{\mathbf{x}\in \mathcal{X}: g(\mathbf{x})>0\}$, $P_{X_{\rm I}}=\{\mathbf{x}\in \mathcal{X}: g_{X_{\rm I}}(\mathbf{x})>0\}$ and $P_{X_{\rm O}}=\{\mathbf{x}\in \mathcal{X}: g_{X_{\rm O}}(\mathbf{x})>0\}$. To any two points $\mathbf{x}$ and $\mathbf{x}'$, we consider that
  \begin{equation*}
  \begin{split}
    R(r, \mathbf{x},\mathbf{x}')& =  \big [ \mathbf{1}_{r(\mathbf{x})>r(\mathbf{x}')}  + \frac{1}{2} \mathbf{1}_{r(\mathbf{x})=r(\mathbf{x}')}\big] g_{X_{\rm I}}(\mathbf{x})g_{X_{\rm O}}(\mathbf{x}')\\& +  \big [ \mathbf{1}_{r(\mathbf{x}')>r(\mathbf{x})}  + \frac{1}{2} \mathbf{1}_{r(\mathbf{x})=r(\mathbf{x}')}\big] g_{X_{\rm I}}(\mathbf{x}')g_{X_{\rm O}}(\mathbf{x}).
    \end{split}
  \end{equation*}
 {We set} $r^*(\mathbf{x}) = {\rm sigmoid} \big ( {g_{X_{\rm I}}(\mathbf{x})}/{g_{X_{\rm O}}(\mathbf{x})} \big )$, if $g_{X_{\rm O}}(\mathbf{x})>0$; otherwise, $r^*(\mathbf{x})=1$.
  It is clear that 
  we have that
  \begin{equation*}
       R(r^*, \mathbf{x},\mathbf{x}') = \max_{r\in \mathcal{R}_{\rm all}} R(r, \mathbf{x},\mathbf{x}') = \max \{ g_{X_{\rm I}}(\mathbf{x})g_{X_{\rm O}}(\mathbf{x}'), g_{X_{\rm I}}(\mathbf{x}')g_{X_{\rm O}}(\mathbf{x})\}.
  \end{equation*}
  Due to
  \begin{equation*}
      2{\rm AUC}(r;D_{X_{\rm I}},D_{X_{\rm O}}) = \int_{P} \int_{P} R(r, \mathbf{x},\mathbf{x}') {\rm d} \mu(\mathbf{x}){\rm d} \mu(\mathbf{x}'),
  \end{equation*}
  then 
  \begin{equation*}
  \begin{split}
      &2{\rm AUC}(r^*;D_{X_{\rm I}},D_{X_{\rm O}}) = \int_{P} \int_{P} \max_{r\in \mathcal{R}_{\rm all}} R(r, \mathbf{x},\mathbf{x}') {\rm d} \mu(\mathbf{x}){\rm d} \mu(\mathbf{x}')\\ \geq &\max_{r\in \mathcal{R}_{\rm all}} \int_{P} \int_{P}  R(r, \mathbf{x},\mathbf{x}') {\rm d} \mu(\mathbf{x}){\rm d} \mu(\mathbf{x}') = \max_{r\in \mathcal{R}_{\rm all}} 2{\rm AUC}(r;D_{X_{\rm I}},D_{X_{\rm O}}).
      \end{split}
  \end{equation*}
  Therefore, $r^*$ is the optimal solution, and 
  \begin{equation*}
      {\rm AUC}(r^*;D_{X_{\rm I}},D_{X_{\rm O}}) = \frac{1}{2}\mathbb{E}_{\mathbf{x}\sim \mu}\mathbb{E}_{\mathbf{x}'\sim \mu}\max \{ g_{X_{\rm I}}(\mathbf{x})g_{X_{\rm O}}(\mathbf{x}'), g_{X_{\rm I}}(\mathbf{x}')g_{X_{\rm O}}(\mathbf{x})\}. 
  \end{equation*}
We have completed this proof.
\end{proof}
\begin{lemma}\label{l1-auc}
    {Given a ranking function space} $\mathcal{R}\subset \mathcal{R}_{\rm all}$, $D_{X_{\rm I}}=\int g_{X_{\rm I}} {\rm d}\mu$, $D_{X_{\rm O}}=\int g_{X_{\rm O}} {\rm d}\mu$ and $D_{X_{\rm O}}'=\int g_{X_{\rm O}}' {\rm d}\mu$,
    if
    \begin{equation*}
    \begin{split}
        \sup_{r\in \mathcal{R}} {\rm AUC}(r; D_{X_{\rm I}},D_{X_{\rm O}}) = \sup_{r\in \mathcal{R}_{\rm all}} {\rm AUC}(r; D_{X_{\rm I}},D_{X_{\rm O}}),\\ \sup_{r\in \mathcal{R}} {\rm AUC}(r; D_{X_{\rm I}},D_{X_{\rm O}}') = \sup_{r\in \mathcal{R}_{\rm all}} {\rm AUC}(r; D_{X_{\rm I}},D_{X_{\rm O}}'),
        \end{split}
    \end{equation*}
    and there exists $\alpha \in (0,1)$ such that
    \begin{equation*}
       \alpha \sup_{r\in \mathcal{R}} {\rm AUC}(r; D_{X_{\rm I}},D_{X_{\rm O}}) +(1-\alpha) \sup_{r\in \mathcal{R}} {\rm AUC}(r; D_{X_{\rm I}},D_{X_{\rm O}}') = \sup_{r\in \mathcal{R}} {\rm AUC}(r; D_{X_{\rm I}},{D}_{X_{\rm O}}^{\alpha}),
    \end{equation*}
    {where} ${D}_{X_{\rm O}}^{\alpha}= \alpha D_{X_{\rm O}} +(1-\alpha) D_{X_{\rm O}}'$.
    Then
    \begin{equation*}
        \frac{g_{X_{\rm I}}}{g_{X_{\rm I}}+g_{X_{\rm O}}}\sim \frac{g_{X_{\rm I}}}{g_{X_{\rm I}}+g_{X_{\rm O}}'}~ \text{w.r.t.}~ ~D_{X_{\rm I}}|_{P_{X_{\rm O}}'-P_{X_{\rm O}}}, D_{X_{\rm I}}|_{P_{X_{\rm O}}-P_{X_{\rm O}}'},
    \end{equation*}
    where $P_{X_{\rm O}}=\{\mathbf{x}:g_{X_{\rm O}}(\mathbf{x})>0\}$ and $P_{X_{\rm O}}'=\{\mathbf{x}:g_{X_{\rm O}}'(\mathbf{x})>0\}$
\end{lemma}
\begin{proof}
    Let $\eta_{X_{\rm I}} = \frac{g_{X_{\rm I}}}{g_{X_{\rm I}}+g_{X_{\rm O}}}$ and $\eta_{X_{\rm I}}' = \frac{g_{X_{\rm I}}}{g_{X_{\rm I}}+g_{X_{\rm O}}'}$.
    It is easy to check that the proof process of Lemma \ref{L1_auc} implies that to each $i=1, 2$,
    \begin{equation*}
        \eta_{X_{\rm I}}^i \in \argmax_{r\in \mathcal{R}_{\rm all}}  {\rm AUC}(f; D_{X_{\rm I}},D_{X_{\rm O}}^i).
    \end{equation*}
    Additionally,
    \begin{equation*}
    \begin{split}
    \sup_{r\in \mathcal{R}} {\rm AUC}(r; D_{X_{\rm I}},{D}_{X_{\rm O}}^{\alpha})
         = & \alpha \sup_{r\in \mathcal{R}} {\rm AUC}(r; D_{X_{\rm I}},D_{X_{\rm O}}) +(1-\alpha) \sup_{r\in \mathcal{R}} {\rm AUC}(r; D_{X_{\rm I}},D_{X_{\rm O}}') \\ =  &\alpha \sup_{r\in \mathcal{R}_{\rm all}} {\rm AUC}(r; D_{X_{\rm I}},D_{X_{\rm O}}) +(1-\alpha) \sup_{r\in \mathcal{R}_{\rm all}} {\rm AUC}(r; D_{X_{\rm I}},D_{X_{\rm O}}') \\ \geq &\sup_{r\in \mathcal{R}_{\rm all}} {\rm AUC}(r; D_{X_{\rm I}},{D}_{X_{\rm O}}^{\alpha}) \geq \sup_{r\in \mathcal{R}} {\rm AUC}(r; D_{X_{\rm I}},{D}_{X_{\rm O}}^{\alpha}).
         \end{split}
    \end{equation*}
    Therefore, under the condition that 
     \begin{equation*}
     \begin{split}
       & \sup_{r\in \mathcal{R}} {\rm AUC}(r; D_{X_{\rm I}},D_{X_{\rm O}}) = \sup_{r\in \mathcal{R}_{\rm all}} {\rm AUC}(r; D_{X_{\rm I}},D_{X_{\rm O}}),\\ & \sup_{r\in \mathcal{R}} {\rm AUC}(r; D_{X_{\rm I}},D_{X_{\rm O}}') = \sup_{r\in \mathcal{R}_{\rm all}} {\rm AUC}(r; D_{X_{\rm I}},D_{X_{\rm O}}'),
        \end{split}
    \end{equation*}
    we obtain that
    \begin{equation*}
        \sup_{r\in \mathcal{R}_{\rm all}} {\rm AUC}(r; D_{X_{\rm I}},{D}_{X_{\rm O}}^{\alpha}) = \alpha \sup_{r\in \mathcal{R}_{\rm all}} {\rm AUC}(r; D_{X_{\rm I}},D_{X_{\rm O}}) +(1-\alpha) \sup_{r\in \mathcal{R}_{\rm all}} {\rm AUC}(r; D_{X_{\rm I}},D_{X_{\rm O}}').
    \end{equation*}
    Because $\sup_{r\in \mathcal{R}_{\rm all}} {\rm AUC}(r; D_{X_{\rm I}},D_{X_{\rm O}})$ and $\sup_{r\in \mathcal{R}_{\rm all}} {\rm AUC}(r; D_{X_{\rm I}},D_{X_{\rm O}}) $ are attainable (the proof process of Lemma \ref{L1_auc} implies this), it is easy to check (similar the proof of Lemma \ref{lemma10}) that
    \begin{equation*}
        \argmax_{r\in \mathcal{R}_{\rm all}}{\rm AUC}(r; D_{X_{\rm I}},D_{X_{\rm O}}) \cap \argmax_{r\in \mathcal{R}_{\rm all}}{\rm AUC}(r; D_{X_{\rm I}},D_{X_{\rm O}}')\neq \emptyset.
    \end{equation*}
    Combining with Lemma \ref{L1_auc}, above equality implies there exists $r^*$ such that
    \begin{equation*}
    r^* \sim \eta_{X_{\rm I}},~~\text{w.r.t.}~ D_{X_{\rm I}}, D_{X_{\rm O}},~~~r^* \sim \eta_{X_{\rm I}}',~~\text{w.r.t.}~ D_{X_{\rm I}},D_{X_{\rm O}}'
    \end{equation*}
    Therefore,
    \begin{equation*}
        \eta_{X_{\rm I}} \sim \eta_{X_{\rm I}}'~ \text{w.r.t.}~D_{X_{\rm I}}|_{P_{X_{\rm O}}'-P_{X_{\rm O}}}, D_{X_{\rm I}}|_{P_{X_{\rm O}}-P_{X_{\rm O}}'}
    \end{equation*}
    We have completed the proof.
\end{proof}

\thmImpOneAuc*

\begin{proof}[Proof of Lemma \ref{T5_auc}]
By the condition that $D_{X_{\rm I}}(P_1\cap P_2)< \min \{D_{X_{\rm I}}(P_1), D_{X_{\rm I}}(P_2)\}$, we can ensure that
\begin{equation*}
    D_{X_{\rm I}}(P_1 - P_2)>0,~~~ D_{X_{\rm I}}(P_2 - P_1)>0.
\end{equation*}
By this, one can easily check that
 \begin{equation*}
       \frac{g_{X_{\rm I}}}{g_{X_{\rm I}}+g_{X_{\rm O}}}\sim \frac{g_{X_{\rm I}}}{g_{X_{\rm I}}+g_{X_{\rm O}}'}~ \text{w.r.t.}~D_{X_{\rm I}}|_{P_2-P_1}, D_{X_{\rm I}}|_{P_1-P_2}~\text{does not hold.}
    \end{equation*}
 Therefore, Lemma \ref{l1-auc} implies that Condition \ref{C1_auc} does not hold.
\end{proof}
\subsection{Proof of Theorem \ref{T4_auc}}

\thmImpTotalauc*
\begin{proof}[Proof of Theorem \ref{T4_auc}]
    Let $D_{X_{\rm I}}= \frac{1}{2}\delta_{\mathbf{x}}+\frac{1}{2}\delta_{\mathbf{x}'}$, $D_{X_{\rm O}}= \delta_{\mathbf{x}}$ and $D_{X_{\rm O}'}= \delta_{\mathbf{x}'}$, where $\delta$ is the Dirac measure. Then the condition in Lemma \ref{T5_auc} holds, which implies the result of Theorem \ref{T4_auc}.
\end{proof}
\section{Proof of Theorem \ref{T12_auc}}\label{SH_auc}
\begin{lemma}\label{l2-auc}
Given a separate ranking function space $\mathcal{R}$, if there exists finite discrete $D_{XY} \subset \mathscr{D}_{XY}^s$ such that 
\begin{equation*}
    \sup_{r\in \mathcal{R}} {\rm AUC}(r;D_{XY}) <1,
\end{equation*}
then OOD detection is not learnable under AUC in $\mathscr{D}_{XY}^s$ for $\mathcal{R}$.
\end{lemma}
\begin{proof}[Proof of Lemma \ref{l2-auc}]
   Assume that if OOD detection is learnable under AUC in $\mathscr{D}_{XY}$ for $\mathcal{R}$, then Condition \ref{C1_auc} implies that: let $D_{X_{\rm O}} = \sum_{i=1}^m \lambda_i\delta_{\mathbf{x}_i}$ ($\sum_{i=1}^m =1$),
   \begin{equation*}
    \sup_{r\in \mathcal{R}} {\rm AUC}(r;D_{XY}) =  \sum_{i=1}^m \lambda_i \sup_{r\in \mathcal{R}} {\rm AUC}(r;D_{X_{\rm I}},\delta_{\mathbf{x}_i}).
   \end{equation*}
   Because $\mathcal{R}$ is the separate ranking function space, we know that
   \begin{equation*}
    \sup_{r\in \mathcal{R}} {\rm AUC}(r;D_{XY}) =  \sum_{i=1}^m \lambda_i \sup_{r\in \mathcal{R}} {\rm AUC}(r;D_{X_{\rm I}},\delta_{\mathbf{x}_i})=1,
   \end{equation*}
   which is conflict with the condition that $\sup_{r\in \mathcal{R}} {\rm AUC}(r;D_{XY})<1$. We have completed this proof.
\end{proof}

\thmImpSeptwoauc*
\begin{proof}[Proof of Theorem \ref{T12_auc}]
     Given disjoint samples $\mathbf{X} =\{\mathbf{x}_1,...,\mathbf{x}_m\}$. Consider the following matrix and set
   \begin{equation*}
       \mathbf{B}_r[\mathbf{X}]=\big [\mathbf{1}_{r(\mathbf{x}_i)>r(\mathbf{x}_j)}\big ], ~~~ \mathbf{B}_{\mathcal{R}}[\mathbf{X}]=\{ \mathbf{B}_r[\mathbf{X}] :r\in \mathcal{R}\}.
   \end{equation*}
   Let $\mathcal{D}_{X_{\rm I},X_{\rm O}}[\mathbf{X}]$ be the set consisting of all $(D_{X_{\rm I}}, D_{X_{\rm O}})$ which satisfies the following conditions:
   \begin{itemize}
       \item $D_{X_{\rm I}},D_{X_{\rm O}}$ are from the separate space;
       \item $D_{X_{\rm I}},D_{X_{\rm O}}$ are uniform distributions;
       \item ${\rm supp}D_{X_{\rm I}}\cup {\rm supp}D_{X_{\rm O}} = \mathbf{X}$.
   \end{itemize}
   Additionally, let
   \begin{equation*}
       \mathcal{D}_{\mathcal{R}}[\mathbf{X}] = \{(D_{X_{\rm I}},D_{X_{\rm O}})\in \mathcal{D}_{X_{\rm I},X_{\rm O}}[\mathbf{X}]: \exists r\in \mathcal{R}, {\rm AUC}(r; D_{X_{\rm I}},D_{X_{\rm O}})=1\}.
   \end{equation*}
   It is clear that
   \begin{equation*}
  |\mathbf{B}_{\mathcal{R}}[\mathbf{X}]| \leq   \frac{e^d}{d^d} m^{2d};~~~ |\mathcal{D}_{\mathcal{R}}[\mathbf{X}]|\leq  (m-1)|\mathbf{B}_{\mathcal{R}}[\mathbf{X}]|\leq  \frac{e^d}{d^d} m^{2d+1}; ~~~|\mathcal{D}_{X_{\rm I},X_{\rm O}}[\mathbf{X}] |=  2^{m}-2.
   \end{equation*}
When $m$ is large enough ($m\geq (28d+14)\log(14d+7)$), we have that 
\begin{equation*}
   |\mathcal{D}_{\mathcal{R}}[\mathbf{X}]|<|\mathcal{D}_{X_{\rm I},X_{\rm O}}[\mathbf{X}]| .
\end{equation*}
Therefore, we can find 
\begin{equation*}
    (D_{X_{\rm I}},D_{X_{\rm O}})\in \mathcal{D}_{X_{\rm I},X_{\rm O}}[\mathbf{X}]~ \text{such~that}~ \sup_{r\in \mathcal{R}} {\rm AUC}(r;D_{X_{\rm I}},D_{X_{\rm O}})<1.
\end{equation*}
By Lemma \ref{l2-auc}, we have completed this proof.
\end{proof}

\section{Proofs of Theorem \ref{T13} and Theorem \ref{T17}}\label{SI}
\subsection{Proof of Theorem \ref{T13}}
Firstly, we need two lemmas, which are motivated by Lemma {19.2} and Lemma {19.3} in \citep{shalev2014understanding}.
\begin{lemma}\label{L14}
Let $C_1$,...,$C_{r}$ be a cover of space $\mathcal{X}$, \textit{i.e.}, $\sum_{i=1}^r C_i=\mathcal{X}$. Let $S_{X}=\{\mathbf{x}^1,...,\mathbf{x}^n\}$ be a sequence of $n$ data  drawn from $D_{X_{\rm I}}$, i.i.d. Then 
\begin{equation*}
    \mathbb{E}_{S_X \sim D^n_{X_{\rm I}}} \Big(\sum_{i:C_i \cap S_X =\emptyset} D_{X_{\rm I}}(C_i) \Big) \leq \frac{r}{en}.
\end{equation*}
\end{lemma}
\begin{proof}[Proof of Lemma \ref{L14}]
\begin{equation*}
 \mathbb{E}_{S_X\sim D^n_{X_{\rm I}}} \Big(\sum_{i:C_i \cap S_X =\emptyset} D_{X_{\rm I}}(C_i) \Big) = \sum_{i=1}^r\Big( D_{X_{\rm I}}(C_i) \cdot \mathbb{E}_{S_X\sim D^n_{X_{\rm I}}} \big (\mathbf{1}_{C_i \cap S_X =\emptyset}\big ) \Big),
 \end{equation*}
 where $\mathbf{1}$ is the characteristic function.
 
 \noindent For each $i$,
 \begin{equation*}
 \begin{split}
     \mathbb{E}_{S_X\sim D^n_{X_{\rm I}}} \big (\mathbf{1}_{C_i \cap S_X =\emptyset}\big ) & = \int_{\mathcal{X}^n} \mathbf{1}_{C_i \cap S_X =\emptyset} {\rm d} D^n_{X_{\rm I}}(S_X)\\ &  = \big ( \int_{\mathcal{X}} \mathbf{1}_{C_i \cap \{\mathbf{x}\} =\emptyset} {\rm d} D_{X_{\rm I}}(\mathbf{x}) \big )^n \\ & =\big ( 1- D_{X_{\rm I}}(C_i)\big )^n \leq e^{-nD_{X_{\rm I}}(C_i)}.
     \end{split}
 \end{equation*}
 Therefore, 
 \begin{equation*}
 \begin{split}
     \mathbb{E}_{S_X\sim D^n_{X_{\rm I}}} \Big(\sum_{i:C_i \cap S =\emptyset} D_{X_{\rm I}}(C_i) \Big)  &\leq \sum_{i=1}^r D_{X_{\rm I}}(C_i)e^{-nD_{X_{\rm I}}(C_i)} \\ & \leq r\max_{i\in \{1,...,r\}} D_{X_{\rm I}}(C_i)e^{-nD_{X_{\rm I}}(C_i)}\leq \frac{r}{ne},
     \end{split}
 \end{equation*}
 here we have used inequality: $ \max_{i\in \{1,...,r\}} a_i e^{-na_i}\leq 1/{(ne)}$. The proof has been completed.
\end{proof}
\begin{lemma}\label{L15} Let $K=1$. When $\mathcal{X}\subset \mathbb{R}^d$ is a bounded set, there exists a monotonically decreasing sequence $\epsilon_{\rm cons}(m)$ satisfying that  $\epsilon_{\rm cons}(m) \rightarrow 0$, as $m\rightarrow 0$, such that
\begin{equation*}
    \mathbb{E}_{\mathbf{x}\sim  D_{X_{\rm I}},S\sim D^n_{X_{\rm I}Y_{\rm I}}}{\rm dist}( \mathbf{x},\pi_1(\mathbf{x},S))<\epsilon_{\rm cons}(n),
\end{equation*}
where ${\rm dist}$ is the Euclidean distance, $\pi_1(\mathbf{x},S)= \argmin_{\tilde{\mathbf{x}}\in S_X} {\rm dist}(\mathbf{x},\tilde{\mathbf{x}})$, here $S_X$ is the feature part of $S$, \textit{i.e.}, $S_X=\{\mathbf{x}^1,...,\mathbf{x}^n\}$, if $S=\{(\mathbf{x}^1,y^1),...,(\mathbf{x}^n,y^n)\}$.

\end{lemma}
\begin{proof}[Proof of Lemma \ref{L15}]
Since $\mathcal{X}$ is bounded, without loss of generality, we set  $\mathcal{X}\subset [0,1)^d$. Fix  $\epsilon=1/T$, for some integer $T$. Let $r = T^d$ and $C_1,C_2,...,C_r$ be a cover of $\mathcal{X}$: for every $(a_1,...,a_T)\in [T]^d:=[1,...,T]^d$, there exists a $C_i=\{\mathbf{x}=(x_1,...,x_d):\forall j\in \{1,...,d\}, x_j\in [(a_j-1)/T,a_j/T)\}$. 
\\
\\
\noindent If $\mathbf{x}, \mathbf{x}'$ belong to some $C_i$, then ${\rm dist}(\mathbf{x}, \mathbf{x}') \leq \sqrt{d}\epsilon$; otherwise, ${\rm dist}(\mathbf{x}, \mathbf{x}') \leq \sqrt{d}$. Therefore,
\begin{equation*}
\begin{split}
      &\mathbb{E}_{\mathbf{x}\sim  D_{X_{\rm I}},S\sim D^n_{X_{\rm I}Y_{\rm I}}}{\rm dist}( \mathbf{x},\pi_1(\mathbf{x},S)) 
      \\ \leq & \mathbb{E}_{S\sim D^n_{X_{\rm I}Y_{\rm I}}} \Big ( \sqrt{d}\epsilon \sum_{i:C_i \cap S_X  \neq \emptyset} D_{X_{\rm I}}(C_i) + \sqrt{d}\sum_{i:C_i \cap S_X  = \emptyset} D_{X_{\rm I}}(C_i) \Big )  \\ \leq &\mathbb{E}_{S_X \sim D^n_{X_{\rm I}}} \Big ( \sqrt{d}\epsilon \sum_{i:C_i \cap S_X  \neq \emptyset} D_{X_{\rm I}}(C_i) + \sqrt{d}\sum_{i:C_i \cap S_X  = \emptyset} D_{X_{\rm I}}(C_i) \Big )  .
      \end{split}
\end{equation*}
Note that $C_1,...,C_r$ are disjoint.  Therefore, $\sum_{i:C_i \cap S_X  \neq \emptyset} D_{X_{\rm I}}(C_i)\leq D_{X_{\rm I}}( \sum_{i:C_i \cap S_X  \neq \emptyset} C_i) \leq 1$. Using Lemma \ref{L14}, we obtain
\begin{equation*}
\begin{split}
      &\mathbb{E}_{\mathbf{x}\sim  D_{X_{\rm I}},S\sim D^n_{X_{\rm I}Y_{\rm I}}}{\rm dist}( \mathbf{x},\pi_1(\mathbf{x},S)) 
   \leq \sqrt{d}\epsilon+\frac{r\sqrt{d}}{ne}=\sqrt{d}\epsilon+\frac{\sqrt{d}}{ne{\epsilon}^d}.
      \end{split}
\end{equation*}
If we set $\epsilon= 2n^{-1/(d+1)}$, then
\begin{equation*}
\begin{split}
      &\mathbb{E}_{\mathbf{x}\sim  D_{X_{\rm I}},S\sim D^n_{X_{\rm I}Y_{\rm I}}}{\rm dist}( \mathbf{x},\pi_1(\mathbf{x},S)) 
   \leq \frac{2\sqrt{d}}{n^{1/(d+1)}}+\frac{\sqrt{d}}{2^den^{1/(d+1)}}.
      \end{split}
\end{equation*}
If we set $\epsilon_{\rm cons}(n)=\frac{2\sqrt{d}}{n^{1/(d+1)}}+\frac{\sqrt{d}}{2^den^{1/(d+1)}}$, we complete this proof.
\end{proof}
\thmPosbSep*

\begin{proof}[Proof of Theorem \ref{T13}]
\textbf{First}, we prove that  if the hypothesis space $\mathcal{H}$ is a separate space for OOD (\textit{i.e.}, Assumption \ref{ass1} holds), the constant function $h^{\rm in}:=1 \in \mathcal{H}$, then that OOD detection is learnable in  $\mathscr{D}_{XY}^s$ for $\mathcal{H}$ implies  $\mathcal{H}_{\rm all}-\{h^{\rm out}\} \subset \mathcal{H}$.

\noindent \textbf{Proof by Contradiction}: suppose that there exists $h'\in \mathcal{H}_{\rm all}$ such that $h'\neq h^{\rm out}$ and $h'\notin \mathcal{H}$.

\noindent Let $\mathcal{X}=\{\mathbf{x}_1,...,\mathbf{x}_m\}$, $C_{\rm I}=\{\mathbf{x}\in \mathcal{X}:h'(\mathbf{x})\in \mathcal{Y}\}$ and $C_{\rm O}=\{\mathbf{x}\in \mathcal{X}:h'(\mathbf{x})=K+1\}$.

\noindent Because $h'\neq h^{\rm out}$, we know that $C_{\rm I}\neq \emptyset$.

\noindent We construct a special domain $D_{XY}\in \mathscr{D}_{XY}^s$: if $C_{\rm O}=\emptyset$, then $ D_{XY} = D_{X_{\rm I}}*D_{Y_{\rm I}|X_{\rm I}}$; otherwise,
\begin{equation*}
    D_{XY} = 0.5*D_{X_{\rm I}}*D_{Y_{\rm I}|X_{\rm I}}+0.5*D_{X_{\rm O}}*D_{Y_{\rm O}|X_{\rm O}},~~\textnormal{where}
\end{equation*}
\begin{equation*}
    D_{X_{\rm I}} = \frac{1}{|C_{\rm I}|} \sum_{\mathbf{x}\in C_{\rm I}}\delta_{\mathbf{x}}~~\textnormal{and}~~D_{Y_{\rm I}|X_{\rm I}}(y|\mathbf{x})=1,~~\textnormal{if}~h'(\mathbf{x})=y~\textnormal{and}~\mathbf{x}\in C_{\rm I},
\end{equation*}
and
\begin{equation*}
    D_{X_{\rm O}} = \frac{1}{|C_{\rm O}|} \sum_{\mathbf{x}\in C_{\rm O}}\delta_{\mathbf{x}}~~\textnormal{and}~~D_{Y_{\rm O}|X_{\rm O}}(K+1|\mathbf{x})=1,~~\textnormal{if}~\mathbf{x}\in C_{\rm O}.
\end{equation*}
\noindent Since $h'\notin \mathcal{H}$ and $|\mathcal{X}|<+\infty$, then $\argmin_{h\in \mathcal{H}} R_D(h)\neq \emptyset$, and $\inf_{h\in \mathcal{H}} R_D(h)>0$. Additionally, $R_D^{\rm in}(h^{\rm in})=0$ (here $h^{\rm in}=1$), hence, $\inf_{h\in \mathcal{H}} R_D^{\rm in}(h)=0$.

\noindent Since OOD detection  is learnable in  $\mathscr{D}_{XY}^s$ for $\mathcal{H}$, Lemma \ref{C1andC2} implies that
\begin{equation*}
    \inf_{h\in \mathcal{H}} R_D(h) = (1-\pi^{\rm out}) \inf_{h\in \mathcal{H}} R_D^{\rm in}(h)+ \pi^{\rm out}\inf_{h\in \mathcal{H}} R_D^{\rm out}(h),
    \end{equation*}
    where $\pi^{\rm out}=D_{Y}(Y=K+1)=1$ or $0.5$.
Since $\inf_{h\in \mathcal{H}} R_D^{\rm in}(h)=0$ and $\inf_{h\in \mathcal{H}} R_D(h)>0$, we obtain that $\inf_{h\in \mathcal{H}} R_D^{\rm out}(h)>0$.

\noindent Until now, we have constructed a special domain $D_{XY}\in \mathscr{D}_{XY}^s$ satisfying that $\inf_{h\in \mathcal{H}} R_D^{\rm out}(h)>0$. Using Lemma \ref{T9}, we know that OOD detection in  $\mathscr{D}_{XY}^s$  is not learnable  for $\mathcal{H}$, which is \textbf{inconsistent} with the condition that OOD detection is learnable in  $\mathscr{D}_{XY}^s$ for $\mathcal{H}$. Therefore, the assumption (there exists $h'\in \mathcal{H}_{\rm all}$ such that $h'\neq h^{\rm out}$ and $h\notin \mathcal{H}$) doesn't hold, which implies that $\mathcal{H}_{\rm all}-\{h^{\rm out}\} \subset \mathcal{H}$.
\\
\\
\noindent \textbf{Second}, we prove that if $\mathcal{H}_{\rm all}-\{h^{\rm out}\} \subset \mathcal{H}$, then OOD detection is learnable in  $\mathscr{D}_{XY}^s$ for $\mathcal{H}$.
\\
\\
\noindent To prove this result, we need to design a special  algorithm. Let $d_0=\min_{\mathbf{x},\mathbf{x}'\in \mathcal{X} ~{\rm and}~ {\mathbf{x}\neq\mathbf{x}'}}{\rm dist}(\mathbf{x},\mathbf{x}')$, where ${\rm dist}$ is the Euclidean distance. It is clear that $d_0>0$. Let 
\\
$~~~~~~~~~~~~~~~~~~~~~~~~~~~~~~${\centering{
$ \mathbf{A}(S)(\mathbf{x})=\left \{
\begin{aligned}
~~~~~1,&~~~~{\rm if}~~ {\rm dist}( \mathbf{x},\pi_1(\mathbf{x},S))<0.5*d_0;\\
~~~~~2,&~~~~{\rm if}~~ {\rm dist}( \mathbf{x},\pi_1(\mathbf{x},S))\geq 0.5*d_0,\\
\end{aligned}
 \right.
$}}
\\
{ where}~$\pi_1(\mathbf{x},S)= \argmin_{\tilde{\mathbf{x}}\in S_{X}} {\rm dist}(\mathbf{x},\tilde{\mathbf{x}})$,
 here $S_X$ is the feature part of $S$, \textit{i.e.}, $S_X=\{\mathbf{x}^1,...,\mathbf{x}^n\}$, if $S=\{(\mathbf{x}^1,y^1),...,(\mathbf{x}^n,y^n)\}$.

\noindent For any $\mathbf{x}\in {\rm supp} D_{X_{\rm I}}$, it is easy to check that for almost all $S\sim D^n_{X_{\rm I}Y_{\rm I}}$,
\begin{equation*}
    {\rm dist}( \mathbf{x},\pi_1(\mathbf{x},S))>0.5*d_0,
\end{equation*}
which implies that
\begin{equation*}
     \mathbf{A}({S})(\mathbf{x}) =2,
\end{equation*}
hence,
\begin{equation}\label{E1}
    \mathbb{E}_{S\sim D^n_{X_{\rm I}Y_{\rm I}}} R_{D}^{\rm out}(\mathbf{A}(S))=0.
\end{equation}

\noindent Using Lemma \ref{L15}, for any $\mathbf{x}\in {\rm supp} D_{X_{\rm I}}$, we have
\begin{equation*}
    \mathbb{E}_{\mathbf{x}\sim  D_{X_{\rm I}},S\sim D^n_{X_{\rm I}Y_{\rm I}}}{\rm dist}( \mathbf{x},\pi_1(\mathbf{x},S))<\epsilon_{\rm cons}(n),
\end{equation*}
\noindent where $\epsilon_{\rm cons}(n) \rightarrow 0$, as $n\rightarrow 0$ and $\epsilon_{\rm cons}(n)$ is a monotonically decreasing
sequence.
\\
\\
\noindent Hence, we have that 
\begin{equation*}
   D_{X_{\rm I}}\times  D_{X_{\rm I}Y_{\rm I}}^n(\{(\mathbf{x},S):{\rm dist}( \mathbf{x},\pi_1(\mathbf{x},S))\geq 0.5*d_0 \})\leq 2\epsilon_{\rm cons}(n)/d_0,
\end{equation*}
where $ D_{X_{\rm I}}\times  D_{X_{\rm I}Y_{\rm I}}^n$ is the product measure of $D_{X_{\rm I}}$ and $D_{X_{\rm I}Y_{\rm I}}^n $ \citep{cohn2013measure}.
Therefore,
\begin{equation*}
    D_{X_{\rm I}}\times  D_{X_{\rm I}Y_{\rm I}}^n(\{(\mathbf{x},S):\mathbf{A}(S)(\mathbf{x})=1\})> 1- 2\epsilon_{\rm cons}(n)/d_0,
\end{equation*}
which implies that
\begin{equation}\label{E2}
  \mathbb{E}_{S\sim D^n_{X_{\rm I}Y_{\rm I}}} R_{D}^{\rm in} (\mathbf{A}(S)) \leq 2B\epsilon_{\rm cons}(n)/d_0,
\end{equation}
where $B=\max\{\ell(1,2),\ell(2,1)\}$. Using Eq. \eqref{E1} and Eq. \eqref{E2}, we have proved that 
\begin{equation}
    \mathbb{E}_{S\sim D_{X_{\rm I}Y_{\rm I}}^n} R_D(\mathbf{A}(S))\leq 0+ 2B\epsilon_{\rm cons}(m)/d_0 \leq \inf_{h\in \mathcal{H}}R_D(h)+2B\epsilon_{\rm cons}(m)/d_0.
\end{equation}
It is easy to check that $\mathbf{A}(S)\in \mathcal{H}_{\rm all}-\{h^{\rm out}\}$. Therefore, we have constructed a consistent  algorithm $\mathbf{A}$ for $\mathcal{H}$. We have completed this proof.
\end{proof}


\subsection{Proof of Theorem \ref{T17}}

\thmPosbMultione*
\begin{proof}[Proof of Theorem \ref{T17}]
Since $|\mathcal{X}|<+\infty$, we know that $|\mathcal{H}|<+\infty$, which implies that $\mathcal{H}^{\rm in}$ is agnostic PAC learnable for supervised learning in classification. Therefore, there exist an algorithm $\mathbf{A}^{\rm in}: \cup_{n=1}^{+\infty}(\mathcal{X}\times\mathcal{Y})^n\rightarrow \mathcal{H}^{\rm in}$ and a monotonically decreasing
sequence $\epsilon(n)$, such that $\epsilon(n)\rightarrow 0$, as $n\rightarrow +\infty$, and for any $D_{XY}\in \mathscr{D}_{XY}^s$,
\begin{equation*}
    \mathbb{E}_{S\sim D^n_{X_{\rm I}Y_{\rm I}}} R_D^{\rm in}(\mathbf{A}^{\rm in}(S)) \leq \inf_{h\in \mathcal{H}^{\rm in}} R_{D}^{\rm in}(h)+\epsilon(n).
\end{equation*}

\noindent Since $|\mathcal{X}|<+\infty$ and  $\mathcal{H}^{\rm b}$ almost contains all binary classifiers, then using Theorem \ref{T13} and Theorem \ref{T1}, we obtain that there exist an algorithm $\mathbf{A}^{\rm b}: \cup_{n=1}^{+\infty}(\mathcal{X}\times \{1,2\})^n\rightarrow \mathcal{H}^{\rm b}$ and a monotonically decreasing sequence $\epsilon'(n)$, such that
$\epsilon'(n)\rightarrow 0$, as $n\rightarrow +\infty$, and for any $D_{XY}\in \mathscr{D}_{XY}^s$,
\begin{equation*}
\begin{split}
    & \mathbb{E}_{S\sim D^n_{X_{\rm I}Y_{\rm I}}} R_{\phi({D})}^{\rm in}(\mathbf{A}^{\rm b}(\phi(S))) \leq \inf_{h\in \mathcal{H}^{\rm b}} R_{\phi(D)}^{\rm in}(h)+\epsilon'(n),
    \\ &
    \mathbb{E}_{S\sim D^n_{X_{\rm I}Y_{\rm I}}} R_{\phi({D})}^{\rm out}(\mathbf{A}^{\rm b}(\phi(S))) \leq \inf_{h\in \mathcal{H}^{\rm b}} R_{\phi(D)}^{\rm out}(h)+\epsilon'(n),
    \end{split}
\end{equation*}
where $\phi$ maps ID's labels to $1$ and OOD's label to $2$,
\begin{equation}
    R_{\phi({D})}^{\rm in}(\mathbf{A}^{\rm b}(\phi(S)))  = \int_{\mathcal{X}\times \mathcal{Y}} \ell(\mathbf{A}^{\rm b}(\phi(S))(\mathbf{x}), \phi(y)) {\rm d} D_{X_{\rm I}Y_{\rm I}}(\mathbf{x},y),
\end{equation}
\begin{equation}
    R_{\phi({D})}^{\rm in}(h)  = \int_{\mathcal{X}\times \mathcal{Y}} \ell(h(\mathbf{x}), \phi(y)) {\rm d} D_{X_{\rm I}Y_{\rm I}}(\mathbf{x},y),
\end{equation}
\begin{equation}
    R_{\phi({D})}^{\rm out}(\mathbf{A}^{\rm b}(\phi(S)))  = \int_{\mathcal{X}\times \{K+1\}} \ell(\mathbf{A}^{\rm b}(\phi(S))(\mathbf{x}), \phi(y)) {\rm d} D_{X_{\rm O}Y_{\rm O}}(\mathbf{x},y),
\end{equation}
and
\begin{equation}
    R_{\phi({D})}^{\rm out}(h)  = \int_{\mathcal{X}\times \{K+1\}} \ell(h(\mathbf{x}), \phi(y)) {\rm d} D_{X_{\rm O}Y_{\rm O}}(\mathbf{x},y),
\end{equation}
here $\phi(S)=\{(\mathbf{x}^1,\phi({y}^1)),...,(\mathbf{x}^n,\phi({y}^n))\}$, if $S=\{(\mathbf{x}^1,{y}^1),...,(\mathbf{x}^n,{y}^n)\}$.
\\

\noindent Note that $\mathcal{H}^b$ almost contains all classifiers, and $\mathscr{D}^s_{XY}$ is the separate space. Hence,  
\begin{equation*}
\begin{split}
    & \mathbb{E}_{S\sim D^n_{X_{\rm I}Y_{\rm I}}} R_{\phi({D})}^{\rm in}(\mathbf{A}^{\rm b}(\phi(S))) \leq \epsilon'(n),
    ~~~
    \mathbb{E}_{S\sim D^n_{X_{\rm I}Y_{\rm I}}} R_{\phi({D})}^{\rm out}(\mathbf{A}^{\rm b}(\phi(S))) \leq \epsilon'(n).
    \end{split}
\end{equation*}

\noindent \textbf{Next}, we construct an algorithm $\mathbf{A}$ using $\mathbf{A}^{\rm in}$ and $\mathbf{A}^{\rm out}$.
\\

~~~~~~~~~~~~~~~~~~~~~~~~{\centering{
$ \mathbf{A}(S)(\mathbf{x})=\left \{
\begin{aligned}
K+1,&~~~~{\rm if}~~ \mathbf{A}^{\rm b}(\phi(S))(\mathbf{x})=2;\\
\mathbf{A}^{\rm in}(S)(\mathbf{x}),&~~~~{\rm if}~~ \mathbf{A}^{\rm b}(\phi(S))(\mathbf{x})=1.\\
\end{aligned}
 \right.
$}}
\\

\noindent Since $\inf_{h\in \mathcal{H}} R_{\phi(D)}^{\rm in}(\phi\circ h)=0$, $\inf_{h\in \mathcal{H}} R_{D}^{\rm out}(h)=0$, then by Condition \ref{C3}, it is easy to check that
\begin{equation*}
\inf_{h\in \mathcal{H}^{\rm in}} R_{D}^{\rm in}(h)=\inf_{h\in \mathcal{H}} R_{D}^{\rm in}(h).
\end{equation*}

\noindent Additionally, the risk $R_D^{\rm in}(\mathbf{A}(S))$ is from two parts: 1) ID data are detected as OOD data; 2) ID data are detected as ID data, but are classified as incorrect ID classes. Therefore, we have the inequality: 
\begin{equation}\label{IDlearnable}
\begin{split}
    \mathbb{E}_{S\sim D^n_{X_{\rm I}Y_{\rm I}}} R_D^{\rm in}(\mathbf{A}(S))& \leq    \mathbb{E}_{S\sim D^n_{X_{\rm I}Y_{\rm I}}} R_D^{\rm in}(\mathbf{A}^{\rm in}(S)) +  c\mathbb{E}_{S\sim D^n_{X_{\rm I}Y_{\rm I}}} R_{\phi({D})}^{\rm in}(\mathbf{A}^{\rm b}(\phi(S))) \\ & \leq \inf_{h\in \mathcal{H}^{\rm in}} R_{D}^{\rm in}(h)+\epsilon(n) + c\epsilon'(n)  =  \inf_{h\in \mathcal{H}} R_{D}^{\rm in}(h)+\epsilon(n) + c\epsilon'(n), 
    \end{split}
\end{equation}
where $c= \max_{y_1,y_2\in \mathcal{Y}} \ell(y_1,y_2)/\min\{\ell(1,2),\ell(2,1)\}$.
\\

\noindent Note that the risk $R_D^{\rm out}(\mathbf{A}(S))$ is from the case that  OOD data are detected as ID data. Therefore,
\begin{equation}\label{OODlearnable}
\begin{split}
    \mathbb{E}_{S\sim D^n_{X_{\rm I}Y_{\rm I}}} R_D^{\rm out}(\mathbf{A}(S))& \leq      c\mathbb{E}_{S\sim D^n_{X_{|rm I}Y_{\rm I}}} R_{\phi({D})}^{\rm out}(\mathbf{A}^{\rm b}(\phi(S))) \\ & \leq   c\epsilon'(n)\leq \inf_{h\in \mathcal{H}} R_D^{\rm out}(h)+ c\epsilon'(n). 
    \end{split}
\end{equation}
Note that $ (1-\alpha)\inf_{h\in \mathcal{H}} R_{D}^{\rm in}(h)+\alpha \inf_{h\in \mathcal{H}} R_{D}^{\rm out}(h)\leq \inf_{h\in \mathcal{H}} R_D^{\alpha}(h)$. Then, using Eq. \eqref{IDlearnable} and Eq. \eqref{OODlearnable}, we obtain that for any $\alpha\in [0,1]$,
\begin{equation*}
\begin{split}
    \mathbb{E}_{S\sim D^n_{X_{\rm I}Y_{\rm I}}} R_D^{\alpha}(\mathbf{A}(S))& \leq     \inf_{h\in \mathcal{H}} R_D^{\alpha}(h)+\epsilon(n)+ c\epsilon'(n). 
    \end{split}
\end{equation*}
According to Theorem \ref{T1} (the second result), we complete the proof.
\end{proof}

\section{Proof of Theorem \ref{T13_auc}}\label{SI_auc}
\begin{lemma}\label{l14-auc}
   Suppose that $|\mathcal{X}|<+\infty$. If AUC-based Realizability Assumption holds for AUC metric, then  OOD detection is learnable  under AUC in separate space for $\mathcal{R}$.
\end{lemma}

\begin{proof}[Proof of Lemma \ref{l14-auc}]
    Without loss of generality, we assume that $K=1$, and any $r\in \mathcal{R}$ satisfies that $0<r<1$ (one can achieve this by using sigmoid function). Given $m$ data points $S_m = \{\mathbf{x}_1',...,\mathbf{x}_m'\}\subset \mathcal{X}^m$. We consider the following learning rule
\begin{equation*}
    \max_{r\in \mathcal{R}, \tau\in(0,1)}  \sum_{i=1}^m  \mathbf{1}_{r(\mathbf{x}_i')\leq \tau} \mathbf{1}_{\mathbf{x}_i'\notin S}, ~\text{subject~to}~\frac{1}{n} \sum_{j=1}^n \mathbf{1}_{r(\mathbf{x}_j)\leq \tau} =0.
\end{equation*}
 We denote the algorithm, which solves the above rule, as $\mathbf{A}_{S_m}$ ($\mathbf{A}_{S_m}$ outputs $r$ and a corresponding $\tau$). For different data points $S_m$, we have different algorithm $\mathbf{A}_{S_m}$. Let $\mathcal{S}$ be the set that consists of all data points, \textit{i.e.},
\begin{equation}
    \mathcal{S}:= \{S_m: S_m~\textnormal{are any}~m \textnormal{ data points},~m=1,...,+\infty\}.
\end{equation}

\noindent Using $\mathcal{S}$, we construct an algorithm space as follows:
\begin{equation*}
    \mathscr{A}:= \{ \mathbf{A}_{{S}'}: \forall~S'\in \mathcal{S}\}.
\end{equation*}
We will find an algorithm $\mathbf{A}$ from $ \mathscr{A}$, which is learnable. Let $S' = \mathcal{X}$. We will show that $\mathbf{A}_{\mathcal{X}}$ can guarantee the learnability. Suppose that $r_S$ and $\tau_S$ is the output of $\mathbf{A}_{\mathcal{X}}(S)$, then the realizability assumption implies that there exists learning rate $\epsilon(n)$ such that
\begin{equation}\label{5.76q0}
   \mathbb{E}_{S\sim D^n_{X_{\rm I}}} \mathbb{E}_{\mathbf{x}\sim D_{X_{\rm I}}} \mathbf{1}_{r_S(\mathbf{x})\leq \tau_S}
    \leq \epsilon(n).
\end{equation}
Additionally, due to ${\rm supp}(D_{X_{\rm O}})\subset \mathcal{X}-S$, the AUC-based Realizability Assumption implies that
\begin{equation}\label{5.76q1}
     \mathbb{E}_{S\sim D^n_{X_{\rm I}}} \mathbb{E}_{\mathbf{x}\sim D_{X_{\rm O}}} \mathbf{1}_{r_S(\mathbf{x})\leq \tau_S}
   =1.
\end{equation}
Then, \begin{equation*}
\begin{split}
& \mathbb{E}_{S\sim D_{X_{\rm I}}^n}   {\rm AUC}(f_{S};D_{X_{\rm I}},D_{X_{\rm O}}) \\   \geq &   \mathbb{E}_{S\sim D_{X_{\rm I}}^n} \mathbb{E}_{\mathbf{x}\sim D_{X_{\rm O}}} \mathbb{E}_{\mathbf{x}'\sim D_{X_{\rm I}}} \mathbf{1}_{r_{S}(\mathbf{x})< r_{S}(\mathbf{x}')} \\ \geq &   \mathbb{E}_{S\sim D_{X_{\rm I}}^n} \mathbb{E}_{\mathbf{x}\sim D_{X_{\rm O}}} \mathbb{E}_{\mathbf{x}'\sim D_{X_{\rm I}}} \mathbf{1}_{r_{S}(\mathbf{x})\leq \tau_{S}} \mathbf{1}_{r_{S}(\mathbf{x}')>  \tau_{S}}\\ \geq & 1- \epsilon(n).
     \end{split}
\end{equation*}
Then the ranking function part  $r_{S}$ of $\mathbf{A}_{\mathcal{X}}\in \mathscr{A}$ is the universally consistent algorithm, \textit{i.e.}, 
\begin{equation*}
\begin{split}
& \mathbb{E}_{S\sim D_{X_{\rm I}}^n}  {\rm AUC}(r_{S};D_{X_{\rm I}},D_{X_{\rm O}})  \geq 1- \epsilon(n).
     \end{split}
\end{equation*}
\end{proof}
\thmPosbSepauc*
\begin{proof}[Proof of Theorem \ref{T13_auc}]
    Lemmas \ref{l2-auc} and \ref{l14-auc} imply this result.
\end{proof}
\section{Proofs of Theorems \ref{T-SET} and \ref{T-SET2}}
\subsection{Proof of Theorem \ref{T-SET}}
\begin{lemma}\label{lemma7-new}
Given a prior-unknown space $\mathscr{D}_{XY}$ and a hypothesis space $\mathcal{H}$, if Condition \ref{Con2} holds, then for any equivalence class $[D_{XY}']$ with respect to $\mathscr{D}_{XY}$, OOD detection is learnable in the equivalence class $[D_{XY}']$ for $\mathcal{H}$. Furthermore, the learning rate can attain $O(1/n)$.
\end{lemma}
\begin{proof}
 Let $\mathscr{F}$ be a set consisting of all infinite sequences, whose coordinates are hypothesis functions, \textit{i.e.},
 \begin{equation*}
 \mathscr{F}=\{{\bm h}=(h_1,...,h_n,...): \forall h_n\in \mathcal{H}, n=1,....,+\infty\}.
 \end{equation*}
\\
\noindent For each ${\bm h}\in \mathscr{F}$, there is a corresponding algorithm $\mathbf{A}_{\bm h}$: $\mathbf{A}_{\bm h}(S)=h_n,~{\rm if}~|S|=n$. $\mathscr{F}$ generates an algorithm class $\mathscr{A}=\{\mathbf{A}_{\bm h}: \forall {\bm h}\in \mathscr{F}\}$. We select a consistent algorithm from the algorithm class $\mathscr{A}$. 
\\
\\
\noindent We construct a special infinite sequence $\tilde{{\bm h}}=(\tilde{h}_1,...,\tilde{h}_n,...)\in \mathscr{F}$. For each positive integer $n$, we select $\tilde{h}_n$ from
\begin{equation*}
\bigcap_{\forall D_{XY}\in [D_{XY}']}\{ h' \in \mathcal{H}: R_D^{\rm out}(h') \leq \inf_{h\in \mathcal{H}} R_D^{\rm out}(h)+2/n\}\bigcap  \{ h' \in \mathcal{H}: R_D^{\rm in}(h') \leq \inf_{h\in \mathcal{H}} R_D^{\rm in}(h)+2/n\}.
\end{equation*}
The existence of $\tilde{h}_n$ is based on Condition \ref{Con2}. It is easy to check that for any $D_{XY}\in [D_{XY}']$,
 \begin{equation*}
 \begin{split}
  & \mathbb{E}_{S\sim D^n_{X_{\rm I}Y_{\rm I}}}  R_D^{\rm in}(\mathbf{A}_{\tilde{{\bm h}}}(S))\leq  \inf_{h\in \mathcal{H}} R_D^{\rm in}(h)+2/n.
   \\
   & \mathbb{E}_{S\sim D^n_{X_{\rm I}Y_{\rm I}}}  R_D^{\rm out}(\mathbf{A}_{\tilde{{\bm h}}}(S))\leq  \inf_{h\in \mathcal{H}} R_D^{\rm out}(h)+2/n.
    \end{split}
 \end{equation*}
 Since $(1-\alpha)\inf_{h\in \mathcal{H}} R_D^{\rm in}(h)+\alpha \inf_{h\in \mathcal{H}} R_D^{\rm out}(h)\leq \inf_{h\in \mathcal{H}} R_D^{\rm \alpha}(h)$, we obtain that for any $\alpha\in [0,1]$, 
 \begin{equation*}
 \begin{split}
  & \mathbb{E}_{S\sim D^n_{X_{\rm I}Y_{\rm I}}}  R_D^{\alpha}(\mathbf{A}_{\tilde{{\bm h}}}(S))\leq  \inf_{h\in \mathcal{H}} R_D^{\alpha}(h)+2/n.
    \end{split}
 \end{equation*}
Using Theorem \ref{T1} (the second result), we have completed this proof.
\end{proof}
\thmPosbtennew*
\begin{proof}[Proof of Theorem \ref{T-SET}]
\noindent \textbf{First}, we prove that if OOD detection is learnable in $\mathscr{D}_{XY}^F$ for $\mathcal{H}$, then Condition \ref{Con2} holds.
\\
\\
\noindent Since $\mathscr{D}^{F}_{XY}$ is the prior-unknown space, by Theorem \ref{T1}, there exist an algorithm $\mathbf{A}: \cup_{n=1}^{+\infty}(\mathcal{X}\times\mathcal{Y})^n\rightarrow \mathcal{H}$ and a monotonically decreasing sequence $\epsilon_{\rm cons}(n)$, such that $\epsilon_{\rm cons}(n)\rightarrow 0$, as $n\rightarrow +\infty$, and for any $D_{XY}\in \mathscr{D}_{XY}^F$, 
\begin{equation*}
\begin{split}
     &\mathbb{E}_{S\sim D^n_{X_{\rm I}Y_{\rm I}}}\big[ R^{\rm in}_D(\mathbf{A}(S))- \inf_{h\in \mathcal{H}}R^{\rm in}_D(h)\big]\leq \epsilon_{\rm cons}(n),
     \\ &\mathbb{E}_{S\sim D^n_{X_{\rm I}Y_{\rm I}}}\big[ R^{\rm out}_D(\mathbf{A}(S))- \inf_{h\in \mathcal{H}}R^{\rm out}_D(h)\big]\leq \epsilon_{\rm cons}(n).
     \end{split}
\end{equation*}
Then, for any $\epsilon>0$, we can find $n_{\epsilon}$ such that $\epsilon\geq \epsilon_{\rm cons}(n_{\epsilon})$, therefore, if $n={n_{\epsilon}}$, we have
\begin{equation*}
\begin{split}
     &\mathbb{E}_{S\sim D^{n_{\epsilon}}_{X_{\rm I}Y_{\rm I}}}\big[ R^{\rm in}_D(\mathbf{A}(S))- \inf_{h\in \mathcal{H}}R^{\rm in}_D(h)\big]\leq \epsilon,
     \\ &\mathbb{E}_{S\sim D^{n_{\epsilon}}_{X_{\rm I}Y_{\rm I}}}\big[ R^{\rm out}_D(\mathbf{A}(S))- \inf_{h\in \mathcal{H}}R^{\rm out}_D(h)\big]\leq \epsilon, 
     \end{split}
\end{equation*}
which implies that there exists $S_{\epsilon}\sim D^{n_{\epsilon}}_{X_{\rm I}Y_{\rm I}}$ such that
\begin{equation*}
\begin{split}
     & R^{\rm in}_D(\mathbf{A}(S_{\epsilon}))- \inf_{h\in \mathcal{H}}R^{\rm in}_D(h)\leq \epsilon,
     \\ & R^{\rm out}_D(\mathbf{A}(S_{\epsilon}))- \inf_{h\in \mathcal{H}}R^{\rm out}_D(h)\leq \epsilon. 
     \end{split}
\end{equation*}
Therefore, for any equivalence class $[D_{XY}']$ with respect to $\mathscr{D}_{XY}^F$ and any $\epsilon>0$, there exists a hypothesis function $\mathbf{A}(S_{\epsilon})\in \mathcal{H}$ such that for any domain $D_{XY}\in [D_{XY}']$,
 \begin{equation*}
 \mathbf{A}(S_{\epsilon})\in \{ h' \in \mathcal{H}: R_D^{\rm out}(h') \leq \inf_{h\in \mathcal{H}} R_D^{\rm out}(h)+\epsilon\} \cap   \{ h' \in \mathcal{H}: R_D^{\rm in}(h') \leq \inf_{h\in \mathcal{H}} R_D^{\rm in}(h)+\epsilon\},
 \end{equation*}
 which implies that Condition \ref{Con2} holds.
 \\
 \\
\noindent  \textbf{Second}, we prove Condition \ref{Con2} implies the learnability of OOD detection in $\mathscr{D}_{XY}^F$ for $\mathcal{H}.$
For convenience, we assume that all equivalence classes are $[D^1_{XY}],...,[D^m_{XY}]$. By Lemma \ref{lemma7-new},  for every  equivalence class $[D_{XY}^i]$, we can find a corresponding algorithm $\mathbf{A}_{D^i}$ such that OOD detection is learnable in $[D_{XY}^i]$ for $\mathcal{H}$. Additionally, we also set the learning rate for $
\mathbf{A}_{D^i}$ is $\epsilon^i(n)$. By Lemma \ref{lemma7-new}, we know that $\epsilon^i(n)$ can attain $O(1/n)$.
\\
\\
\noindent Let ${\mathcal{Z}}$ be $\mathcal{X}\times\mathcal{Y}$. Then, we consider a bounded universal kernel $K(\cdot,\cdot)$ defined over $\mathcal{Z}\times \mathcal{Z}$. Consider the \emph{maximum mean discrepancy} (MMD) \citep{DBLP:journals/jmlr/GrettonBRSS12}, which is a metric between distributions: for any distributions $P$ and $Q$ defined over ${\mathcal{Z}}$, we use ${\rm MMD}_K(Q,P)$ to represent the distance.

\noindent Let $\mathscr{F}$ be a set consisting of all finite sequences, whose coordinates are labeled data, \textit{i.e.},
 \begin{equation*}
 \mathscr{F}=\{\mathbf{S}=(S_1,...,S_i,...,S_m): \forall i=1,...,m~\textnormal{and}~\forall~\textnormal{labeled~data}~ S_i\}.
 \end{equation*}
 
\noindent  Then, we define an algorithm space as follows:
 \begin{equation*}
     \mathscr{A}=\{\mathbf{A}_{\mathbf{S}}\footnote{In this paper, we regard an algorithm as a mapping from $\cup_{n=1}^{+\infty}(\mathcal{X}\times\mathcal{Y})^n$ to $\mathcal{H}$ or $\mathcal{R}$. So we can design an algorithm like this.}:\forall~\mathbf{S}\in  \mathscr{F}\},
 \end{equation*}
 where
 \begin{equation*}
      \mathbf{A}_{\mathbf{S}}(S) = \mathbf{A}_{D^i}(S),~\textnormal{if}~i=\argmin_{i\in\{1,...m\}}  {\rm MMD}_K(P_{S_i},P_S),
 \end{equation*}
  here 
 \begin{equation*}
     P_S = \frac{1}{n}\sum_{(\mathbf{x},y)\in S} \delta_{(\mathbf{x},y)},~~~P_{S_i} = \frac{1}{n}\sum_{(\mathbf{x},y)\in S_i}, \delta_{(\mathbf{x},y)}
 \end{equation*}
 and $\delta_{(\mathbf{x},y)}$ is the Dirac measure.
 Next, we prove that we can find an algorithm $\mathbf{A}$ from the algorithm space $\mathscr{A}$  such that $\mathbf{A}$ is the consistent algorithm.
 
\noindent   Since the number of different equivalence classes is finite, we know that there exists a constant $c>0$ such that for any different equivalence classes $[D_{XY}^i]$ and $[D_{XY}^j]$ ($i\neq j$),
 \begin{equation*}
     {\rm MMD}_K(D^i_{X_{\rm I}Y_{\rm I}},D^j_{X_{\rm I}Y_{\rm I}}) >c.
 \end{equation*}
 
\noindent   Additionally,  according to \citep{DBLP:journals/jmlr/GrettonBRSS12} and the property of $\mathscr{D}_{XY}^F$ (the number of different equivalence classes is finite), there exists a monotonically decreasing $\epsilon(n)\rightarrow 0$, as $n\rightarrow +\infty$ such that for any $D_{XY}\in \mathscr{D}$,
 \begin{equation}\label{MMDcon}
     \mathbb{E}_{S\sim D^n_{X_{\rm I}Y_{\rm I}}} {\rm MMD}_K(D_{X_{\rm I}Y_{\rm I}},P_S) \leq \epsilon(n),~\text{where}~\epsilon(n)=O(\frac{1}{\sqrt{n^{1-\theta}}}).
 \end{equation}
 Therefore, for every  equivalence class $[D_{XY}^i]$, we can find data points $S_{D^i}$ such that
 \begin{equation*}
      {\rm MMD}_K(D^i_{X_{\rm I}Y_{\rm I}},P_{S_{D^i}}) < \frac{c}{100}.
 \end{equation*}
 \\
 Let $\mathbf{S}'=\{S_{D^1},...,S_{D^i},...,S_{D^m}\}$. Then, we prove that $\mathbf{A}_{\mathbf{S}'}$ is a consistent algorithm.
 By Eq. \eqref{MMDcon}, it is easy to check that for any $i\in\{1,...,m\}$ and any $0<\delta<1$,
 \begin{equation*}
     \mathbb{P}_{S\sim D^{i,n}_{X_{\rm I}Y_{\rm I}}}\big [{\rm MMD}_K(D^i_{X_{\rm I}Y_{\rm I}},P_{S})\leq \frac{\epsilon(n)}{\delta} \big ] >1- \delta,
 \end{equation*}
 which implies that
 \begin{equation*}
     \mathbb{P}_{S\sim D^{i,n}_{X_{\rm I}Y_{\rm I}}}\big [{\rm MMD}_K(P_{S_{D^i}},P_{S})\leq \frac{\epsilon(n)}{\delta}+\frac{c}{100} \big ] >1- \delta.
 \end{equation*}
 Therefore, (here we set $\delta=200\epsilon(n)/c$)
 \begin{equation*}
      \mathbb{P}_{S\sim D^{i,n}_{X_{\rm I}Y_{\rm I}}}\big [ \mathbf{A}_{\mathbf{S}'}(S) \neq \mathbf{A}_{D^i}(S) \big ] \leq \frac{200\epsilon(n)}{c}.
 \end{equation*}
 Because $\mathbf{A}_{D^i}$  is a consistent algorithm for $[D_{XY}^i]$, we conclude that for all $\alpha\in [0,1]$,
 \begin{equation*}
\begin{split}
     \mathbb{E}_{S\sim D^{i,n}_{X_{\rm I}Y_{\rm I}}}\big[ R^{\alpha}_D( \mathbf{A}_{\mathbf{S}'}(S) )- \inf_{h\in \mathcal{H}}R^{\alpha}_D(h)\big]\leq \epsilon^i(n)+\frac{200B\epsilon(n)}{c},
     \end{split}
 \end{equation*}
where $\epsilon^i(n)=O(1/n)$ is the learning rate of $
\mathbf{A}_{D^i}$ and $B$ is the upper bound of the loss $\ell$.

\noindent  Let $\epsilon^{\rm max}(n)=\max \{\epsilon^1(n),...,\epsilon^m(n)\}+\frac{200B\epsilon(n)}{c}$. 
 
\noindent  Then, we obtain that 
  for any $D_{XY}\in \mathscr{D}_{XY}^F$ and all $\alpha\in [0,1]$,
 \begin{equation*}
\begin{split}
     \mathbb{E}_{S\sim D^n_{X_{\rm I}Y_{\rm I}}}\big[ R^{\alpha}_D( \mathbf{A}_{\mathbf{S}'}(S) )- \inf_{h\in \mathcal{H}}R^{\alpha}_D(h)\big]\leq \epsilon^{\rm max}(n)=O(\frac{1}{\sqrt{n^{1-\theta}}}).
     \end{split}
 \end{equation*}
According to Theorem \ref{T1} (the second result), $\mathbf{A}_{\mathbf{S}'}$ is the consistent algorithm. This proof is completed.
\end{proof}
\subsection{Proof of Theorem \ref{T-SET2}}
\thmPosbtennewtwo*
\begin{proof}[Proof of Theorem \ref{T-SET2}]
\textbf{First}, we consider the case that the loss $\ell$ is the zero-one loss.

\noindent Since $\mu(\mathcal{X})<+\infty$, without loss of generality, we assume that $\mu(\mathcal{X})=1$.  We also assume that $f_{\rm I}$ is $D_{X_{\rm I}}$'s density function and $f_{\rm O}$ is $D_{X_{\rm O}}$'s density function. Let $f$ be the density function for $0.5*D_{X_{\rm I}} + 0.5*D_{X_{\rm O}}$. It is easy to check that $f=0.5*f_{\rm I}+0.5*f_{\rm O}$.
Additionally, due to Risk-based Realizability Assumption, it is obvious that for any samples $S=\{(\mathbf{x}_1,y_1),...,(\mathbf{x}_n,y_n)\}\sim D^n_{X_{\rm I}Y_{\rm I}}$, i.i.d., we have that there exists $h^*\in \mathcal{H}$ such that
\begin{equation*}
    \frac{1}{n} \sum_{i=1}^n \ell(h^*(\mathbf{x}_i),y_i) = 0.
\end{equation*}

\noindent  Given $m$ data points $S_m = \{\mathbf{x}_1',...,\mathbf{x}_m'\}\subset \mathcal{X}^m$. We consider the following learning rule:
\begin{equation*}
    \min_{h\in \mathcal{H}} \frac{1}{m} \sum_{j=1}^m \ell(h(\mathbf{x}_j'),K+1),~
    \textnormal{ subject~to~} \frac{1}{n} \sum_{i=1}^n \ell(h(\mathbf{x}_i),y_i) = 0.
\end{equation*}

\noindent We denote the algorithm, which solves the above rule, as $\mathbf{A}_{S_m}\footnote{In this paper, we regard an algorithm as a mapping from $\cup_{n=1}^{+\infty}(\mathcal{X}\times\mathcal{Y})^n$ to $\mathcal{H}$ or $\mathcal{R}$. So we can design an algorithm like this.}$. For different data points $S_m$, we have different algorithm $\mathbf{A}_{S_m}$. Let $\mathcal{S}$ be the infinite sequence set that consists of all infinite sequences, whose coordinates are data points, \textit{i.e.},
\begin{equation}
    \mathcal{S}:= \{\mathbf{S}:=(S_1,S_2,...,S_m,...): S_m ~\textnormal{are~any~}m~\textnormal{ data~points},~m=1,...,+\infty\}.
\end{equation}

\noindent Using $\mathcal{S}$, we construct an algorithm space as follows:
\begin{equation*}
    \mathscr{A}:= \{ \mathbf{A}_{\mathbf{S}}: \forall~\mathbf{S}\in \mathcal{S}\},~\textnormal{where}~ \mathbf{A}_{\mathbf{S}}(S)=\mathbf{A}_{S_n}(S), ~{\rm if}~ |S|=n.
\end{equation*}

\noindent Next, we prove that there exists an algorithm $\mathbf{A}_{\mathbf{S}}\in \mathscr{A}$, which is a
consistent algorithm. \noindent Given data points $S_n\sim \mu^n$, i.i.d., using the  Natarajan
dimension theory and Empirical risk minimization principle \citep{shalev2014understanding}, it is easy to obtain that there exists a uniform constant $C_{\theta}$ such that (we mainly use the uniform bounds to obtain the following bounds)
\begin{equation*}
    \mathbb{E}_{S\sim D^n_{X_{\rm I}Y_{\rm I}}} \sup_{h\in \mathcal{H}_S}R_{D}^{\rm in}(h) \leq \inf_{h\in \mathcal{H}} R_D^{\rm in}(h)+\frac{C_{\theta}}{\sqrt{n^{1-\theta}}},
\end{equation*}
and because of $\mathcal{H}_S \subset \mathcal{H}$,
\begin{equation}\label{5.6q1}
    \mathbb{E}_{S_n \sim \mu^n} \sup_{S\in (\mathcal{X}\times \mathcal{Y})^n}[R_{\mu}(\mathbf{A}_{S_n}(S),K+1) - \inf_{h\in \mathcal{H}_S} R_{\mu}(h,K+1)] \leq \frac{C_{\theta}}{\sqrt{n^{1-\theta}}},
\end{equation}
where 
\begin{equation*}
    \mathcal{H}_S = \{h\in \mathcal{H}: \sum_{i=1}^n \ell(h(\mathbf{x}_i),y_i) = 0\},~\textnormal{here}~S=\{(\mathbf{x}_1,y_1),...,(\mathbf{x}_n,y_n)\},
    \end{equation*}
    and 
\begin{equation*}
   R_{\mu}(h,K+1) = \mathbb{E}_{\mathbf{x}\sim \mu} \ell(h(\mathbf{x}),K+1)=\int_{\mathcal{X}} \ell(h(\mathbf{x}),K+1) {\rm d}\mu(\mathbf{x}).
\end{equation*}

\noindent We set $\mathscr{D}_{\rm I} = \{D_{X_{\rm I}Y_{\rm I}}:\text{there exists } D_{X_{\rm O}Y_{\rm O}} \text{such that}~ (1-\alpha)D_{X_{\rm I}Y_{\rm I}}+\alpha D_{X_{\rm O}Y_{\rm O}} \in \mathscr{D}^{\mu,b}_{XY}\}$. Then by Eq. \eqref{5.6q1}, we have
\begin{equation}\label{5.6q11}
     \mathbb{E}_{S_n \sim \mu^n} \sup_{D_{X_{\rm I}Y_{\rm I}}\in \mathscr{D}_{\rm I}} \mathbb{E}_{S\sim D^n_{X_{\rm I}Y_{\rm I}}}[R_{\mu}(\mathbf{A}_{S_n}(S),K+1) - \inf_{h\in \mathcal{H}_S} R_{\mu}(h,K+1)] \leq \frac{C_{\theta}}{\sqrt{n^{1-\theta}}}.
\end{equation}

\noindent Due to Risk-based Realizability Assumption, we obtain that $\inf_{h\in \mathcal{H}} R_D^{\rm in}(h)=0$. Therefore,
\begin{equation}\label{5.6q3}
    \mathbb{E}_{S\sim D^n_{X_{\rm I}Y_{\rm I}}} \sup_{h\in \mathcal{H}_S}R_{D}^{\rm in}(h) \leq \frac{C_{\theta}}{\sqrt{n^{1-\theta}}},
\end{equation}
which implies that (in following inequalities, $g$ is the groundtruth labeling function, \textit{i.e.}, $R_D(g)=0$)
\begin{equation*}
\begin{split}
    \frac{C_{\theta}}{\sqrt{n}} \geq  \mathbb{E}_{S\sim D^n_{X_{\rm I}Y_{\rm I}}} \sup_{h\in \mathcal{H}_S}R_{D}^{\rm in}(h)=&  \mathbb{E}_{S\sim D^n_{X_{\rm I}Y_{\rm I}}} \sup_{h\in \mathcal{H}_S}\int_{g<K+1} \ell(h(\mathbf{x}),g(\mathbf{x}))f_{\rm I}(\mathbf{x}) {\rm d} \mu(\mathbf{x}) \\ \geq & \frac{2}{b}\mathbb{E}_{S\sim D^n_{X_{\rm I}Y_{\rm I}}} \sup_{h\in \mathcal{H}_S}\int_{g<K+1}\ell(h(\mathbf{x}),g(\mathbf{x})){\rm d} \mu(\mathbf{x}).
    \end{split}
\end{equation*}
 This implies that (here we have used the property of zero-one loss)
\begin{equation*}
    \mathbb{E}_{S\sim D^n_{X_{\rm I}Y_{\rm I}}} \inf_{h\in \mathcal{H}_S}\int_{g<K+1}\ell(h(\mathbf{x}),K+1){\rm d} \mu(\mathbf{x}) \geq   \mu({\mathbf{x}\in \mathcal{X}:g(\mathbf{x})<K+1})-\frac{C_{\theta}b}{2\sqrt{n^{1-\theta}}}.
\end{equation*}
Therefore,
\begin{equation}\label{zhen1}
     \mathbb{E}_{S\sim D^n_{X_{\rm I}Y_{\rm I}}} \inf_{h\in \mathcal{H}_S} R_{\mu}(h,K+1) \geq \mu({\mathbf{x}\in \mathcal{X}:g(\mathbf{x})<K+1})-\frac{C_{\theta}b}{2\sqrt{n^{1-\theta}}}.
\end{equation}

\noindent Additionally, $R_{\mu}(g,K+1)=\mu({\mathbf{x}\in \mathcal{X}:g(\mathbf{x})<K+1})$ and $g\in \mathcal{H}_S$, which implies that 
\begin{equation}\label{zhen2}
\inf_{h\in \mathcal{H}_S} R_{\mu}(h,K+1)\leq \mu({\mathbf{x}\in \mathcal{X}:g(\mathbf{x})<K+1}).
\end{equation}

\noindent Combining inequalities \eqref{zhen1} and \eqref{zhen2}, we obtain that
\begin{equation}\label{5.6q2}
   \big | \mathbb{E}_{S\sim D^n_{X_{\rm I}Y_{\rm I}}} \inf_{h\in \mathcal{H}_S} R_{\mu}(h,K+1) - \mu({\mathbf{x}\in \mathcal{X}:g(\mathbf{x})<K+1}) \big |\leq \frac{C_{\theta}b}{2\sqrt{n^{1-\theta}}}.
\end{equation}

\noindent Using inequalities \eqref{5.6q11} and \eqref{5.6q2}, we obtain that
\begin{equation}\label{5.6q5}
 \mathbb{E}_{S_n \sim \mu^n}   \sup_{D_{X_{\rm I}Y_{\rm I}}\in \mathscr{D}_{\rm I} } \big [ \mathbb{E}_{S\sim D^n_{X_{\rm I}Y_{\rm I}}} R_{\mu}(\mathbf{A}_{S_n}(S),K+1) -  \mu({\mathbf{x}\in \mathcal{X}:g(\mathbf{x})<K+1}) \big ] \leq \frac{C_{\theta}(b+1)}{\sqrt{n^{1-\theta}}}.
\end{equation}

\noindent Using inequality \eqref{5.6q3}, we have
\begin{equation}\label{5.6q4}
    \mathbb{E}_{S_n \sim \mu^n} \sup_{D_{X_{\rm I}Y_{\rm I}}\in \mathscr{D}_{\rm I} } \mathbb{E}_{S\sim D^n_{X_{\rm I}Y_{\rm I}}} R_{D}^{\rm in}(\mathbf{A}_{S_n}(S)) \leq \frac{C_{\theta}}{\sqrt{n^{1-\theta}}},
\end{equation}
which implies that (here we use the property of zero-one loss)
\begin{equation}\label{5.6q6}
\begin{split}
  \mathbb{E}_{S_n \sim \mu^n}\sup_{D_{X_{\rm I}Y_{\rm I}}\in \mathscr{D}_{\rm I}}\mathbb{E}_{S\sim D^n_{X_{\rm I}Y_{\rm I}}} \big[& -\int_{g<K+1} \ell(\mathbf{A}_{S_n}(S)(\mathbf{x}),K+1){\rm d}\mu(\mathbf{x})\\&+\mu({\mathbf{x}\in \mathcal{X}:g(\mathbf{x})<K+1})\big ] \leq \frac{2bC_{\theta}}{\sqrt{n^{1-\theta}}}.
  \end{split}
\end{equation}

\noindent Combining inequalities \eqref{5.6q5} and \eqref{5.6q6}, we have
\begin{equation*}
    \mathbb{E}_{S_n \sim \mu^n}\sup_{D_{X_{\rm I}Y_{\rm I}}\in \mathscr{D}_{\rm I}}\mathbb{E}_{S\sim D^n_{X_{\rm I}Y_{\rm I}}} \int_{g=K+1} \ell(\mathbf{A}_{S_n}(S)(\mathbf{x}),K+1){\rm d}\mu(\mathbf{x}) \leq \frac{2bC_{\theta}}{\sqrt{n^{1-\theta}}}+\frac{C_{\theta}(b+1)}{\sqrt{n^{1-\theta}}}.
\end{equation*}
Therefore, there exist data points $S_n'$ such that
\begin{equation}\label{5.6q7}
\begin{split}
&\sup_{D_{X_{\rm I}Y_{\rm I}}\in \mathscr{D}_{\rm I}} \mathbb{E}_{S\sim D^n_{X_{\rm I}Y_{\rm I}}} R_D^{\rm out}(\mathbf{A}_{S_n'})\\ =& \sup_{D_{X_{\rm I}Y_{\rm I}}\in \mathscr{D}_{\rm I}}\mathbb{E}_{S\sim D^n_{X_{\rm I}Y_{\rm I}}} \int_{g=K+1} \ell(\mathbf{A}_{S_n'}(S)(\mathbf{x}),K+1)f_{\rm O}(\mathbf{x}){\rm d}\mu(\mathbf{x}) \\ \leq  &   2b\sup_{D_{X_{\rm I}Y_{\rm I}}\in \mathscr{D}_{\rm I}}\mathbb{E}_{S\sim D^n_{X_{\rm I}Y_{\rm I}}} \int_{g=K+1} \ell(\mathbf{A}_{S_n'}(S)(\mathbf{x}),K+1){\rm d}\mu(\mathbf{x}) \leq \frac{4b^2C_{\theta}}{\sqrt{n^{1-\theta}}}+\frac{2C_{\theta}(b^2+b)}{\sqrt{n^{1-\theta}}}.
     \end{split}
\end{equation}

\noindent Combining inequalities \eqref{5.6q3} and \eqref{5.6q7}, we obtain that for any $n$, there exists data points $S_n'$ such that
\begin{equation*}
    \mathbb{E}_{S\sim D^n_{X_{\rm I}Y_{\rm I}}} R_D^{\alpha}(\mathbf{A}_{S_n'}) \leq \max \big \{\frac{4b^2C_{\theta}}{\sqrt{n^{1-\theta}}}+\frac{2C_{\theta}(b^2+b)}{\sqrt{n^{1-\theta}}},\frac{C_{\theta}}{\sqrt{n^{1-\theta}}} \big \}.
\end{equation*}

\noindent We set data point sequences $\mathbf{S}'=(S_1',S_2',...,S_n',...)$. Then, $\mathbf{A}_{\mathbf{S}'}\in \mathscr{A}$ is the universally consistent algorithm, \textit{i.e.}, for any $\alpha\in [0,1]$
\begin{equation*}
     \mathbb{E}_{S\sim D^n_{X_{\rm I}Y_{\rm I}}} R_D^{\alpha}(\mathbf{A}_{\mathbf{S}'}) \leq \max \big \{\frac{4b^2C_{\theta}}{\sqrt{n^{1-\theta}}}+\frac{2C_{\theta}(b^2+b)}{\sqrt{n^{1-\theta}}},\frac{C_{\theta}}{\sqrt{n^{1-\theta}}} \big\}.
\end{equation*}
We have completed this proof when $\ell$ is the zero-one loss.
\\

\noindent \textbf{Second}, we prove the case that $\ell$ is not the zero-one loss. We use the notation $\ell_{0-1}$ as the zero-one loss. According the definition of loss introduced in Section \ref{S3}, we know that there exists a constant $M>0$ such that for any $y_1,y_2\in \mathcal{Y}_{\rm all}$,
\begin{equation*}
    \frac{1}{M}\ell_{0-1}(y_1,y_2)\leq \ell(y_1,y_2) \leq M \ell_{0-1}(y_1,y_2).
\end{equation*}
Hence,
\begin{equation*}
     \frac{1}{M} R_D^{\alpha,\ell_{0-1}}(h) \leq R_D^{\alpha,\ell}(h) \leq M R_D^{\alpha,\ell_{0-1}}(h),
\end{equation*}
where
$R_D^{\alpha,\ell_{0-1}}$ is the $\alpha$-risk with zero-one loss, and $R_D^{\alpha,\ell}$ is the $\alpha$-risk for loss $\ell$.
\\
Above inequality tells us that Risk-based Realizability Assumption holds with zero-one loss if and only if Risk-based Realizability Assumption holds with the loss $\ell$. Therefore, we use the result proven in first step. We can find a consistent algorithm $\mathbf{A}$ such that for any $\alpha\in [0,1]$,
\begin{equation*}
     \mathbb{E}_{S\sim D^n_{X_{\rm I}Y_{\rm I}}} R_D^{\alpha, \ell_{0-1}}(\mathbf{A}) \leq O(\frac{1}{\sqrt{n^{1-\theta}}}),
\end{equation*}
which implies that for any $\alpha\in [0,1]$,
\begin{equation*}
     \frac{1}{M}\mathbb{E}_{S\sim D^n_{X_{\rm I}Y_{\rm I}}} R_D^{\alpha, \ell}(\mathbf{A}) \leq O(\frac{1}{\sqrt{n^{1-\theta}}}).
\end{equation*}
We have completed this proof.
\end{proof}

\section{Proof of Theorem \ref{T-SET2_auc}}
\thmPosbtennewtwoauc*
\begin{proof}[Proof of Theorem \ref{T-SET2_auc}]
    Without loss of generality, we assume that $K=1$, and any $r\in \mathcal{R}$ satisfies that $0<r<1$ (one can achieve this by using sigmoid function). Then it is clear that $\mathcal{R}_{(0,1)}=\{ \mathbf{1}_{r(\mathbf{x})\leq \tau}: \forall r\in \mathcal{R}, \forall \tau \in(0,1)\}$ has finite VC-dimension by the condition constant closure and ${\rm VC}[\phi\circ \mathcal{F}]<+\infty$. Given $m$ data points $S_m = \{\mathbf{x}_1',...,\mathbf{x}_m'\}\subset \mathcal{X}^m$. We consider the following learning rule:
\begin{equation*}
    \max_{r\in \mathcal{R},\tau \in (0,1)} \frac{1}{m} \sum_{i=1}^m  \mathbf{1}_{r(\mathbf{x}_i')\leq \tau}, ~\text{subject~to}~\frac{1}{n} \sum_{j=1}^n \mathbf{1}_{r(\mathbf{x}_j)\leq \tau} =0.
\end{equation*}
We denote the algorithm, which solves the above rule, as $\mathbf{A}_{S_m}$. For different data points $S_m$, we have different algorithm $\mathbf{A}_{S_m}$. Let $\mathcal{S}$ be the infinite sequence set that consists of all infinite sequences, whose coordinates are data points, \textit{i.e.},
\begin{equation}
    \mathcal{S}:= \{\mathbf{S}:=(S_1,S_2,...,S_m,...): S_m~\textnormal{are any}~m \textnormal{ data points},~m=1,...,+\infty\}.
\end{equation}

\noindent Using $\mathcal{S}$, we construct an algorithm space as follows:
\begin{equation*}
    \mathscr{A}:= \{ \mathbf{A}_{\mathbf{S}}: \forall~\mathbf{S}\in \mathcal{S}\},~\textnormal{where}~ \mathbf{A}_{\mathbf{S}}(S)=\mathbf{A}_{S_n}(S), ~\textnormal{if}~ |S|=n.
\end{equation*}
Then we can check that
\begin{equation*}
   \mathbb{E}_{S\sim D^n_{X_{\rm I}}} \sup_{\mathbf{1}_{r\leq \tau} \in \mathcal{G}_S} \mathbb{E}_{\mathbf{x}\sim D_{X_{\rm I}}} \mathbf{1}_{r(\mathbf{x})\leq \tau}
    \leq \inf_{\mathbf{1}_{r\leq \tau} \in \mathcal{R}_{(0,1)}} \mathbb{E}_{\mathbf{x}\sim D_{X_{\rm I}}} \mathbf{1}_{r(\mathbf{x})\leq \tau}+ \frac{C_{\theta}}{\sqrt{n^{1-\theta}}},
\end{equation*}
and because of $\mathcal{G}_S \subset \mathcal{R}_{(0,1)}$,

\begin{equation}\label{5.6q0_auc}
    \mathbb{E}_{S_n \sim \mu^n} \sup_{S\in \mathcal{X}^n}[\sup_{\mathbf{1}_{r \leq \tau}\in \mathcal{G}_S} R_{\mu}(\mathbf{1}_{r \leq \tau})-R_{\mu}(\mathbf{A}_{S_n}(S))  ] \leq \frac{C_{\theta}}{\sqrt{n^{1-\theta}}},
\end{equation}
where 
\begin{equation*}
    \mathcal{G}_{S} = \{\mathbf{1}_{r(\mathbf{x})\leq \tau} \in \mathcal{R}_{(0,1)}: \sum_{i=1}^n \mathbf{1}_{r(\mathbf{x}_j)\leq \tau}  = 0  < 0\},~\textnormal{here}~S=\{\mathbf{x}_1,...,\mathbf{x}_n\},
    \end{equation*}
    and 
\begin{equation*}
   R_{\mu}(\mathbf{1}_{r\leq \tau}) = \mathbb{E}_{\mathbf{x}\sim \mu} \mathbf{1}_{r(\mathbf{x})\leq \tau}.
\end{equation*}

\noindent Let $\mathcal{D}_{I}$ be the set consisting of all ID distribution in the density-based space. Then we have
\begin{equation}\label{5.6q1_auc}
    \mathbb{E}_{S_n \sim \mu^n} \sup_{D_{X_{\rm I}} \in \mathcal{D}_{I}}\mathbb{E}_{S\sim D_{X_{\rm I}}^n}[\sup_{\mathbf{1}_{r \leq \tau}\in \mathcal{G}_S} R_{\mu}(\mathbf{1}_{r \leq \tau})-R_{\mu}(\mathbf{A}_{S_n}(S))  ] \leq \frac{C_{\theta}}{\sqrt{n^{1-\theta}}},
\end{equation}

\noindent Due to AUC-based Realizability Assumption, we obtain that $\inf_{\mathbf{1}_{r\leq \tau} \in \mathcal{R}_{(0,1)}} \mathbb{E}_{\mathbf{x}\sim D_{X_{\rm I}}} \mathbf{1}_{r(\mathbf{x})\leq \tau}=0$, therefore,
\begin{equation}\label{5.6q3_auc}
      \mathbb{E}_{S\sim D^n_{X_{\rm I}}} \sup_{\mathbf{1}_{r\leq \tau} \in \mathcal{G}_S} \mathbb{E}_{\mathbf{x}\sim D_{X_{\rm I}}} \mathbf{1}_{r(\mathbf{x})\leq \tau}
    \leq \frac{C_{\theta}}{\sqrt{n^{1-\theta}}}.
\end{equation}
Let $r^* \in \mathcal{R}$ be the optimal ranking function satisfying that
\begin{equation*}
    {\rm AUC}(r^*; D_{X_{\rm I}},D_{X_{\rm O}})=1,
\end{equation*}
which implies that there exists $\tau^*$ satisfying that for any $\epsilon>0$
\begin{equation*}
   D_{X_{\rm I}} (\mathbf{x}:r^*(\mathbf{x}) < \tau^*-\epsilon)=0,~~~ D_{X_{\rm I}} (\mathbf{x}:r^*(\mathbf{x}) < \tau^*+\epsilon)>0.
\end{equation*}
Then we consider set $\Omega_{X_{\rm I}}: = \{\mathbf{x}\in \mathcal{X}: r^*(\mathbf{x}) > \tau^*\}$, if $D_{X_{\rm I}} (\mathbf{x}:r^*(\mathbf{x}) = \tau^*)=0$; otherwise, $\Omega_{X_{\rm I}}: = \{\mathbf{x}\in \mathcal{X}: r^*(\mathbf{x}) \geq \tau^*\}$. $\Omega_{X_{\rm O}}: = \mathcal{X}-\Omega_{X_{\rm I}}$.
Then we have that
\begin{equation*}
    \frac{C_{\theta}}{\sqrt{n^{1-\theta}}} \geq  \mathbb{E}_{S\sim D^n_{X_{\rm I}}} \sup_{\mathbf{1}_{r\leq \tau} \in \mathcal{G}_S} \mathbb{E}_{\mathbf{x}\sim D_{X_{\rm I}}} \mathbf{1}_{r(\mathbf{x})\leq \tau}\geq \frac{2}{b}\mathbb{E}_{S\sim D^n_{X_{\rm I}}} \sup_{\mathbf{1}_{r\leq \tau} \in \mathcal{G}_S}  \int_{\Omega_{X_{\rm I}}}\mathbf{1}_{r(\mathbf{x})\leq \tau}{\rm d}\mu(\mathbf{x}),
\end{equation*}
which implies that
\begin{equation*}
\begin{split}
  \mathbb{E}_{S\sim D^n_{X_{\rm I}}} \inf_{\mathbf{1}_{r\leq \tau} \in \mathcal{G}_S}  \int_{\Omega_{X_{\rm I}}}\mathbf{1}_{r(\mathbf{x})> \tau}{\rm d}\mu(\mathbf{x})+    \frac{b C_{\theta}}{2\sqrt{n^{1-\theta}}} &\geq  \mu(\Omega_{X_{\rm I}}).
      \end{split}
\end{equation*}
Therefore,
\begin{equation*}
    \mathbb{E}_{S\sim D^n_{X_{\rm I}}} \inf_{\mathbf{1}_{r\leq \tau} \in \mathcal{G}_S}  \int_{\mathcal{X}} 1- \mathbf{1}_{r(\mathbf{x})\leq \tau}{\rm d}\mu(\mathbf{x})+    \frac{b C_{\theta}}{2\sqrt{n^{1-\theta}}} \geq  \mu(\Omega_{X_{\rm I}}),
\end{equation*}
which implies that
\begin{equation}\label{Eq::mu1_auc}
    \mu(\Omega_{X_{\rm O}}) +   \frac{b C_{\theta}}{2\sqrt{n^{1-\theta}}} \geq \mathbb{E}_{S\sim D^n_{X_{\rm I}}} \sup_{\mathbf{1}_{r\leq \tau} \in \mathcal{G}_S}  R_{\mu} (\mathbf{1}_{r(\mathbf{x})\leq \tau}).
\end{equation}
Additionally, due to $\mathbf{1}_{r^*(\mathbf{x})\leq \tau-\epsilon}\in \mathcal{G}_S$, it is clear that
\begin{equation}\label{Eq::mu2_auc}
      \sup_{\mathbf{1}_{r \leq \tau}\in \mathcal{G}_S} R_{\mu}(\mathbf{1}_{r \leq \tau}) \geq \mu(\Omega_{X_{\rm O}}).
\end{equation}
Combining Eq. \eqref{Eq::mu1_auc} with Eq. \eqref{Eq::mu2_auc}, we have
\begin{equation*}
   |  \mathbb{E}_{S\sim D^n_{X_{\rm I}}}\sup_{\mathbf{1}_{r \leq \tau}\in \mathcal{G}_S} R_{\mu}(\mathbf{1}_{r \leq \tau}) - \mu(\Omega_{X_{\rm O}})|\leq \frac{b C_{\theta}}{2\sqrt{n^{1-\theta}}}.
\end{equation*}
By using Eq. \eqref{5.6q1_auc} and Eq. \eqref{Eq::mu2_auc}, we have
\begin{equation}\label{5.6q2_auc}
    \mathbb{E}_{S_n \sim \mu^n} \sup_{D_{X_{\rm I}} \in \mathcal{D}_{I}}\mathbb{E}_{S\sim D_{X_{\rm I}}^n}[\mu(\Omega_{X_{\rm O}})-R_{\mu}(\mathbf{A}_{S_n}(S))  ] \leq \frac{C_{\theta}}{\sqrt{n^{1-\theta}}}+\frac{b C_{\theta}}{2\sqrt{n^{1-\theta}}}.
\end{equation}
Using inequality \eqref{5.6q3_auc}, we have
\begin{equation}\label{5.6q40_auc}
    \mathbb{E}_{S_n \sim \mu^n} \sup_{D_{X_{\rm I}} \in \mathcal{D}_{I}} \mathbb{E}_{S\sim D_{X_{\rm I}}^n} \mathbb{E}_{\mathbf{x}\sim D_{X_{\rm I}}} {\mathbf{A}_{S_n}(S)}(\mathbf{x}) \leq \frac{C_{\theta}}{\sqrt{n^{1-\theta}}},
\end{equation}
which implies that
\begin{equation}\label{5.6q4_auc}
    \mathbb{E}_{S_n \sim \mu^n} \sup_{D_{X_{\rm I}} \in \mathcal{D}_{I}} \mathbb{E}_{S\sim D_{X_{\rm I}}^n} \int_{\Omega_{X_{\rm I}}} {\mathbf{A}_{S_n}(S)}(\mathbf{x}){\rm d}\mu(\mathbf{x}) \leq \frac{bC_{\theta}}{2\sqrt{n^{1-\theta}}}.
\end{equation}
Then inequalities \eqref{5.6q2_auc} and \eqref{5.6q4_auc} imply that
\begin{equation*}
    \mathbb{E}_{S_n \sim \mu^n} \sup_{D_{X_{\rm I}} \in \mathcal{D}_{I}} \mathbb{E}_{S\sim D_{X_{\rm I}}^n} \int_{\Omega_{X_{\rm O}}} 1- {\mathbf{A}_{S_n}(S)}(\mathbf{x}){\rm d}\mu(\mathbf{x}) \leq \frac{(1+b)C_{\theta}}{\sqrt{n^{1-\theta}}}.
\end{equation*}
Therefore,
\begin{equation*}
     \mathbb{E}_{S_n \sim \mu^n} \sup_{D_{X_{\rm I}} \in \mathcal{D}_{I}} \mathbb{E}_{S\sim D_{X_{\rm I}}^n} \int_{\mathcal{X}} 1-{\mathbf{A}_{S_n}(S)}(\mathbf{x}){\rm d}D_{X_{\rm O}}(\mathbf{x}) \leq \frac{2b(1+b)C_{\theta}}{\sqrt{n^{1-\theta}}}.
\end{equation*}
We assume that ${\mathbf{A}_{S_n}(S)} = \mathbf{1}_{r_{S_n,S}\leq \tau_{S_n,S}}$. Then above inequality implies that
\begin{equation*}
     \mathbb{E}_{S_n \sim \mu^n} \inf_{D_{X_{\rm I}} \in \mathcal{D}_{I}} \mathbb{E}_{S\sim D_{X_{\rm I}}^n} \mathbb{E}_{\mathbf{x}\sim D_{X_{\rm O}}} \mathbf{1}_{r_{S_n,S}(\mathbf{x})\leq \tau_{S_n,S}} \geq 1- \frac{2b(1+b)C_{\theta}}{\sqrt{n^{1-\theta}}}.
\end{equation*}
Inequality \eqref{5.6q40_auc} implies that
\begin{equation*}
    \mathbb{E}_{S_n \sim \mu^n} \inf_{D_{X_{\rm I}} \in \mathcal{D}_{I}} \mathbb{E}_{S\sim D_{X_{\rm I}}^n} \mathbb{E}_{\mathbf{x}\sim D_{X_{\rm I}}} \mathbf{1}_{r_{S_n,S}(\mathbf{x})> \tau_{S_n,S}} \geq 1- \frac{C_{\theta}}{\sqrt{n^{1-\theta}}},
\end{equation*}
which shows that there exists $S_n'\sim \mu^n$ and $C'$ such that
\begin{equation*}
\begin{split}
& \inf_{D_{X_{\rm I}} \in \mathcal{D}_{I}}\mathbb{E}_{S\sim D_{X_{\rm I}}^n}  {\rm AUC}(f_{S_n',S};D_{X_{\rm I}},D_{X_{\rm O}}) \\   \geq &  \inf_{D_{X_{\rm I}} \in \mathcal{D}_{I}} \mathbb{E}_{S\sim D_{X_{\rm I}}^n} \mathbb{E}_{\mathbf{x}\sim D_{X_{\rm O}}} \mathbb{E}_{\mathbf{x}'\sim D_{X_{\rm I}}} \mathbf{1}_{r_{S_n',S}(\mathbf{x})< r_{S_n',S}(\mathbf{x}')} \\ \geq &   \inf_{D_{X_{\rm I}} \in \mathcal{D}_{I}} \mathbb{E}_{S\sim D_{X_{\rm I}}^n} \mathbb{E}_{\mathbf{x}\sim D_{X_{\rm O}}} \mathbb{E}_{\mathbf{x}'\sim D_{X_{\rm I}}} \mathbf{1}_{r_{S_n',S}(\mathbf{x})\leq \tau_{S_n',S}} \mathbf{1}_{r_{S_n',S}(\mathbf{x}')>  \tau_{S_n',S}}\\ \geq & 1- \frac{\max\{C_{\theta},C'\}}{\sqrt{n^{1-\theta}}}-\frac{2b(1+b)C_{\theta}}{\sqrt{n^{1-\theta}}}.
     \end{split}
\end{equation*}
We set data point sequences $\mathbf{S}'=(S_1',S_2',...,S_n',...)$. Then, the function part  $r_{S_n',S}$ of $\mathbf{A}_{\mathbf{S}'}\in \mathscr{A}$ is the universally consistent algorithm, \textit{i.e.}, 
\begin{equation*}
\begin{split}
& \inf_{D_{X_{\rm I}} \in \mathcal{D}_{I}}\mathbb{E}_{S\sim D_{X_{\rm I}}^n}  {\rm AUC}(r_{S_n',S};D_{X_{\rm I}},D_{X_{\rm O}})  \geq 1- \frac{\max\{C_{\theta},C'\}}{\sqrt{n^{1-\theta}}}-\frac{2b(1+b)C_{\theta}}{\sqrt{n^{1-\theta}}}.
     \end{split}
\end{equation*}
We have completed this proof.
\end{proof}
\section{Proofs of Proposition \ref{prop_fcnn_rank1} and Proposition \ref{prop_score_rank}}\label{SK_auc}

\propFcnnRankingOne*
\begin{proof}[Proof of Proposition \ref{prop_fcnn_rank1} ]
    \textbf{Firstly}, we prove that if  some $s$ with $1<s<g$, $d=l_1\leq l_2\leq...\leq l_s$, and $l_s \geq 2d$, $\mathcal{F}_{\mathbf{q}}^{\sigma}$ is the separate ranking function space.
    \\

    \noindent First, we show that if $\mathcal{X}\subset \mathbb{R}^1$, and $\mathbf{q} =(1, 2,1)$, then $\mathcal{F}_{\mathbf{q}}^{\sigma}$ is the separate ranking space.
\\
For any ${x}\in \mathcal{X}$, 
    \begin{equation*}
    f_x(x') =   \left[1,1\right] \sigma (
 \left[            
  \begin{array}{ccc}   
    1\\  
    -1\\  
  \end{array}
\right ]                
{x}'+ \left[             
  \begin{array}{ccc}   
    -{x}\\  
   {x}\\  
  \end{array}
\right])+0.
\end{equation*}
It is easy to check that for any $x'\neq x$,
\begin{equation*}
    f_x(x)<f_x(x').
\end{equation*}
Next, we prove that if $\mathcal{X}\subset \mathbb{R}^d$, and $\mathbf{q} =(d, 2d,1)$, then $\mathcal{F}_{\mathbf{q}}^{\sigma}$ is the separate ranking space.
\\
Let $\mathbf{v}_i\in \mathbb{R}^{2d\times 1}$ is a vector whose $2i$-th and $2i+1$-th coordinates are $-1$; otherwise, other coordinates are $0$. For any $\mathbf{x}=[x_1,...,x_d]^\top\in \mathcal{X}$, 
    \begin{equation*}
    f_{\mathbf{x}}(\mathbf{x}') =   \left[1,1...,1\right]_{1\times {(2d)}} \sigma (
 \left[             
    \mathbf{v}_1,\mathbf{v}_2,...,\mathbf{v}_d   
\right ]                
\mathbf{x}'+ \mathbf{M_{d+1}})+0,
\end{equation*}
where $\mathbf{M}_{d+1}\in \mathbb{R}^{2d\times 1}$ is a vector whose $2i$-th coordinate is $-x_i$ and $(2i+1)$-th coordinate is $x_i$.
It is easy to check that for any $\mathbf{x}'\neq \mathbf{x}$,
\begin{equation*}
    f_\mathbf{x}(\mathbf{x})<f_\mathbf{x}(\mathbf{x}').
\end{equation*}
Thirdly, we prove that if $d=l_1=l_2=...=l_{r-1}$ and $l_r = 2d$, then $\mathcal{F}_{\mathbf{q}}^{\sigma}$ is the separate ranking space. Due to $\mathcal{X}$ is bounded, we can find $b$ such that any $\mathbf{x}=[{x}_1,...,x_d]^\top \in \mathcal{X}$ satisfies that $x_i+b\geq 0$. Then, we set $\mathbf{w}_2, \mathbf{w}_3,...,\mathbf{w}_{r-1}$ are identity matrices, $\mathbf{b}_2$ is the matrix whose all coordinates are $b$, and $\mathbf{b}_3,...,\mathbf{b}_{r-1}$ are $\mathbf{0}$. Then the result in second step implies the result. 
\\
Finally, using the result that $\mathbf{q}\lesssim\mathbf{q}'\Rightarrow \mathcal{F}_{\mathbf{q}}^{\sigma}\subset \mathcal{F}_{\mathbf{q}'}^{\sigma}$ (Lemma \ref{L7-contain}) implies the final result.
\\
\\
\textbf{Secondly}, it is clear that $\mathcal{F}_{\mathbf{q}}^{\sigma}$ is constant closure. We omit the proof.
\\
\textbf{Thirdly}, by Theorems 5 and 8 in \citep{bartlett2003vapnik}, we can obtain the third result.
\end{proof}
\propScoreRanking*
\begin{proof}[Proof of Proposition \ref{prop_score_rank} ]
\textbf{Firstly}, we prove that if $\mathcal{F}_{\mathbf{q}'}^{\sigma}$ is a separate ranking function space, $\mathcal{R}$ is the separate ranking function space.
\\

\noindent For the softmax-based function:
    \begin{equation*}
    E(\mathbf{f}) =\max_{k\in\{1,...,l\}}  \frac{\exp{(f^k)}}{\sum_{c=1}^l \exp{(f^c)}}.
\end{equation*}
Note that $\mathcal{F}_{\mathbf{q}'}^{\sigma}$ is separate ranking space. Then for any $\mathbf{x}\in \mathcal{X}$, there exists $f_{\mathbf{x}}\in \mathcal{F}_{\mathbf{q}'}^{\sigma}$ such that $0=f_{\mathbf{x}}(\mathbf{x})<f_{\mathbf{x}}(\mathbf{x}')$, for any $\mathbf{x}'\in \mathcal{X}$ and $\mathbf{x}' \neq \mathbf{x}$. Then $\mathbf{f}_{\mathbf{x}}=[f_{\mathbf{x}},-f_{\mathbf{x}},..,-f_{\mathbf{x}}] \in \mathcal{F}_{\mathbf{q}}^{\sigma}$ can ensure that for any $\mathbf{x}'\in \mathcal{X}$ and $\mathbf{x}' \neq \mathbf{x}$,
\begin{equation*}
    E(\mathbf{f}_{\mathbf{x}}(\mathbf{x}))< E(\mathbf{f}_{\mathbf{x}}(\mathbf{x}')).
\end{equation*}
Using the same strategy, we can prove that the temperature-scaled function is the separate ranking space.
  For the energy-based function:
    \begin{equation*}
  E(\mathbf{f}) =T\log \sum_{c=1}^l  \exp{(f^c/T)}.
\end{equation*}
Note that $\mathcal{F}_{\mathbf{q}'}^{\sigma}$ is separate ranking space. Then for any $\mathbf{x}\in \mathcal{X}$, there exists $f_{\mathbf{x}}\in \mathcal{F}_{\mathbf{q}'}^{\sigma}$ such that $0=f_{\mathbf{x}}(\mathbf{x})<f_{\mathbf{x}}(\mathbf{x}')$, for any $\mathbf{x}'\in \mathcal{X}$ and $\mathbf{x}' \neq \mathbf{x}$. Then $\mathbf{f}_{\mathbf{x}}=[f_{\mathbf{x}},f_{\mathbf{x}},..,f_{\mathbf{x}}] \in \mathcal{F}_{\mathbf{q}}^{\sigma}$ can ensure that for any $\mathbf{x}'\in \mathcal{X}$ and $\mathbf{x}' \neq \mathbf{x}$,
\begin{equation*}
    E(\mathbf{f}_{\mathbf{x}}(\mathbf{x}))< E(\mathbf{f}_{\mathbf{x}}(\mathbf{x}')).
\end{equation*}
\textbf{Secondly}, it is easy to show that $R$ is constant closure. We omit it.
\\
\textbf{Thirdly}, by Theorems 5 and 8 in \citep{bartlett2003vapnik}, we can obtain the third result.
\end{proof}

\section{Proofs of Proposition \ref{Pr1} and Proposition \ref{P2}}\label{SK}
To better understand the contents in Appendices \ref{SK}-\ref{SM},
we introduce the important notations for FCNN-based hypothesis space and score-based hypothesis space detaily.
\\

\noindent  \textbf{FCNN-based Hypothesis Space.} Given a sequence $\mathbf{q}=(l_1,l_2,...,l_g)$, where $l_i$ and $g$ are positive integers and $g>2$, we use $g$ to represent the depth of neural network and use $l_i$ to represent the width of the $i$-th layer. After the activation function $\sigma$ is selected, we can obtain the architecture of FCNN according to the sequence $\mathbf{q}$. Given any weights  $\mathbf{w}_i \in  \mathbb{R}^{l_{i}\times l_{i-1}}$ and bias $\mathbf{b}_i \in \mathbb{R}^{l_{i}\times 1}$, the output of the $i$-layer can be written as follows: for any $\mathbf{x}\in \mathbb{R}^{l_1}$,
\begin{equation*}
    \mathbf{f}_i(\mathbf{x}) = \sigma(\mathbf{w}_i\mathbf{f}_{i-1}(\mathbf{x})+\mathbf{b}_i), ~ \forall i=2,...,g-1,
\end{equation*}
where $\mathbf{f}_{i-1}(\mathbf{x})$ is the $i$-th layer output and $\mathbf{f}_1(\mathbf{x})=\mathbf{x}$. Then, the output of FCNN is
$
    \mathbf{f}_{\mathbf{w},\mathbf{b}}(\mathbf{x})= \mathbf{w}_{g}\mathbf{f}_{{g-1}}(\mathbf{x})+\mathbf{b}_g,
$
where $\mathbf{w}=\{\mathbf{w}_2,...,\mathbf{w}_g\}$ and $\mathbf{b}=\{\mathbf{b}_2,...,\mathbf{b}_g\}$.

\noindent An FCNN-based scoring function space is defined as:
\begin{equation*}
    \mathcal{F}_{\mathbf{q}}^{\sigma}:=\{\mathbf{f}_{\mathbf{w},\mathbf{b}}:\forall \mathbf{w}_i\in \mathbb{R}^{l_{i}\times l_{i-1}},~\forall \mathbf{b}_i\in \mathbb{R}^{l_{i}\times 1},~i=2,...,g\}.
\end{equation*}

\noindent Additionally, given two sequences $\mathbf{q}=(l_1,...,l_g)$ and $\mathbf{q}'=(l_1',...,l_{g'}')$, we use the notation
$
    \mathbf{q}\lesssim\mathbf{q}'
$
to represent the following equations and inequalities:
\begin{equation*}
\begin{split}
g \leq g', ~~~l_1&=l_1',~~~ l_g=l_{g'}', \\
l_i&\leq l'_i, ~~~\forall i=1,...,g-1, \\
l_{g-1}&\leq l'_i, ~~~\forall i=g,...,g'-1.
\end{split}
\end{equation*}

\noindent Given a sequence $\mathbf{q}=(l_1,...l_g)$ satisfying that $l_1=d$ and $l_g=K+1$, the FCNN-based scoring function space $\mathcal{F}_{\mathbf{q}}^{\sigma}$ can induce an FCNN-based hypothesis space. Before defining the FCNN-based hypothesis space, we define the induced hypothesis function. For any $\mathbf{f}_{\mathbf{w},\mathbf{b}}\in \mathcal{F}_{\mathbf{q}}^{\sigma}$, the induced hypothesis function  is:
\begin{equation*}
    h_{\mathbf{w},\mathbf{b}}(\mathbf{x}):=\argmax_{k\in\{1,...,K+1\}} {f}^k_{\mathbf{w},\mathbf{b}}(\mathbf{x}),~\forall \mathbf{x}\in \mathcal{X},
\end{equation*}
where ${f}^k_{\mathbf{w},\mathbf{b}}(\mathbf{x})$ is the $k$-th coordinate of $\mathbf{f}_{\mathbf{w},\mathbf{b}}(\mathbf{x})$. Then, we define the  FCNN-based hypothesis space as follows:
\begin{equation*}
  \mathcal{H}_{\mathbf{q}}^{\sigma}:= \{h_{\mathbf{w},\mathbf{b}}:\forall \mathbf{w}_i\in \mathbb{R}^{l_{i}\times l_{i-1}},~\forall \mathbf{b}_i\in \mathbb{R}^{l_{i}\times 1},~i=2,...,g\}.
\end{equation*}
\\
\textbf{Score-based Hypothesis Space.}
 Many OOD algorithms detect OOD data using a score-based strategy. That is, given a threshold $\lambda$, a scoring function space $\mathcal{F}_l\subset \{\mathbf{f}:\mathcal{X}\rightarrow \mathbb{R}^l\}$ and a scoring function $E: \mathcal{F}_l\rightarrow \mathbb{R}$, then $\mathbf{x}$ is regarded as ID, if $E(\mathbf{f}(\mathbf{x}))\geq \lambda$; otherwise, $\mathbf{x}$ is regarded as OOD. 
\\
\\
\noindent Using $E$, $\lambda$ and $\mathbf{f}\in \mathcal{F}_{\mathbf{q}}^{\sigma}$, we can generate a binary classifier $h^{\lambda}_{\mathbf{f},E}$:
\begin{equation*}
    h^{\lambda}_{\mathbf{f},E}(\mathbf{x}) := \left \{
    \begin{aligned}
   &~~~~1,~~{\rm if}~E(\mathbf{f}(\mathbf{x}))\geq \lambda;\\ 
   &~~~~  2,~~{\rm if}~E(\mathbf{f}(\mathbf{x}))< \lambda,
    \end{aligned}
    \right.
\end{equation*}
where $1$ represents ID data, and $2$ represents OOD data. Hence, a binary classification hypothesis space $\mathcal{H}^b$, which consists of all $h^{\lambda}_{\mathbf{f},E}$, is generated. We define the score-based hypothesis space $
    \mathcal{H}^{{\sigma},\lambda}_{{\mathbf{q}},E} := \{h^{\lambda}_{\mathbf{f},E}:\forall \mathbf{f}\in \mathcal{F}_{\mathbf{q}}^{\sigma}\}
$.

\noindent Next, we introduce two important propositions.



\begin{Proposition}\label{Pr1}
Given a sequence $\mathbf{q}=(l_1,...l_g)$  satisfying that $l_1=d$ and $l_g=K+1$ (note that $d$ is the dimension of input data and $K+1$ is the dimension of output), then the constant functions $h_{1}$, $h_{2}$,...,$h_{K+1}$ belong to $\mathcal{H}_{\mathbf{q}}^{\sigma}$, where $h_{i}(\mathbf{x})=i$, for any $\mathbf{x}\in \mathcal{X}$. Therefore, Assumption \ref{ass1} holds for $\mathcal{H}_{\mathbf{q}}^{\sigma}$.
\end{Proposition}

\begin{proof}[Proof of Proposition \ref{Pr1}]
Note that the output of FCNN can be written as
\begin{equation*}
    \mathbf{f}_{\mathbf{w},\mathbf{b}}(\mathbf{x})= \mathbf{w}_{g}\mathbf{f}_{g-1}(\mathbf{x})+\mathbf{b}_g,
\end{equation*}
where $\mathbf{w}_{g}\in \mathbb{R}^{(K+1)\times l_{g-1}}$, $\mathbf{b}_g\in  \mathbb{R}^{(K+1)\times 1}$ and $\mathbf{f}_{{g-1}}(\mathbf{x})$ is the output of the $l_{g-1}$-th layer. If we set $\mathbf{w}_{g}=\mathbf{0}$, and set $\mathbf{b}_g=\mathbf{y}_i$, where $\mathbf{y}_i$ is the one-hot vector corresponding to label $i$. Then $\mathbf{f}_{\mathbf{w},\mathbf{b}}(\mathbf{x}) = \mathbf{y}_i$, for any $\mathbf{x}\in \mathcal{X}$. Therefore, $h_i(\mathbf{x})\in \mathcal{H}_{\mathbf{q}}^{\sigma}$, for any $i=1,...,K,K+1.$
\end{proof}

\noindent Note that in some works \citep{Safran2017Depth-Width}, $\mathbf{b}_g$ is fixed to $\mathbf{0}$. In fact, it is easy to check that when $g>2$ and activation function $\sigma$ is not a constant, Proposition \ref{P1} still holds, even if $\mathbf{b}_g=\mathbf{0}$.

\begin{Proposition}\label{P2}
For any sequence $\mathbf{q}=(l_1,...,l_g)$ satisfying that $l_1=d$ and $l_g=l$ $($note that $d$ is the dimension of input data and $l$ is the dimension of output$)$, if $\{\mathbf{v}\in \mathbb{R}^l:E(\mathbf{v}) \geq \lambda\}\neq \emptyset$ and  $\{\mathbf{v}\in \mathbb{R}^l:E(\mathbf{v}) < \lambda\}\neq \emptyset$, then the functions $h_{1}$ and $h_{2}$ belong to $\mathcal{H}^{{\sigma},\lambda}_{{\mathbf{q}},E}$, where $h_{1}(\mathbf{x})=1$ and $h_{2}(\mathbf{x})=2$, for any $\mathbf{x}\in \mathcal{X}$, where $1$ represents the ID labels, and $2$ represents the OOD labels. Therefore, Assumption \ref{ass1} holds.
\end{Proposition}

\begin{proof}[Proof of Proposition \ref{P2}]
Since $\{\mathbf{v}\in \mathbb{R}^l:E(\mathbf{v})\geq \lambda\}\neq \emptyset $ and $\{\mathbf{v}\in \mathbb{R}^l:E(\mathbf{v})<\lambda\}\neq \emptyset $, we can find $\mathbf{v}_1\in \{\mathbf{v}\in \mathbb{R}^l:E(\mathbf{v})\geq \lambda\}$ and $\mathbf{v}_2\in \{\mathbf{v}\in \mathbb{R}^l:E(\mathbf{v})< \lambda\}$.

\noindent For any  $\mathbf{f}_{\mathbf{w},\mathbf{b}}\in \mathcal{F}^{\sigma}_{\mathbf{q}}$, we have \begin{equation*}
    \mathbf{f}_{\mathbf{w},\mathbf{b}}(\mathbf{x})= \mathbf{w}_{g}\mathbf{f}_{{g-1}}(\mathbf{x})+\mathbf{b}_g,
\end{equation*}
where $\mathbf{w}_{g}\in \mathbb{R}^{l\times l_{g-1}}$, $\mathbf{b}_g\in  \mathbb{R}^{l\times 1}$ and $\mathbf{f}_{{g-1}}(\mathbf{x})$ is the output of the $l_{g-1}$-th layer. 
\\
\\
\noindent If we set $\mathbf{w}_g=\mathbf{0}_{l\times l_{g-1}}$ and $\mathbf{b}_g=\mathbf{v}_1$, then $\mathbf{f}_{\mathbf{w},\mathbf{b}}(\mathbf{x})=\mathbf{v}_1$ for any $\mathbf{x} \in \mathcal{X}$, where $\mathbf{0}_{l\times l_{g-1}}$ is $l\times l_{g-1}$ zero matrix. Hence, $h_1$ can be induced by $\mathbf{f}_{\mathbf{w},\mathbf{b}}$. Therefore, $h_1\in \mathcal{H}_{\mathbf{q},E}^{\sigma,\lambda}$.
\\
\\
\noindent Similarly, if we set $\mathbf{w}_g=\mathbf{0}_{l\times l_{g-1}}$ and $\mathbf{b}_g=\mathbf{v}_2$, then $\mathbf{f}_{\mathbf{w},\mathbf{b}}(\mathbf{x})=\mathbf{v}_2$ for any $\mathbf{x} \in \mathcal{X}$, where $\mathbf{0}_{l\times l_{g-1}}$ is $l\times l_{g-1}$ zero matrix. Hence, $h_2$ can be induced by $\mathbf{f}_{\mathbf{w},\mathbf{b}}$. Therefore, $h_2\in \mathcal{H}_{\mathbf{q},E}^{\sigma,\lambda}$.
\end{proof}
It is easy to check that when $g>2$ and activation function $\sigma$ is not a constant, Proposition \ref{P2} still holds, even if $\mathbf{b}_g=\mathbf{0}$.

\section{Proof of Theorem \ref{T24}}\label{SL}

Before proving Theorem \ref{T24}, we need several lemmas.

\begin{lemma}\label{L22}
Let $\sigma$ be ReLU function: $\max\{x,0\}$. Given $\mathbf{q}=(l_1,...,l_g)$ and $\mathbf{q}'=(l_1',...,l_{g}')$ such that $l_g=l_{g}'$ and $l_1=l_{1}'$, and $l_i\leq l'_i$ $(i=1,...,g-1)$, then $\mathcal{F}_{\mathbf{q}}^{\sigma}\subset \mathcal{F}_{\mathbf{q}'}^{\sigma}$ and $\mathcal{H}_{\mathbf{q}}^{\sigma}\subset \mathcal{H}_{\mathbf{q}'}^{\sigma}$.
\end{lemma}
\begin{proof}[Proof of Lemma \ref{L22}]
Given any weights  $\mathbf{w}_i \in  \mathbb{R}^{l_{i}\times l_{i-1}}$ and bias $\mathbf{b}_i \in \mathbb{R}^{l_{i}\times 1}$, the $i$-layer output of FCNN with architecture $\mathbf{q}$ can be written as
\begin{equation*}
    \mathbf{f}_i(\mathbf{x}) = \sigma(\mathbf{w}_i\mathbf{f}_{i-1}(\mathbf{x})+\mathbf{b}_i),~~\forall \mathbf{x}\in \mathbb{R}^{l_1}, \forall i=2,...,g-1,
\end{equation*}
where $\mathbf{f}_{i-1}(\mathbf{x})$ is the $i$-th layer output and $\mathbf{f}_1(\mathbf{x})=\mathbf{x}$. Then, the output of last layer is
\begin{equation*}
    \mathbf{f}_{\mathbf{w},\mathbf{b}}(\mathbf{x})= \mathbf{w}_{g}\mathbf{f}_{{g-1}}(\mathbf{x})+\mathbf{b}_g.
\end{equation*}
We will show that $\mathbf{f}_{\mathbf{w},\mathbf{b}}
\in \mathcal{F}_{\mathbf{q}'}^{\sigma}$. We construct $\mathbf{f}_{\mathbf{w}',\mathbf{b}'}$ as follows: for every $\mathbf{w}'_i\in \mathbb{R}^{l_{i}'\times l_{i-1}'}$, if $l_i'-l_i>0$ and $l_{i-1}'-l_{i-1}>0$, we set
\begin{equation*}
    \mathbf{w}'_i = \left[\begin{matrix}
 &\mathbf{w}_i&\mathbf{0}_{l_i\times (l_{i-1}'-l_{i-1})}\\
& \mathbf{0}_{(l_i'-l_i)\times l_{i-1}'}& \mathbf{0}_{(l_i'-l_i)\times (l_{i-1}'-l_{i-1})}
  \end{matrix}\right],~~~ \mathbf{b}'_i = \left[\begin{matrix}
 &\mathbf{b}_i\\
& \mathbf{0}_{(l_i'-l_i)\times 1}
  \end{matrix}\right]
\end{equation*}
where $\mathbf{0}_{pq}$ means the $p\times q$ zero matrix.
If $l_i'-l_i=0$ and $l_{i-1}'-l_{i-1}>0$, we set
\begin{equation*}
    \mathbf{w}'_i = \left[\begin{matrix}
 &\mathbf{w}_i&\mathbf{0}_{l_i\times (l_{i-1}'-l_{i-1})}
  \end{matrix}\right],~~~\mathbf{b}'_i = \mathbf{b}_i.
\end{equation*}
If $l_{i-1}'-l_{i-1}=0$ and $l_i'-l_i>0$, we set
\begin{equation*}
    \mathbf{w}'_i = \left[\begin{matrix}
 &\mathbf{w}_i\\
&\mathbf{0}_{(l_i'-l_i)\times l_{i-1}'}
  \end{matrix}\right], ~~~ \mathbf{b}'_i = \left[\begin{matrix}
 &\mathbf{b}_i\\
& \mathbf{0}_{(l_i'-l_i)\times 1}
\end{matrix}\right].
\end{equation*}
If $l_{i-1}'-l_{i-1}=0$ and $l_i'-l_i=0$, we set
\begin{equation*}
    \mathbf{w}'_i = \mathbf{w}_i, ~~~ \mathbf{b}'_i = \mathbf{b}_i.
\end{equation*}
It is easy to check that if $l_i'-l_i>0$
\begin{equation*}
    \mathbf{f}_i'= \left[\begin{matrix}
 &\mathbf{f}_i\\
& \mathbf{0}_{(l_i'-l_i)\times 1}
  \end{matrix}\right].
\end{equation*}

\noindent If $l_i'-l_i=0$,
\begin{equation*}
    \mathbf{f}_i'=\mathbf{f}_i.
\end{equation*}

\noindent Since $l_g'-l_g=0$,
\begin{equation*}
    \mathbf{f}_g'=\mathbf{f}_g,~ \textit{i.e.}, ~  \mathbf{f}_{\mathbf{w}',\mathbf{b}'}=\mathbf{f}_{\mathbf{w},\mathbf{b}}.
\end{equation*}
Therefore, $f_{\mathbf{w},\mathbf{b}}\in \mathcal{F}_{\mathbf{q}'}^{\sigma}$, which implies that $\mathcal{F}_{\mathbf{q}}^{\sigma}\subset \mathcal{F}_{\mathbf{q}'}^{\sigma}$. Therefore, $\mathcal{H}_{\mathbf{q}}^{\sigma}\subset \mathcal{H}_{\mathbf{q}'}^{\sigma}$.
\end{proof}
\begin{lemma}\label{L7-contain}
Let $\sigma$ be the ReLU function: $\sigma(x)=\max\{x,0\}$. Then, $\mathbf{q}\lesssim \mathbf{q}'$ implies that $\mathcal{F}_{\mathbf{q}}^{\sigma}\subset \mathcal{F}_{\mathbf{q}'}^{\sigma}$, $\mathcal{H}_{\mathbf{q}}^{\sigma}\subset \mathcal{H}_{\mathbf{q}'}^{\sigma}$, where $\mathbf{q}=(l_1,...,l_g)$ and $\mathbf{q}'=(l_1',...,l_{g'}')$.
\end{lemma}
\begin{proof}[Proof of Lemma \ref{L7-contain}] 
Given $l''=(l_1'',...,l_{g''}'')$ satisfying that $g\leq g''$, $l_i''=l_i$ for $i=1,...,g-1$, $l_i''=l_{g-1}$ for $i=g,...,g''-1$, and $l_{g''}''=l_{g}$, we first prove that $\mathcal{F}_{\mathbf{q}}^{\sigma}\subset \mathcal{F}_{\mathbf{q}''}^{\sigma}$ and $\mathcal{H}_{\mathbf{q}}^{\sigma}\subset \mathcal{H}_{\mathbf{q}''}^{\sigma}$.
\\
\\
\noindent Given any weights  $\mathbf{w}_i \in  \mathbb{R}^{l_{i}\times l_{i-1}}$ and bias $\mathbf{b}_i \in \mathbb{R}^{l_{i}\times 1}$, the $i$-th layer output of FCNN with architecture $\mathbf{q}$ can be written as
\begin{equation*}
    \mathbf{f}_i(\mathbf{x}) = \sigma(\mathbf{w}_i\mathbf{f}_{i-1}(\mathbf{x})+\mathbf{b}_i),~~\forall \mathbf{x}\in \mathbb{R}^{l_1}, \forall i=2,...,g-1,
\end{equation*}
where $\mathbf{f}_{i-1}(\mathbf{x})$ is the $i$-th layer output and $\mathbf{f}_1(\mathbf{x})=\mathbf{x}$. Then, the output of the last layer is
\begin{equation*}
    \mathbf{f}_{\mathbf{w},\mathbf{b}}(\mathbf{x})= \mathbf{w}_{g}\mathbf{f}_{{g-1}}(\mathbf{x})+\mathbf{b}_g.
\end{equation*}
We will show that $\mathbf{f}_{\mathbf{w},\mathbf{b}}
\in \mathcal{F}_{\mathbf{q}''}^{\sigma}$. We construct $\mathbf{f}_{\mathbf{w}'',\mathbf{b}''}$ as follows: if $i=2,...,g-1$, then $\mathbf{w}''_i=\mathbf{w}$ and $\mathbf{b}_i''=\mathbf{b}_i$; if $i=g,...,g''-1$, then $\mathbf{w}''_i=\mathbf{I}_{l_{g-1}\times l_{g-1}}$ and $\mathbf{b}''_i=\mathbf{0}_{l_{g-1}\times 1}$, where $\mathbf{I}_{l_{g-1}\times l_{g-1}}$ is the $l_{g-1}\times l_{g-1}$ identity matrix, and $\mathbf{0}_{l_{g-1}\times 1}$ is the $l_{g-1}\times 1$ zero matrix; and if $i=g''$, then $\mathbf{w}''_{g''}=\mathbf{w}_{g}$, $\mathbf{b}''_{g''}=\mathbf{b}_{g}$. Then it is easy to check that the output of the $i$-th layer is
\begin{equation*}
    \mathbf{f}''_i=\mathbf{f}_{g-1}, \forall i=g-1,g,...,g''-1.
\end{equation*}
Therefore, $\mathbf{f}_{\mathbf{w}'',\mathbf{b}''}=\mathbf{f}_{\mathbf{w},\mathbf{b}}$, which implies that $\mathcal{F}_{\mathbf{q}}^{\sigma}\subset \mathcal{F}_{\mathbf{q}''}^{\sigma}$. Hence, $\mathcal{H}_{\mathbf{q}}^{\sigma}\subset \mathcal{H}_{\mathbf{q}''}^{\sigma}$.

\noindent When $g'' = g'$, we use Lemma \ref{L22} ($\mathbf{q}''$ and $\mathbf{q}$ satisfy the condition in Lemma \ref{L22}), which implies that $\mathcal{F}_{\mathbf{q}''}^{\sigma}\subset \mathcal{F}_{\mathbf{q}'}^{\sigma}$, $\mathcal{H}_{\mathbf{q}''}^{\sigma}\subset \mathcal{H}_{\mathbf{q}'}^{\sigma}$. Therefore, $\mathcal{F}_{\mathbf{q}}^{\sigma}\subset \mathcal{F}_{\mathbf{q}'}^{\sigma}$, $\mathcal{H}_{\mathbf{q}}^{\sigma}\subset \mathcal{H}_{\mathbf{q}'}^{\sigma}$.
\end{proof}
\begin{lemma}\citep{pinkus1999approximation}\label{L7_UPT}
If the activation function $\sigma$ is not a polynomial, then for any continuous function $f$ defined in $\mathbb{R}^d$, and any compact set $C\subset \mathbb{R}^d$, there exists a fully-connected neural network with architecture $\mathbf{q}$ $(l_1=d,l_g=1)$ such that
\begin{equation*}
 \inf_{f_{\mathbf{w},\mathbf{b}}\in \mathcal{F}_{\mathbf{q}}^{\sigma}} \max_{\mathbf{x}\in C} |f_{\mathbf{w},\mathbf{b}}(\mathbf{x})-f(\mathbf{x})|<\epsilon.
\end{equation*}
\end{lemma}
\begin{proof}[Proof of Lemma \ref{L7_UPT}]
The proof of Lemma \ref{L7_UPT} can be found in Theorem 3.1 in \citep{pinkus1999approximation}.
\end{proof}
\begin{lemma}\label{L8UPT}
If the activation function $\sigma$ is the ReLU function, then for any continuous vector-valued function $\mathbf{f}\in C(\mathbb{R}^d;\mathbb{R}^l)$, and any compact set $C\subset \mathbb{R}^d$, there exists a fully-connected neural network with architecture $\mathbf{q}$ $(l_1=d,l_g=l)$ such that
\begin{equation*}
 \inf_{\mathbf{f}_{\mathbf{w},\mathbf{b}}\in \mathcal{F}_{\mathbf{q}}^{\sigma}} \max_{\mathbf{x}\in C} \|\mathbf{f}_{\mathbf{w},\mathbf{b}}(\mathbf{x})-\mathbf{f}(\mathbf{x})\|_2<\epsilon,
\end{equation*}
where $\|\cdot\|_{2}$ is the $\ell_2$ norm. $($Note that we can also prove the same result, if $\sigma$ is not a polynomial.$)$
\end{lemma}\label{L8_UPT}
\begin{proof}[Proof of Lemma \ref{L8UPT}]
Let $\mathbf{f}=[f_1,...,f_l]^{\top}$, where $f_i$ is the $i$-th coordinate of $\mathbf{f}$. Based on Lemma \ref{L7_UPT}, we obtain $l$ sequences $\mathbf{q}^1$, $\mathbf{q}^2$,...,$\mathbf{q}^l$ such that
\begin{equation*}
\begin{split}
     &\inf_{g_{1}\in \mathcal{F}_{\mathbf{q}^1}^{\sigma}} \max_{\mathbf{x}\in C} |g_{1}(\mathbf{x})-f_1(\mathbf{x})|<\epsilon/\sqrt{l},
   \\ & \inf_{g_{2}\in \mathcal{F}_{\mathbf{q}^2}^{\sigma}} \max_{\mathbf{x}\in C} |g_{2}(\mathbf{x})-f_2(\mathbf{x})|<\epsilon/\sqrt{l},
      \\&~~~~~~~~~~~~~~~~~~~~~~~~~~~~~~~~~...
        \\&~~~~~~~~~~~~~~~~~~~~~~~~~~~~~~~~~...
      \\
      & \inf_{g_{l}\in \mathcal{F}_{\mathbf{q}^l}^{\sigma}} \max_{\mathbf{x}\in C} |g_{l}(\mathbf{x})-f_l(\mathbf{x})|<\epsilon/\sqrt{l}.
     \end{split}
\end{equation*}
 It is easy to find a sequence $\mathbf{q}=(l_1,...,l_g)$ ($l_g=1$) such that $\mathbf{q}^i\lesssim\mathbf{q}$, for all $i=1,...,l$. Using Lemma \ref{L7-contain},  we obtain that $\mathcal{F}_{\mathbf{q}^i}^{\sigma}\subset \mathcal{F}_{\mathbf{q}}^{\sigma}$. Therefore,
 \begin{equation*}
\begin{split}
     &\inf_{g\in \mathcal{F}_{\mathbf{q}}^{\sigma}} \max_{\mathbf{x}\in C} |g(\mathbf{x})-f_1(\mathbf{x})|<\epsilon/\sqrt{l},
   \\ & \inf_{g\in \mathcal{F}_{\mathbf{q}}^{\sigma}} \max_{\mathbf{x}\in C} |g(\mathbf{x})-f_2(\mathbf{x})|<\epsilon/\sqrt{l},
      \\&~~~~~~~~~~~~~~~~~~~~~~~~~~~~~~~~~...
        \\&~~~~~~~~~~~~~~~~~~~~~~~~~~~~~~~~~...
      \\
      & \inf_{g\in \mathcal{F}_{\mathbf{q}}^{\sigma}} \max_{\mathbf{x}\in C} |g(\mathbf{x})-f_l(\mathbf{x})|<\epsilon/\sqrt{l}.
     \end{split}
\end{equation*}
Therefore, for each $i$, we can find $g_{\mathbf{w}^i,\mathbf{b}^i}$ from $\mathcal{F}_{\mathbf{q}}^{\sigma}$ such that
\begin{equation*}
     \max_{\mathbf{x}\in C} |g_{\mathbf{w}^i,\mathbf{b}^i}(\mathbf{x})-f_i(\mathbf{x})|<\epsilon/\sqrt{l},
\end{equation*}
where $\mathbf{w}^i$ represents weights and $\mathbf{b}^i$ represents bias.

\noindent We construct a larger FCNN with $\mathbf{q}'=(l_1',l_2',...,l_g')$ satisfying that $l_1'=d$, $l_i'=l* l_i$, for $i=2,...,g$. We can regard this larger FCNN as a combinations of $l$ FCNNs with architecture $\mathbf{q}$, that is: there are $m$ disjoint sub-FCNNs with architecture $\mathbf{q}$ in the larger FCNN with architecture $\mathbf{q}'$. For $i$-th sub-FCNN, we use weights $\mathbf{w}^i$ and bias $\mathbf{b}^i$. For weights and bias which connect different sub-FCNNs, we set these weights and bias to $\mathbf{0}$. Finally, we can obtain that $\mathbf{g}_{\mathbf{w},\mathbf{b}}=[g_{\mathbf{w}^1,\mathbf{b}^1},g_{\mathbf{w}^2,\mathbf{b}^2},...,g_{\mathbf{w}^l,\mathbf{b}^l}]^{\top}\in \mathcal{F}_{\mathbf{q}'}^{\sigma}$, which implies that
\begin{equation*}
     \max_{\mathbf{x}\in C} \|\mathbf{g}_{\mathbf{w},\mathbf{b}}(\mathbf{x})-\mathbf{f}(\mathbf{x})\|_2<\epsilon.
\end{equation*}
We have completed this proof.
\end{proof}
\noindent Given a sequence $\mathbf{q}=(l_1,...,l_g)$,  we are interested in following function space $\mathcal{F}^{\sigma}_{\mathbf{q},\mathbf{M}}$:
\begin{equation*}
    \mathcal{F}^{\sigma}_{\mathbf{q},\mathbf{M}} := \{ \mathbf{M}\cdot (\sigma \circ \mathbf{f}): \forall~ \mathbf{f}\in \mathcal{F}^{\sigma}_{\mathbf{q}}\},
\end{equation*}
where  $\circ$ means the composition of two functions, $\cdot$ means the product of two matrices, and
\begin{equation*}
\mathbf{M} = \left[\begin{matrix}
 &\mathbf{1}_{1\times (l_g-1)}&0\\
& \mathbf{0}_{1\times (l_g-1)}&1
  \end{matrix}\right],
\end{equation*}
here $\mathbf{1}_{1\times (l_g-1)}$ is the $1\times (l_g-1)$ matrix whose all elements are $1$, and $\mathbf{0}_{1 \times (l_g-1)}$ is the $1\times (l_g-1)$ zero matrix.
Using $ \mathcal{F}^{\sigma}_{\mathbf{q},\mathbf{M}}$, we can construct a binary classification space $\mathcal{H}^{\sigma}_{\mathbf{q},\mathbf{M}}$, which consists of all classifiers satisfying the following condition:
\begin{equation*}
    h(\mathbf{x})=  \argmin_{k=\{1,2\}} f^k_{\mathbf{M}}(\mathbf{x}),
\end{equation*}
where $f^k_{\mathbf{M}}(\mathbf{x})$ is the $k$-th coordinate of $\mathbf{M}\cdot (\sigma \circ \mathbf{f})$.
\begin{lemma}\label{L20}
Suppose that $\sigma$ is the ReLU function: $\max \{x,0\}$. Given a sequence $\mathbf{q}=(l_1,...,l_g)$ satisfying that $l_1=d$ and $l_g=K+1$, then the space $\mathcal{H}^{\sigma}_{\mathbf{q},\mathbf{M}}$ contains $\phi\circ \mathcal{H}^{\sigma}_{\mathbf{q}}$, and $\mathcal{H}^{\sigma}_{\mathbf{q},\mathbf{M}}$ has finite VC dimension $($Vapnik–Chervonenkis dimension$)$, where  $\phi$ maps ID data to $1$ and OOD data to $2$. Furthermore, if given $\mathbf{q}'=(l_1',...,l_g')$ satisfying that $l_g'=K$ and $l_i'=l_i$, for $i=1,...,g-1$, then $\mathcal{H}^{\sigma}_{\mathbf{q}}\subset \mathcal{H}^{\sigma}_{\mathbf{q}'}\bullet \mathcal{H}^{\sigma}_{\mathbf{q},\mathbf{M}}$.
\end{lemma}
\begin{proof}[Proof of Lemma \ref{L20}]
For any $h_{\mathbf{w},\mathbf{b}}\in \mathcal{H}^{\sigma}_{\mathbf{q}}$, then there exists $\mathbf{f}_{\mathbf{w},\mathbf{b}}\in \mathcal{F}^{\sigma}_{\mathbf{q}}$ such that $h_{\mathbf{w},\mathbf{b}}$ is induced by $\mathbf{f}_{\mathbf{w},\mathbf{b}}$. We can write $\mathbf{f}_{\mathbf{w},\mathbf{b}}$ as follows:
\begin{equation*}
    \mathbf{f}_{\mathbf{w},\mathbf{b}}(\mathbf{x})= \mathbf{w}_{g}\mathbf{f}_{{g-1}}(\mathbf{x})+\mathbf{b}_g,
\end{equation*}
where $\mathbf{w}_{g}\in \mathbb{R}^{(K+1)\times l_{g-1}}$, $\mathbf{b}_g\in  \mathbb{R}^{(K+1)\times 1}$ and $\mathbf{f}_{{g-1}}(\mathbf{x})$ is the output of the $l_{g-1}$-th layer.

\noindent Suppose that 
\begin{equation*}
\mathbf{w}_{g} = \left[\begin{matrix}
 &\mathbf{v}_1\\
&  \mathbf{v}_2\\
&  ...\\
&  \mathbf{v}_{K}\\
& \mathbf{v}_{K+1}
  \end{matrix}\right], ~~~
\mathbf{b}_{g} = \left[\begin{matrix}
 &{b}_1\\
&{b}_2\\
&...\\
&{b}_{K}\\
&{b}_{K+1}
  \end{matrix}\right],
 \end{equation*}
where $\mathbf{v}_i\in \mathbb{R}^{1\times l_{g-1}}$ and $b_i \in \mathbb{R}$.
\\
\\
\noindent We set
\begin{equation*}
     \mathbf{f}_{\mathbf{w}',\mathbf{b}'}(\mathbf{x})= \mathbf{w}'_{g}\mathbf{f}_{{g-1}}(\mathbf{x})+\mathbf{b}_g',
\end{equation*}
where 
\begin{equation*}
\mathbf{w}'_{g} = \left[\begin{matrix}
 &\mathbf{v}_1\\
&  \mathbf{v}_2\\
&  ...\\
& ~~~ \mathbf{v}_{K}
  \end{matrix}\right], ~~~
\mathbf{b}_{g}' = \left[\begin{matrix}
 &{b}_1\\
&{b}_2\\
&...\\
&{b}_{K}
  \end{matrix}\right],
 \end{equation*}
It is obvious that $ \mathbf{f}_{\mathbf{w}',\mathbf{b}'} \in \mathcal{F}_{\mathbf{q}'}^{\sigma}$. Using  $\mathbf{f}_{\mathbf{w}',\mathbf{b}'} \in \mathcal{F}_{\mathbf{q}'}^{\sigma}$, we construct a classifier ${h}_{\mathbf{w}',\mathbf{b}'}\in \mathcal{H}^{\sigma}_{\mathbf{q}'}$:
\begin{equation*}
    {h}_{\mathbf{w}',\mathbf{b}'}= \argmax_{k\in \{1,...,K\}}{f}_{\mathbf{w}',\mathbf{b}'}^k,
\end{equation*}
where ${f}_{\mathbf{w}',\mathbf{b}'}^k$ is the $k$-th coordinate of $\mathbf{f}_{\mathbf{w}',\mathbf{b}'}$.
\\
\\
\noindent Additionally,  we consider
\begin{equation*}
    \mathbf{f}_{\mathbf{w},\mathbf{b},\mathbf{B}}= \mathbf{M} \cdot \sigma(\mathbf{B}\cdot \mathbf{f}_{\mathbf{w},\mathbf{b}})\in \mathcal{F}_{\mathbf{q},\mathbf{M}}^{\sigma},
\end{equation*}
where 
\begin{equation*}
\mathbf{B} = \left[\begin{matrix}
 &\mathbf{I}_{(l_g-1)\times (l_g-1) }&- \mathbf{1}_{(l_g-1)\times 1}\\
& \mathbf{0}_{1\times (l_g-1)}& 0
  \end{matrix}\right],
\end{equation*}
here $\mathbf{I}_{(l_g-1)\times (l_g-1) }$ is the $(l_g-1)\times (l_g-1)$ identity matrix, $\mathbf{0}_{1\times (l_g-1)}$ is the $1\times (l_g-1)$ zero matrix, and $\mathbf{1}_{(l_g-1)\times 1}$ is the $(l_g-1)\times 1$ matrix, whose all elements are $1$.

\noindent Then, we define that for any $\mathbf{x}\in \mathcal{X}$,
\begin{equation*}
    h_{\mathbf{w},\mathbf{b},\mathbf{B}}(\mathbf{x}):= \argmax_{k\in\{1,2\}} {f}_{\mathbf{w},\mathbf{b},\mathbf{B}}^k(\mathbf{x}),
\end{equation*}
where ${f}_{\mathbf{w},\mathbf{b},\mathbf{B}}^k(\mathbf{x})$ is the $k$-th coordinate of $ \mathbf{f}_{\mathbf{w},\mathbf{b},\mathbf{B}}(\mathbf{x})$. Furthermore, we can check that $h_{\mathbf{w},\mathbf{b},\mathbf{B}}$ can be written as follows: for any $\mathbf{x}\in \mathcal{X}$,
\begin{equation*}
    h_{\mathbf{w},\mathbf{b},\mathbf{B}}(\mathbf{x})=\left \{ 
    \begin{aligned}
       ~~~1,&~~~&{\rm if}~{f}_{\mathbf{w},\mathbf{b},\mathbf{B}}^1(\mathbf{x})>0;
     \\  ~~~2,&~~~&{\rm if}~{f}_{\mathbf{w},\mathbf{b},\mathbf{B}}^1(\mathbf{x})\leq 0.
    \end{aligned}
    \right.
\end{equation*}
\\
\\
\noindent It is easy to check that
\begin{equation*}
  h_{\mathbf{w},\mathbf{b},\mathbf{B}} = \phi \circ h_{\mathbf{w},\mathbf{b}},
\end{equation*}
where $\phi$ maps ID labels to $1$ and OOD labels to $2$.
\\
\\
\noindent Therefore, $h_{\mathbf{w},\mathbf{b}}(\mathbf{x})=K+1$ if and only if $h_{\mathbf{w},\mathbf{b},\mathbf{B}}=2$; and $h_{\mathbf{w},\mathbf{b}}(\mathbf{x})=k$ ($k\neq K+1$) if and only if $h_{\mathbf{w},\mathbf{b},\mathbf{B}}=1$ and $h_{\mathbf{w}',\mathbf{b}'}(\mathbf{x})=k$. This implies that $\mathcal{H}^{\sigma}_{\mathbf{q}}\subset \mathcal{H}^{\sigma}_{\mathbf{q}'}\bullet \mathcal{H}^{\sigma}_{\mathbf{q},\mathbf{M}}$ and $\phi\circ \mathcal{H}^{\sigma}_{\mathbf{q}} \subset \mathcal{H}^{\sigma}_{\mathbf{q},\mathbf{M}}$.

\noindent Let $\tilde{\mathbf{q}}$ be $(l_1,...,l_g,2)$. Then $\mathcal{F}_{\mathbf{q},\mathbf{M}}^{\sigma}\subset \mathcal{F}_{\tilde{\mathbf{q}}}^{\sigma}$. Hence, $\mathcal{H}_{\mathbf{q},\mathbf{M}}^{\sigma}\subset \mathcal{H}_{\tilde{\mathbf{q}}}^{\sigma}$. According to the VC dimension theory \citep{Peter2019Nearly} for feed-forward neural networks, $\mathcal{H}_{\tilde{\mathbf{q}}}^{\sigma}$ has finite VC dimension. Hence, $\mathcal{H}_{\mathbf{q},\mathbf{M}}^{\sigma}$ has finite VC-dimension. We have completed the proof.
\end{proof}
\begin{lemma}\label{L23}
Let $|\mathcal{X}|<+\infty$ and $\sigma$ be the ReLU function: $\max\{x,0\}$. Given $r$ hypothesis functions $h_1,h_2,...,h_r\in \{h:\mathcal{X}\rightarrow \{1,...,l\}\}$, then there exists a sequence $\mathbf{q}=(l_1,...,l_g)$ with $l_1=d$ and $l_g=l$, such that $h_1,...,h_r\in \mathcal{H}_{\mathbf{q}}^{\sigma}$.
\end{lemma}
\begin{proof}[Proof of Lemma \ref{L23}]
For each $h_i$ ($i=1,...,r$), we introduce a corresponding $\mathbf{f}_i$ (defined over $\mathcal{X}$) satisfying that for any $\mathbf{x}\in \mathcal{X}$, $\mathbf{f}_i(\mathbf{x})=\mathbf{y}_k$ if and only if $h_i(\mathbf{x})=k$, where $\mathbf{y}_k\in \mathbb{R}^{l}$ is the one-hot vector corresponding to the label $k$. Clearly, $\mathbf{f}_i$ is a continuous function in $\mathcal{X}$, because $\mathcal{X}$ is a discrete set. Tietze Extension Theorem implies that $\mathbf{f}_i$ can be extended to a  continuous function in $\mathbb{R}^d$.
\\
\\
\noindent Since $\mathcal{X}$ is a compact set, then Lemma \ref{L8UPT} implies that there exist a sequence $\mathbf{q}^i=(l_1^i,...,l_{g^i}^i)$ ($l_1^i=d$ and $l_{g^i}^i=l$) and $\mathbf{f}_{\mathbf{w},\mathbf{b}}\in \mathcal{F}_{\mathbf{q}^i}^{\sigma}$  such that 
\begin{equation*}
 \max_{\mathbf{x}\in \mathcal{X}} \|\mathbf{f}_{\mathbf{w},\mathbf{b}}(\mathbf{x})-\mathbf{f}_i(\mathbf{x})\|_{\ell_2}<\frac{1}{10\cdot l},
\end{equation*}
where $\|\cdot\|_{\ell_2}$ is the $\ell_2$ norm in $\mathbb{R}^{l}$.
Therefore, for any $\mathbf{x}\in \mathcal{X}$, it easy to check that
\begin{equation*}
    \argmax_{k\in\{1,...,l\}}{f}^k_{\mathbf{w},\mathbf{b}}(\mathbf{x})=h_i(\mathbf{x}),
\end{equation*}
where ${f}^k_{\mathbf{w},\mathbf{b}}(\mathbf{x})$ is the $k$-th coordinate of $\mathbf{f}_{\mathbf{w},\mathbf{b}}(\mathbf{x})$. Therefore,
$h_i(\mathbf{x})\in \mathcal{H}_{\mathbf{q}^i}^{\sigma}$.
\\
\\
\noindent Let $\mathbf{q}$ be $(l_1,...,l_g)$ $(l_1=d$ and $l_g=l)$ satisfying that $\mathbf{q}^i\lesssim\mathbf{q}$. Using Lemma \ref{L7-contain}, we obtain that $\mathcal{H}_{\mathbf{q}^i}^{\sigma}\subset \mathcal{H}_{\mathbf{q}}^{\sigma}$, for each $i=1,...,r$. Therefore, $h_1,...,h_r \in \mathcal{H}_{\mathbf{q}}^{\sigma}$.
\end{proof}
\begin{lemma}\label{L26}
Let the activation function $\sigma$ be the ReLU function. Suppose that $|\mathcal{X}|<+\infty$. If $\{\mathbf{v}\in \mathbb{R}^l:E(\mathbf{v})\geq \lambda\}$ and $\{\mathbf{v}\in \mathbb{R}^l:E(\mathbf{v})<\lambda\}$ both contain nonempty open sets of $\mathbb{R}^l$ $($here, open set is a topological terminology$)$. There exists a sequence $\mathbf{q}=(l_1,...,l_g)$ $(l_1=d$ and $l_g=l)$ such that $\mathcal{H}^{\sigma,\lambda}_{\mathbf{q},E}$ consists of all binary classifiers. 
\end{lemma}
\begin{proof}[Proof of Lemma \ref{L26}]
Since $\{\mathbf{v}\in \mathbb{R}^l:E(\mathbf{v})\geq \lambda\}$, $\{\mathbf{v}\in \mathbb{R}^l:E(\mathbf{v})<\lambda\}$ both contain nonempty open sets, we can find $\mathbf{v}_1\in \{\mathbf{v}\in \mathbb{R}^l:E(\mathbf{v})\geq \lambda\}$, $\mathbf{v}_2\in \{\mathbf{v}\in \mathbb{R}^l:E(\mathbf{v})< \lambda\}$ and a constant $r>0$ such that $B_r(\mathbf{v}_1)\subset \{\mathbf{v}\in \mathbb{R}^l:E(\mathbf{v})\geq \lambda\}$ and $B_r(\mathbf{v}_2)\subset \{\mathbf{v}\in \mathbb{R}^l:E(\mathbf{v})< \lambda\}$, where $B_r(\mathbf{v}_1)=\{\mathbf{v}:\|\mathbf{v}-\mathbf{v}_1\|_{\ell_2}<r\}$ and $B_r(\mathbf{v}_2)=\{\mathbf{v}:\|\mathbf{v}-\mathbf{v}_2\|_{\ell_2}<r\}$, here $\|\cdot\|_{\ell_2}$ is the $\ell_2$ norm.
\\
\\
\noindent For any binary classifier $h$ over $\mathcal{X}$, we can induce a vector-valued function as follows: for any $\mathbf{x}\in \mathcal{X}$,
\begin{equation*}
    \mathbf{f}(\mathbf{x}) = \left \{
    \begin{aligned}
   &~~~~\mathbf{v}_1,~~{\rm if}~h(\mathbf{x})=1;\\ 
   &~~~~  \mathbf{v}_2,~~{\rm if}~h(\mathbf{x})=2.
    \end{aligned}
    \right.
\end{equation*}

\noindent Since $\mathcal{X}$ is a finite set, then Tietze Extension Theorem implies that $\mathbf{f}$ can be extended to a  continuous function in $\mathbb{R}^d$. Since $\mathcal{X}$ is a compact set, Lemma \ref{L8UPT} implies that there exists a sequence $\mathbf{q}^h=(l_1^h,...,l_{g^h}^h)$ ($l_1^h=d$ and $l_{g^h}^h=l$) and $\mathbf{f}_{\mathbf{w},\mathbf{b}}\in \mathcal{F}_{\mathbf{q}^h}^{\sigma}$  such that 
\begin{equation*}
 \max_{\mathbf{x}\in \mathcal{X}} \|\mathbf{f}_{\mathbf{w},\mathbf{b}}(\mathbf{x})-\mathbf{f}(\mathbf{x})\|_{\ell_2}<\frac{r}{2},
\end{equation*}
where $\|\cdot\|_{\ell_2}$ is the $\ell_2$ norm in $\mathbb{R}^{l}$.
Therefore, for any $\mathbf{x}\in \mathcal{X}$, it is easy to check that $E(\mathbf{f}_{\mathbf{w},\mathbf{b}}(\mathbf{x}))\geq \lambda$ if and only if $h(\mathbf{x})=1$, and $E(\mathbf{f}_{\mathbf{w},\mathbf{b}}(\mathbf{x}))< \lambda$ if and only if $h(\mathbf{x})=2$.

\noindent For each $h$, we have found a sequence $\mathbf{q}^h$ such that $h$ is induced by $\mathbf{f}_{\mathbf{w},\mathbf{b}}\in \mathcal{F}_{\mathbf{q}^h}^{\sigma}$, $E$ and $\lambda$. Since $|\mathcal{X}|<+\infty$, only finite binary classifiers are defined over $\mathcal{X}$.  Using Lemma \ref{L23}, we can find a sequence $\mathbf{q}$ such that $\mathcal{H}^b_{\rm all}=\mathcal{H}_{\mathbf{q},E}^{\sigma,\lambda}$, where  $\mathcal{H}^b_{\rm all}$ consists of all binary classifiers.
\end{proof}

\begin{lemma}
\label{T27}
Suppose the hypothesis space is score-based.
Let $|\mathcal{X}|<+\infty$. If $\{\mathbf{v}\in \mathbb{R}^l:E(\mathbf{v})\geq \lambda\}$ and $\{\mathbf{v}\in \mathbb{R}^l:E(\mathbf{v})<\lambda\}$ both contain nonempty open sets, and Condition \ref{C3} holds, then there exists a sequence $\mathbf{q}=(l_1,...,l_g)$ $(l_1=d$ and $l_g=l)$ such that for any sequence $\mathbf{q}'$ satisfying $\mathbf{q}\lesssim\mathbf{q}'$ and any ID hypothesis space $\mathcal{H}^{\rm in}$, OOD detection is learnable in the separate space $\mathscr{D}_{XY}^{s}$ for $\mathcal{H}^{\rm in}\bullet \mathcal{H}^{\rm b}$, where $\mathcal{H}^{\rm b}=\mathcal{H}^{{\sigma},\lambda}_{{\mathbf{q}'},E}$ and $\mathcal{H}^{\rm in}\bullet \mathcal{H}^{\rm b}$ is defined below Eq. (\ref{Eq.dot}).
\end{lemma}

\begin{proof}[Proof of Lemma \ref{T27}]
Note that we use the ReLU function as the activation function in this lemma. Using Lemma \ref{L7-contain}, Lemma \ref{L26} and Theorem \ref{T17}, we can prove this result.
\end{proof}
\thmAppFCNN*
\begin{proof}[Proof of Theorem \ref{T24}]
Note that we use the ReLU function as the activation function in this theorem.

\noindent \textbf{$\bullet$ The Case that $\mathcal{H}$ is FCNN-based.}

\noindent \textbf{First}, we prove that if $|\mathcal{X}|=+\infty$, then OOD detection is not learnable in $\mathscr{D}^s_{XY}$ for $\mathcal{H}_{\mathbf{q}}^{\sigma}$, for any sequence $\mathbf{q}=(l_1,...,l_g)$ $(l_1=d$ and $l_g=K+1)$.

\noindent By Lemma \ref{L20}, Theorems 5 and 8 in \citep{bartlett2003vapnik}, we know that ${\rm VCdim}(\phi\circ \mathcal{H}_{\mathbf{q}}^{\sigma})<+\infty$, where $\phi$ maps ID data to $1$ and maps OOD data to $2$. Additionally, Proposition \ref{Pr1} implies that Assumption \ref{ass1} holds and $\sup_{{h}\in \mathcal{H}_\mathbf{q}^{\sigma}} |\{\mathbf{x}\in \mathcal{X}: {h}(\mathbf{x})\in \mathcal{Y}\}|=+\infty$, when $|\mathcal{X}|=+\infty$. Therefore, Theorem \ref{T12} implies that  OOD detection is not learnable in $\mathscr{D}^s_{XY}$ for $\mathcal{H}_{\mathbf{q}}^{\sigma}$, when $|\mathcal{X}|=+\infty$.
\\
\\
\noindent \textbf{Second}, we prove that if $|\mathcal{X}|<+\infty$, there exists a sequence $\mathbf{q}=(l_1,...,l_g)$ $(l_1=d$ and $l_g=K+1)$ such that OOD detection is learnable in $\mathscr{D}^s_{XY}$ for $\mathcal{H}_{\mathbf{q}}^{\sigma}$.

\noindent Since $|\mathcal{X}|<+\infty$, it is clear that $|\mathcal{H}_{\rm all}|<+\infty$, where $\mathcal{H}_{\rm all}$ consists of all hypothesis functions from $\mathcal{X}$ to $\mathcal{Y}_{\rm all}$. According to Lemma \ref{L23}, there exists a sequence $\mathbf{q}$ such that $\mathcal{H}_{\rm all}\subset \mathcal{H}_{\mathbf{q}}^{\sigma} $. Additionally, Lemma \ref{L20} implies that there exist $\mathcal{H}^{\rm in}$ and $\mathcal{H}^{\rm b}$ such that $\mathcal{H}_{\mathbf{q}}^{\sigma}\subset \mathcal{H}^{\rm in}\bullet \mathcal{H}^{\rm b}$. Since $\mathcal{H}_{\rm all}$ consists all hypothesis space, $\mathcal{H}_{\rm all}=\mathcal{H}_{\mathbf{q}}^{\sigma}= \mathcal{H}^{\rm in}\bullet \mathcal{H}^{\rm b}$. Therefore, $\mathcal{H}^{\rm b}$ contains all binary classifiers from $\mathcal{X}$ to $\{1,2\}$. Theorem \ref{T17} implies that OOD detection is learnable in $\mathscr{D}^s_{XY}$ for $\mathcal{H}_{\mathbf{q}}^{\sigma}$. 
\\
\\
\noindent \textbf{Third}, we prove that if $|\mathcal{X}|<+\infty$, then there exists a sequence $\mathbf{q}=(l_1,...,l_g)$ $(l_1=d$ and $l_g=K+1)$ such that for any sequence $\mathbf{q}'=(l_1',...,l_{g'}')$ satisfying that $\mathbf{q}\lesssim \mathbf{q}'$, OOD detection is learnable in $\mathscr{D}^s_{XY}$ for $\mathcal{H}_{\mathbf{q}'}^{\sigma}$.

\noindent We can use the sequence $\mathbf{q}$ constructed in the second step of the proof. Therefore, $\mathcal{H}_{\mathbf{q}}^{\sigma}=\mathcal{H}_{\rm all}$.  Lemma \ref{L7-contain} implies that  $\mathcal{H}_{\mathbf{q}}^{\sigma}\subset \mathcal{H}_{\mathbf{q}'}^{\sigma}$. Therefore, $\mathcal{H}_{\mathbf{q}'}^{\sigma}=\mathcal{H}_{\rm all}=\mathcal{H}_{\mathbf{q}}^{\sigma}$. The proving process (second step of the proof) has shown that if $|\mathcal{X}|<+\infty$, Condition \ref{C3} holds and hypothesis space $\mathcal{H}$ consists of all hypothesis functions, then OOD detection is learnable in $\mathscr{D}^s_{XY}$ for $\mathcal{H}$. Therefore, OOD detection is learnable in $\mathscr{D}^s_{XY}$ for $\mathcal{H}_{\mathbf{q}'}^{\sigma}$. We complete the proof when the hypothesis space $\mathcal{H}$ is FCNN-based.
\\
\\
\noindent \textbf{$\bullet$ The Case that $\mathcal{H}$ is score-based}

\noindent \textbf{Fourth}, we prove that if $|\mathcal{X}|=+\infty$, then OOD detection is not learnable in $\mathscr{D}^s_{XY}$ for $\mathcal{H}^{\rm in}\bullet \mathcal{H}^{\rm b}$, where $\mathcal{H}^{\rm b}=\mathcal{H}_{\mathbf{q},E}^{\sigma,\lambda}$ for any sequence $\mathbf{q}=(l_1,...,l_g)$ $(l_1=d$, $l_g=l$), where $E$ is in Eq. \eqref{score1}, \eqref{score2}, or \eqref{score3}.

\noindent By Theorems 5 and 8 in \citep{bartlett2003vapnik}, we know that ${\rm VCdim}( \mathcal{H}_{\mathbf{q},E}^{\sigma,\lambda})<+\infty$. Additionally, Proposition \ref{P2} implies that Assumption \ref{ass1} holds and $\sup_{{h}\in \mathcal{H}_\mathbf{q}^{\sigma}} |\{\mathbf{x}\in \mathcal{X}: {h}(\mathbf{x})\in \mathcal{Y}\}|=+\infty$, when $|\mathcal{X}|=+\infty$. Hence, Theorem \ref{T12} implies that  OOD detection is not learnable in $\mathscr{D}^s_{XY}$ for $\mathcal{H}_{\mathbf{q}}^{\sigma}$, when $|\mathcal{X}|=+\infty$.
\\
\\
\noindent \textbf{Fifth}, we prove that if $|\mathcal{X}|<+\infty$, there exists a sequence $\mathbf{q}=(l_1,...,l_g)$ $(l_1=d$ and $l_g=l)$ such that OOD detection is learnable in $\mathscr{D}^s_{XY}$ for for $\mathcal{H}^{\rm in}\bullet \mathcal{H}^{\rm b}$, where $\mathcal{H}^{\rm b}=\mathcal{H}_{\mathbf{q},E}^{\sigma,\lambda}$ for any sequence $\mathbf{q}=(l_1,...,l_g)$ $(l_1=d$, $l_g=l$), where $E$ is in Eq. \eqref{score1}, \eqref{score2}, or Eq. \eqref{score3}.

\noindent Based on Lemma \ref{T27}, we only need to show that $\{\mathbf{v}\in \mathbb{R}^l:E(\mathbf{v})\geq \lambda\}$ and $\{\mathbf{v}\in \mathbb{R}^l:E(\mathbf{v})<\lambda\}$ both contain nonempty open sets for different score functions $E$.

\noindent Since $\max_{k\in\{1,...,l\}}  \frac{\exp{(v^k)}}{\sum_{c=1}^l \exp{(v^c)}}$,  $\max_{k\in\{1,...,l\}}  \frac{\exp{(v^k/T)}}{\sum_{c=1}^K \exp{(v^c/T)}}$ and $T\log \sum_{c=1}^l  \exp{(v^c/T)}$ are continuous functions, whose ranges contain $(\frac{1}{l},1)$, $(\frac{1}{l},1)$, $(0,+\infty)$ and $(0,+\infty)$, respectively. 

\noindent Based on the property of continuous function ($E^{-1}(A)$ is an open set, if $A$ is an open set), we obtain that $\{\mathbf{v}\in \mathbb{R}^l:E(\mathbf{v})\geq \lambda\}$ and $\{\mathbf{v}\in \mathbb{R}^l:E(\mathbf{v})<\lambda\}$ both contain nonempty open sets. Using Lemma \ref{T27}, we complete the fifth step.
\\
\\
\noindent \textbf{Sixth}, we prove that if $|\mathcal{X}|<+\infty$, then there exists a sequence $\mathbf{q}=(l_1,...,l_g)$ $(l_1=d$ and $l_g=l)$ such that for any sequence $\mathbf{q}'=(l_1',...,l_{g'}')$ satisfying that $\mathbf{q}\lesssim \mathbf{q}'$, OOD detection is learnable in $\mathscr{D}^s_{XY}$ for $\mathcal{H}^{\rm in}\bullet \mathcal{H}^{\rm b}$, where $\mathcal{H}^{\rm b}=\mathcal{H}_{\mathbf{q}',E}^{\sigma,\lambda}$, where $E$ is in Eq. \eqref{score1}, \eqref{score2}, or Eq. \eqref{score3}.

\noindent In the fifth step, we have proven that Eqs. \eqref{score1}, \eqref{score2}, and \eqref{score3} meet the condition in Lemma \ref{T27}. Therefore, Lemma \ref{T27} implies this result. We complete the proof when the hypothesis space $\mathcal{H}$ is score-based.
\end{proof}

\section{Proof of Theorem \ref{T24_auc}}\label{SL_auc}

\thmAppFCNNauc*
\begin{proof}[Proof of Theorem \ref{T24_auc}]
  Using a similar strategy of Theorem \ref{T24}, we can prove this theorem by Theorem \ref{T12_auc} and Lemma \ref{l14-auc}.
\end{proof}
\section{Proof of Theorem \ref{T24.3}}\label{SM}
\thmAppFCNNtwo*
\begin{proof}[Proof of Theorem \ref{T24.3}]
$~$

\noindent 1) By Lemma \ref{C1andC2}, we conclude that Learnability in $\mathscr{D}_{XY}^{\mu,b}$ for $\mathcal{H}\Rightarrow$ Condition \ref{C1}.

\noindent 2) By Proposition \ref{Pr1} and Proposition \ref{P2}, we know that when $K=1$, there exist $h_1, h_2\in \mathcal{H}$, where $h_1=1$ and $h_2=2$, here $1$ represents ID, and $2$ represent OOD. Therefore, we know that when $K=1$,
$\inf_{h\in \mathcal{H}}R_D^{\rm in}(h)=0$ and $\inf_{h\in \mathcal{H}}R_D^{\rm out}(h)=0$, for any $D_{XY}\in \mathscr{D}_{XY}^{\mu,b}$.

\noindent By Condition \ref{C1}, we obtain that $\inf_{h\in \mathcal{H}} R_D(h)=0$, for any $D_{XY}\in \mathscr{D}_{XY}^{\mu,b}$. Because each domain $D_{XY}$ in $\mathscr{D}_{XY}^{\mu,b}$ is attainable, we conclude that Risk-based Realizability Assumption holds. 

\noindent We have proven that Condition \ref{C1}$\Rightarrow$ Risk-based Realizability Assumption.

\noindent 3) By Theorems 5 and 8 in \citep{bartlett2003vapnik}, we know that ${\rm VCdim}(\phi\circ \mathcal{H}_{\mathbf{q}}^{\sigma})<+\infty$ and ${\rm VCdim}( \mathcal{H}_{\mathbf{q},E}^{\sigma,\lambda})<+\infty$. Then, using Theorem \ref{T-SET2}, we conclude that Risk-based Realizability Assumption$\Rightarrow$ Learnability in $\mathscr{D}_{XY}^{\mu,b}$ for $\mathcal{H}$.

\noindent 4) According to the results in 1), 2) and 3), we have proven that

$~~~~~~~~~~~~~~~~~~$Learnability in $\mathscr{D}_{XY}^{\mu,b}$ for $\mathcal{H}\Leftrightarrow$Condition \ref{C1}$\Leftrightarrow$ Risk-based Realizability Assumption.

\noindent 5) By Lemma \ref{Lemma1}, we conclude that Condition \ref{Con2}$\Rightarrow$Condition \ref{C1}.

\noindent 6) \textbf{Here we prove that Learnability in $\mathscr{D}_{XY}^{\mu,b}$ for $\mathcal{H}\Rightarrow$Condition \ref{Con2}.} Since $\mathscr{D}^{\mu,b}_{XY}$ is the prior-unknown space, by Theorem \ref{T1}, there exist an algorithm $\mathbf{A}: \cup_{n=1}^{+\infty}(\mathcal{X}\times\mathcal{Y})^n\rightarrow \mathcal{H}$ and a monotonically decreasing sequence $\epsilon_{\rm cons}(n)$, such that $\epsilon_{\rm cons}(n)\rightarrow 0$, as $n\rightarrow +\infty$, and for any $D_{XY}\in \mathscr{D}_{XY}^{\mu,b}$, 
\begin{equation*}
\begin{split}
     &\mathbb{E}_{S\sim D^n_{X_{\rm I}Y_{\rm I}}}\big[ R^{\rm in}_D(\mathbf{A}(S))- \inf_{h\in \mathcal{H}}R^{\rm in}_D(h)\big]\leq \epsilon_{\rm cons}(n),
     \\ &\mathbb{E}_{S\sim D^n_{X_{\rm I}Y_{\rm I}}}\big[ R^{\rm out}_D(\mathbf{A}(S))- \inf_{h\in \mathcal{H}}R^{\rm out}_D(h)\big]\leq \epsilon_{\rm cons}(n).
     \end{split}
\end{equation*}
Then, for any $\epsilon>0$, we can find $n_{\epsilon}$ such that $\epsilon\geq \epsilon_{\rm cons}(n_{\epsilon})$, therefore, if $n={n_{\epsilon}}$, we have
\begin{equation*}
\begin{split}
     &\mathbb{E}_{S\sim D^{n_{\epsilon}}_{X_{\rm I}Y_{\rm I}}}\big[ R^{\rm in}_D(\mathbf{A}(S))- \inf_{h\in \mathcal{H}}R^{\rm in}_D(h)\big]\leq \epsilon,
     \\ &\mathbb{E}_{S\sim D^{n_{\epsilon}}_{X_{\rm I}Y_{\rm I}}}\big[ R^{\rm out}_D(\mathbf{A}(S))- \inf_{h\in \mathcal{H}}R^{\rm out}_D(h)\big]\leq \epsilon, 
     \end{split}
\end{equation*}
which implies that there exists $S_{\epsilon}\sim D^{n_{\epsilon}}_{X_{\rm I}Y_{\rm I}}$ such that
\begin{equation*}
\begin{split}
     & R^{\rm in}_D(\mathbf{A}(S_{\epsilon}))- \inf_{h\in \mathcal{H}}R^{\rm in}_D(h)\leq \epsilon,
     \\ & R^{\rm out}_D(\mathbf{A}(S_{\epsilon}))- \inf_{h\in \mathcal{H}}R^{\rm out}_D(h)\leq \epsilon. 
     \end{split}
\end{equation*}
Therefore, for any equivalence class $[D_{XY}']$ with respect to $\mathscr{D}_{XY}^{\mu,b}$ and any $\epsilon>0$, there exists a hypothesis function $\mathbf{A}(S_{\epsilon})\in \mathcal{H}$ such that for any domain $D_{XY}\in [D_{XY}']$,
 \begin{equation*}
 \mathbf{A}(S_{\epsilon})\in \{ h' \in \mathcal{H}: R_D^{\rm out}(h') \leq \inf_{h\in \mathcal{H}} R_D^{\rm out}(h)+\epsilon\} \cap   \{ h' \in \mathcal{H}: R_D^{\rm in}(h') \leq \inf_{h\in \mathcal{H}} R_D^{\rm in}(h)+\epsilon\},
 \end{equation*}
 which implies that Condition \ref{Con2} holds. Therefore, Learnability in $\mathscr{D}_{XY}^{\mu,b}$ for $\mathcal{H}\Rightarrow$Condition \ref{Con2}.

\noindent 7) Note that in 4), 5) and 6), we have proven that

\noindent Learnability in $\mathscr{D}_{XY}^{\mu,b}$ for $\mathcal{H}\Rightarrow$Condition \ref{Con2}$\Rightarrow$Condition \ref{C1}, and Learnability in $\mathscr{D}_{XY}^{\mu,b}$ for $\mathcal{H}\Leftrightarrow$Condition \ref{C1}, thus, we conclude that Learnability in $\mathscr{D}_{XY}^{\mu,b}$ for $\mathcal{H}\Leftrightarrow$Condition \ref{Con2}$\Leftrightarrow$Condition \ref{C1}.

\noindent 8) Combining 4) and 7), we have completed the proof.

\end{proof}

\section{Proof of Theorem \ref{T24.3_auc}}\label{SM_auc}

\thmAppFCNNtwoauc*
\begin{proof}[Proof of Theorem \ref{T24.3_auc}]
    The result can be obtained  by Theorems \ref{T-SET2_auc} and \ref{T3_auc}. 
\end{proof}
\section{Proof of Theorem \ref{overlapcase}}\label{SM_last}

\thmImpOverlapFCNN*

\begin{proof}[Proof of Theorem \ref{overlapcase}]
Using Proposition \ref{Pr1} and Proposition \ref{P2}, we obtain that $\inf_{h\in \mathcal{H}} R_D^{\rm in}(h)=0$ and $\inf_{h\in \mathcal{H}} R_D^{\rm out}(h)=0$. Then, Theorem \ref{T5} implies this result.
\end{proof}

\noindent {Note that if we replace the activation function $\sigma$ (ReLU function) in Theorem  \ref{overlapcase}  with any other activation functions, Theorem \ref{overlapcase} still hold.}

\section{Proof of Theorem \ref{overlapcase_auc}}\label{SM_last_auc}

\thmImpOverlapFCNNauc*
\begin{proof}[Proof of Theorem \ref{overlapcase_auc}]
    This is a conclusion of Lemma \ref{T5_auc}.
\end{proof}
\end{document}